\documentclass{article}

\usepackage[preprint]{neurips_2026}

\usepackage[utf8]{inputenc} % allow utf-8 input
\usepackage[T1]{fontenc}    % use 8-bit T1 fonts
\usepackage{hyperref}       % hyperlinks
\usepackage{url}            % simple URL typesetting
\usepackage{booktabs}       % professional-quality tables
\usepackage{amsfonts}       % blackboard math symbols
\usepackage{nicefrac}       % compact symbols for 1/2, etc.
\usepackage{microtype}      % microtypography
\usepackage{xcolor}         % colors
\usepackage{microtype}
\usepackage{graphicx}
\usepackage{subcaption}
\usepackage{booktabs} % for professional tables

\usepackage{hyperref}

\usepackage{amsmath}
\usepackage{amssymb}
\usepackage{mathtools}
\usepackage{amsthm}

\usepackage{natbib} % has a nice set of citation styles and commands
\usepackage{mathtools} % amsmath with fixes and additions
\usepackage{booktabs} % commands to create good-looking tables
\usepackage{tikz} % nice language for creating drawings and diagrams
\usepackage{multirow}
\usepackage{amsfonts} % for \mathbb, \mathcal, etc.
\usepackage{bm} % for bold math symbols
\usepackage{xcolor} % for colored text
\usepackage{graphicx} % for figures
\usepackage{subcaption} % for subfigures
\usepackage{enumitem}
\usepackage{tikz}
\usetikzlibrary{positioning, arrows.meta, shapes.geometric, calc, shadows}
\usepackage{cleveref}

\newtheorem{theorem}{Theorem}
\newtheorem{corollary}{Corollary}
\newtheorem{proposition}{Proposition}

\newtheorem{assumption}{Assumption}
\newtheorem{lemma}{Lemma}
\theoremstyle{remark}
\newtheorem{remark}{Remark}

\newcommand{\modelname}{CR-VBLL}

\newcommand{\KL}{D_{\text{KL}}}
\def\colorful{1}

\ifnum\colorful=1
\newcommand{\snew}[1]{{#1}}
\newcommand{\knew}[1]{{#1}}
\newcommand{\jwnew}[1]{{#1}}

\fi

\title{Efficient Analytic Uncertainty Quantification for Multi-Modal Regression}

\author {%
  \parbox{\textwidth}{\centering
    \makebox[2.4cm][c]{{\bf Kun Jin}\textsuperscript{1}} 
    \makebox[2.7cm][c]{{\bf James Harrison}\textsuperscript{2}} 
    \makebox[2.4cm][c]{{\bf Jiawei Li}\textsuperscript{1}} 
    \makebox[2.4cm][c]{{\bf Sihan Liu}\textsuperscript{1}} 
    \makebox[2.4cm][c]{{\bf Jiayi Liu}\textsuperscript{1}} \\ \vspace{0.4em}
    \makebox[3.4cm][c]{{\bf Randolph Linderman}\textsuperscript{1}} 
    \makebox[2.4cm][c]{{\bf Yuening Li}\textsuperscript{2}} 
    \makebox[3.2cm][c]{{\bf Arnab Bhadury}\textsuperscript{1}} 
    \makebox[3.6cm][c]{{\bf Sourabh Prakash Bansod}\textsuperscript{1}} \\ \vspace{0.4em}
    \makebox[2.4cm][c]{{\bf Liang Liu}\textsuperscript{1}} 
    \makebox[2.4cm][c]{{\bf Jasper Snoek}\textsuperscript{2}} \\ \vspace{0.8em}
    \textsuperscript{1}Google \quad \textsuperscript{2}Google DeepMind \\ \vspace{0.5em}
    \texttt{\small \{kunjin, jamesharrison, jweili, liusih, jiayil, randylinderman\}@google.com} \\
    \texttt{\small \{yueningl, arniebh, spbansod, liangliu, jsnoek\}@google.com}
  }
}

\begin{document}

\maketitle

\begin{abstract}
% \snew{    
% \fTBD{did some minor revision on abstract, please read} 
Efficient uncertainty quantification (UQ) is essential for trustworthy large-scale learning. Existing UQ methods for regression tasks mainly operate under the assumption that the conditional label marginal satisfies single-peak parametric models, e.g., Gaussians, where the negative log-likelihood function simplifies to the mean square error. 
However, such single-peak assumptions fail in regression tasks featuring multi-modal distributions.
% —\fTBD{For discussion, do you think we can remove this sentence to formulate our motivation only from high-level rather than from applications. Also, if we don't add Kuairand experiment, maybe let's not overly emphasize autonomous driving applications in abstract and intro} challenges commonly encountered in complex applications like ride sharing and video recommendation systems. 
On the other hand, semi-parametric methods which achieve strong regression performance for multi-modal distributions often lack efficient quantification on their prediction variances. 
In this work, we extend UQ techniques based on Variational Bayesian Inference (VBI) to two widely used semi-parametric regression models that yield histogram-like reconstructions of the conditional label densities: Quantile Regression (QR) and Classification Restoration (CR). Our approach introduces a unified, distribution-agnostic framework that simultaneously achieves accurate estimation of complex conditional distributions and highly efficient UQ.
Theoretically, our method is grounded in novel formulations of QR and CR within the VBI framework, yielding analytic Evidence Lower Bounds (ELBO) to streamline training and a closed-form or analytically approximated predictive density for efficient inference. Empirically, we evaluate our methods on three large-scale regression benchmarks with multi-modal label distributions. Our framework outperforms state-of-the-art multi-modal regression baselines, and even matches predictive performance of computationally expensive ensemble models. Furthermore, by leveraging epistemic uncertainty estimation, our approach enables highly data-efficient active learning strategies.

\end{abstract}

\section{Introduction}\label{sec:intro}

Predictive uncertainty quantification (UQ) --- the process of estimating a model's confidence in its own predictions --- is essential for trustworthy machine learning in high-stakes applications like ride sharing and recommendation systems \citep{kendall2017uncertainty}.
\snew{One of the mainstream approaches to UQ relies on variational Bayesian inference (VBI).}
\snew{In particular, applying VBI restrictively to the last layer (e.g., Variational Bayesian Last Layer, or VBLL) \citep{harrison2024variational} stands out due to its exceptional computational efficiency as compared to more elementary UQ methods based on costly ensemble learning.}
\snew{For continuous regression problems, the canonical approach within this framework is to assume that the conditional label follows a simple Gaussian distribution. This unimodal parametric assumption is favored because it allows for a mathematically simple analytic derivation of the Evidence Lower Bound (ELBO) of the log-likelihood loss, enabling strictly $\mathcal{O}(1)$ per-sample inference and training costs.}

\snew{However, in high-uncertainty real-world applications, the conditional label distribution often deviates from simple Gaussian behaviors, instead exhibiting input-dependent, complex multi-modality.}
\snew{For instance, in video recommendation systems, the watch time distributions are empirically found to possess multi-modality due to disjoint user engagement patterns, which is thoroughly discussed in \citep{zhao2025egmn}.}
% \snew{Similarly, in autonomous driving, environments notoriously demand navigating divergent but equally valid trajectories at intersections \citep{codevilla2018end, cui2019multimodal}.}
\snew{Simple parametric assumptions fail fundamentally to capture the full complexity of this multi-modality. When a model is trained using log-likelihood objectives assuming a Gaussian conditional distribution,  
the train loss simplifies to mean square error (MSE), 
forcing the model to fit a Gaussian distribution approximately at the conditional expectation of the true label distribution.}
\snew{However, in multi-modal distributions characterized by disjoint clusters, this expected mean often resides in a ``valley'' between different clusters,
where the data conditional density is close to zero.} 
We term this the \textbf{``Ghost Value''} pathology: the model fits a distribution whose density peaks at an area with low data-density. \snew{Consequently, the model is forced to produce a bloated predicted variance to maintain data coverage over the actual, distant data modes. This structural failure then leads to severely miscalibrated epistemic uncertainty
% \footnote{
% Epistemic uncertainty refers to 
% the uncertainty in the model parameters arising from a lack of observational data. See \Cref{sec:formulation} for its precise definition.
% } 
and significantly degraded prediction quality (e.g., Figure \ref{fig:two_axes_contribution} A). 
% e.g., driving an autonomous vehicle directly into an obstacle by averaging two valid turning paths.
}

\snew{To unlock more reliable UQ in such setups, our structural innovation is a distribution-agnostic framework that 
generalizes VBI beyond simple Gaussian assumptions to more complex, semi-parametric approaches --- specifically \textbf{Quantile Regression (QR)} and \textbf{Classification Restoration (CR)}.} Our method simultaneously enables (1) the reconstruction of complex conditional distributions in multi-modal regression, and (2) efficient UQ in analytic form, matching the  $\mathcal{O}(1)$ training and inference cost of deterministic networks.

Empirically, we demonstrate two advantages across diverse, large-scale multimodal domains. First, under standard training, our framework achieves competitive performance compared to state-of-the-art baselines in density estimation and out-of-distribution (OOD) sensitivity. Second, our analytic epistemic uncertainty unlocks highly data-efficient active learning; by reliably isolating epistemic uncertainty from aleatoric uncertainty, it achieves better predictive accuracy with substantially fewer labels while remaining computationally tractable at scale.

% \textbf{Paper outline.} Section~\ref{sec:formulation} formalizes UQ and multimodal regression. Section~\ref{sec:method} introduces our modular distribution-agnostic framework, details the analytic derivations for QR-VBLL \snew{and CR-VBLL}. Section~\ref{sec:theory} proves our theoretical guarantees regarding analytic uncertainty decomposition, OOD sensitivity, and estimation consistency. 
% Section~\ref{sec:experiment} validates our approach empirically, highlighting its impact on data-efficient active learning. Section~\ref{sec:conclusion} discusses limitations and future directions.

\subsection{Related Works}\label{sec:related}

\textbf{Efficient Uncertainty Quantification in Regression.}
Deterministic methods like SNGP \citep{liu2020simple}, Deep Evidential Regression (DER) \citep{amini2020deep}, and the standard Variational Bayesian Last Layer (VBLL) \citep{harrison2024variational} achieve $\mathcal{O}(1)$ inference but rely on unimodal (e.g., Gaussian) assumptions, triggering the Ghost Value pathology in multi-modal regimes. Conversely, ensemble and sampling-based methods \citep{lakshminarayanan2017simple, gal2016dropout} face a severe operational trade-off: while \snew{leading to reliable density and epistemic uncertainty estimation}, they incur an $\mathcal{O}(M)$ multiplier that is impractically expensive. 
% Our work introduces a novel framework to support efficient UQ for arbitrary multi-modal regression settings.

\textbf{Multi-Modal Regression and Density Estimation.}
Mixture Density Networks (MDNs) \citep{bishop1994density} model multi-modal targets but suffer from optimization instability and lack analytic epistemic UQ. Distributional approaches like CREAD \citep{sun2024cread} use Classification Restoration to bypass Ghost Values but lack epistemic UQ and consistency guarantees. In the continuous domain, Conformalized Quantile Regression (CQR) \citep{romano2019conformalized} produces contiguous intervals that fail on disjoint supports without epistemic decomposition. Furthermore, prior Bayesian quantile regression \citep{abeywardana2015variational, yang2016posterior} relies on Gaussian scale-mixture representations, requiring auxiliary parameters or expensive Monte-Carlo sampling based approximations. 
% In contrast, our QR-VBLL solves the integral directly to yield an exact analytic ELBO.

\textbf{Uncertainty-Aware Active Learning.}
While Bayesian Active Learning by Disagreement (BALD) \citep{houlsby2011bayesian, gal2017deep} is a dominant paradigm, applying sampling-based acquisition to large-scale regression introduces severe computational bottlenecks. Deep Ensembles \citep{lakshminarayanan2017simple} and MC-Dropout require multiple stochastic forward passes per candidate, stalling massive pool evaluations and retraining loops --- an issue exacerbated in batch settings \citep{kirsch2019batchbald}. Deterministic methods like DER \citep{amini2020deep} offer $\mathcal{O}(1)$ efficiency but fail on disjoint data; non-Bayesian heuristics \citep{yoo2019learning, sener2018active} conflate epistemic ignorance with aleatoric data noise. 
% By analytically disentangling these uncertainties in a single forward pass, our framework makes uncertainty-aware active learning tractable at industrial scales.
% isolating model ignorance from inherent data noise.
% without ensemble overhead.

\begin{figure}[t]
    \centering
    \includegraphics[width=1.00\textwidth]{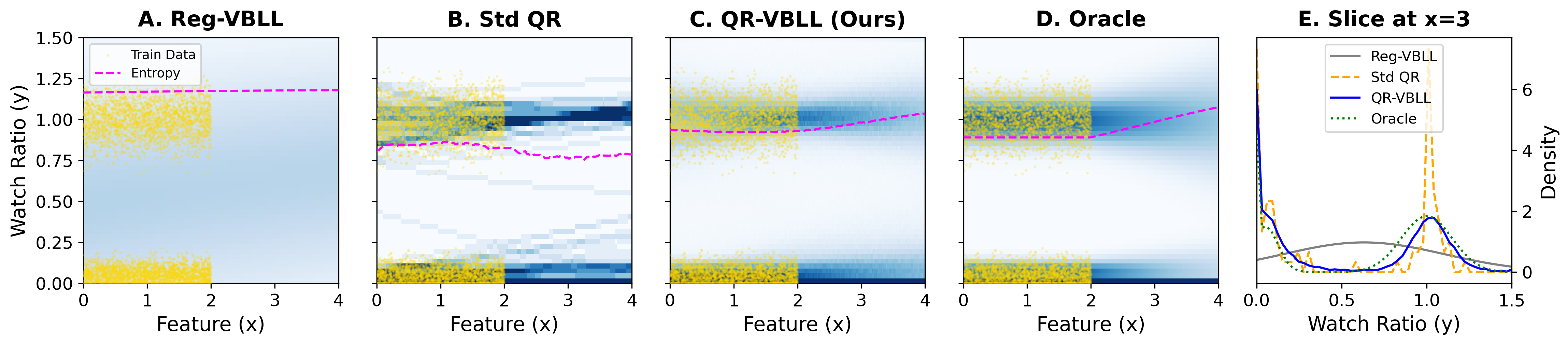}
    \caption{\textbf{Visualize our contributions with synthetic data.} The target distribution is \textbf{bimodal} ($y$ has high probability mass around 0 and 1), which triggers the \textbf{``Ghost Value''} pathology at $y \approx 0.5$ in unimodal models (Panel A). The $x > 2$ region denotes {out-of-distribution (OOD)} inputs where the ground truth distribution is gradually transitioning into pure noise. Our framework achieves (1) \textbf{representation of the conditional distribution} and (2) \textbf{efficient quantification of uncertainty}. Comparing Panel C and A, QR-VBLL resolves the ``Ghost Value'' pathology by leveraging distribution-agnostic semi-parametric representations (e.g., QR), ensuring multi-modal consistency. Comparing Panel C and B, 
    our method provides robust uncertainty quantification: QR-VBLL correctly inflates epistemic uncertainty (distribution entropy) in the OOD region to reflect model ignorance, whereas the deterministic Std QR baseline collapses the weight posterior, effectively exhibiting zero epistemic uncertainty and failing to recognize OOD inputs.
    % our method provides uncertainty quantification, where its quantification of epistemic uncertainty reflects its lack of knowledge in OOD region (proven by the increasing predicted distribution entropy in the OOD region while Standard QR does not recognize model ignorance).
    }
    \label{fig:two_axes_contribution}
    \vspace{-0.3cm}
\end{figure}
\section{Problem Formulation}
\label{sec:formulation}

% We consider a probabilistic regression setting with a dataset $\mathcal{D} = \{(x_i, y_i)\}_{i=1}^N$, where inputs are denoted as $x_i \in \mathbb{R}^d$ and continuous targets as $y_i \in \mathbb{R}$. 
We consider a probabilistic regression setting with a dataset $\mathcal{D} = (\mathcal{X}, \mathcal{Y}) = \{(x_i, y_i)\}_{i=1}^N$, where $\mathcal{X}$ denotes the collection of inputs $x_i \in \mathbb{R}^D$ and $\mathcal{Y}$ denotes the continuous targets $y_i \in \mathbb{R}$.
Our objective is to model the conditional density $p(y|x)$ in a manner that simultaneously (1) estimates the true multi-modal distribution of the data, and (2) efficiently quantifying predictive uncertainties. 

Under the VBI formulation, we assume a prior distribution over the model weights $w$ and would like to estimate the posterior $p(w | \mathcal D)$. 
For inference, 
the predictive distribution is marginalized over the posterior: $p(y|x) = \int p(y|x, w)p(w|\mathcal{D})dw$ 
while the predictive variance 
can be decomposed into \snew{\textbf{aleatoric uncertainty} --- variance caused by inherent data noise --- and \textbf{epistemic uncertainty} --- variance caused by lack of observational data} --- via the Law of Total Variance:
\begin{equation}
    \text{Var}(y|x) = \underbrace{\mathbb{E}_{w \sim p(w|\mathcal{D})}[\text{Var}(y|x, w)]}_{\text{Aleatoric (Data Noise)}} + \underbrace{\text{Var}_{w \sim p(w|\mathcal{D})}(\mathbb{E}[y|x, w])}_{\text{Epistemic (Model Ignorance)}}
    \label{eq:variance_decomposition}
\end{equation}

On the one hand, standard VBLL approaches parameterize $p(y|x)$ as a unimodal Gaussian $\mathcal{N}(y; \mu(x), \text{Var}(y|x))$ minimizing negative log-likelihood (NLL). In the homoscedastic limit, this minimizes MSE, forcing the model to fit a Gaussian centered at the conditional expectation $\mathbb{E}[y|x]$. As introduced in Section 1, this formulation leads to a structural failure on multi-modal data. Mathematically, if the true target distribution $p^*(y|x)$ is a disjoint mixture, the conditional expectation $\mathbb{E}_{p^*}[y|x]$ frequently lies in a region of near-zero density. 
% ($p(y \mid x = \mathbb{E}_{p^*}[y|x]) ) \approx 0$). 
Minimizing NLL forces the unimodal network to predict a Gaussian centered at this near-zero density region, and subsequently inflate its variance $\sigma^2(x)$ to maintain coverage over the distant true modes. This destroys both predictive accuracy and uncertainty calibration (see Figure~\ref{fig:two_axes_contribution}, A).

On the other hand, \snew{semi-parametric methods based on} standard deterministic networks \snew{(e.g., Standard QR)} collapse the weight posterior $p(w|\mathcal{D})$ to a point estimate, eliminating the epistemic term \snew{$\text{Var}_{p(w|\mathcal{D})}(\mathbb{E}[y|x, w])$} entirely and making them incapable of expressing structural ignorance when encountering sparse, out-of-distribution (OOD) samples; moreover, they may also conflate missing knowledge with data noise, \snew{leading to inflated} aleatoric variance \snew{$\mathbb{E}_{p(w|\mathcal{D})}[\text{Var}(y|x, w)]$} (Figure~\ref{fig:two_axes_contribution}, B). 

Our formal objective is to construct a framework \snew{that combines the best of both worlds: accurately modeling the density $p(y|x, w)$ to flexibly support multi-modal topologies, ensuring
the marginalization over $w$ remains analytically tractable, and decoupling estimations of aleatoric and epistemic uncertainties in $\mathcal{O}(1)$ time.}

\section{Modular Distribution-Agnostic Framework}\label{sec:method}

\begin{figure}[t]
    \centering
    
    % --- Subfigure A: QR-VBLL (Path A) ---
    \begin{subfigure}[b]{0.49\textwidth}
        \centering
        \resizebox{\textwidth}{!}{
        \begin{tikzpicture}[scale=0.7, every node/.style={transform shape}]
            % Input
            \node (input) at (0,0) [draw, circle, minimum size=0.8cm] {$x$};
            
            % Backbone (Same as A)
            \node (hidden) at (2,0) [draw, rectangle, minimum width=2cm, minimum height=1cm, align=center] {Spectral\\Backbone\\$\phi(x)$};
            \draw[->, thick] (input) -- (hidden);
            
            % VBLL Layer (Same as A)
            \node (vbll) at (5,0) [draw, rectangle, dashed, minimum width=2.5cm, minimum height=1.5cm, align=center] {VBLL Head\\$q(w) \sim \mathcal{N}$};
            \draw[->, thick] (hidden) -- (vbll);
            
            % Output Node (Quantiles)
            \node (quantile) at (8.5,0) [draw, rectangle, rounded corners, fill=blue!10, minimum width=1.5cm, minimum height=2.5cm, align=center] {Quantiles\\$\hat{q}_\kappa(x)$};
            \draw[->, thick] (vbll) -- node[above] {Scalars} (quantile);
            
            % CDF Visualization (S-Curve)
            \draw[->, thin] (9.5, -1.0) -- (10.8, -1.0) node[right, scale=0.6] {$y$};
            \draw[->, thin] (9.5, -1.0) -- (9.5, 0.8) node[above, scale=0.6] {$\kappa$};
            
            % Draw S-curve points to look like a CDF
            \draw[thick, blue] plot [smooth] coordinates {(9.5, -1.0) (9.8, -0.8) (10.0, -0.2) (10.2, 0.4) (10.5, 0.7)};
            
            % Dashed lines showing mapping
            \draw[dashed, gray] (9.5, 0.4) -- (10.2, 0.4) -- (10.2, -1.0);

            \node at (10.2, -1.5) {Predicted CDF};
        \end{tikzpicture}
        }
    \caption{Path A: \textbf{QR-VBLL}. Estimates conditional quantiles to approximate the CDF.}
        \label{fig:arch_qr}
    \end{subfigure}
        \hfill
        % --- Subfigure B: CR-VBLL (Path B) ---
    \begin{subfigure}[b]{0.49\textwidth}
        \centering
        \resizebox{\textwidth}{!}{
        \begin{tikzpicture}[scale=0.7, every node/.style={transform shape}]
            % Input
            \node (input) at (0,0) [draw, circle, minimum size=0.8cm] {$x$};
            
            % Backbone
            \node (hidden) at (2,0) [draw, rectangle, minimum width=2cm, minimum height=1cm, align=center] {Spectral\\Backbone\\$\phi(x)$};
            \draw[->, thick] (input) -- (hidden);
            
            % VBLL Layer
            \node (vbll) at (5,0) [draw, rectangle, dashed, minimum width=2.5cm, minimum height=1.5cm, align=center] {VBLL Head\\$q(w) \sim \mathcal{N}$};
            \draw[->, thick] (hidden) -- (vbll);
            
            % Output Distribution (Softmax)
            \node (softmax) at (8.5,0) [draw, rectangle, rounded corners, fill=gray!10, minimum width=1.5cm, minimum height=2.5cm, align=center] {Softmax\\$p(\mathbf{y}|x)$};
            \draw[->, thick] (vbll) -- node[above] {Logits $z$} (softmax);
            
            % Histogram Bars (Visualizing restoration)
            \draw[fill=black] (9.5, -1.0) rectangle (9.7, -0.2); 
            \draw[fill=black] (9.7, -1.0) rectangle (9.9, -0.9);
            \draw[fill=black] (9.9, -1.0) rectangle (10.1, -0.9);
            \draw[fill=black] (10.1, -1.0) rectangle (10.3, 0.8); 
            
            \node at (10, -1.5) {Predicted PDF};
        \end{tikzpicture}
        }
        \caption{Path B: \textbf{CR-VBLL}. Discretizes the target space to to approximate the PDF.}
        \label{fig:arch_cr}
    \end{subfigure}

    % FIXED: Added mandatory width argument [\linewidth] and fixed missing closing brace in caption
    \begin{subfigure}{\linewidth} \label{fig:path_selection}
        \includegraphics[width=\linewidth]{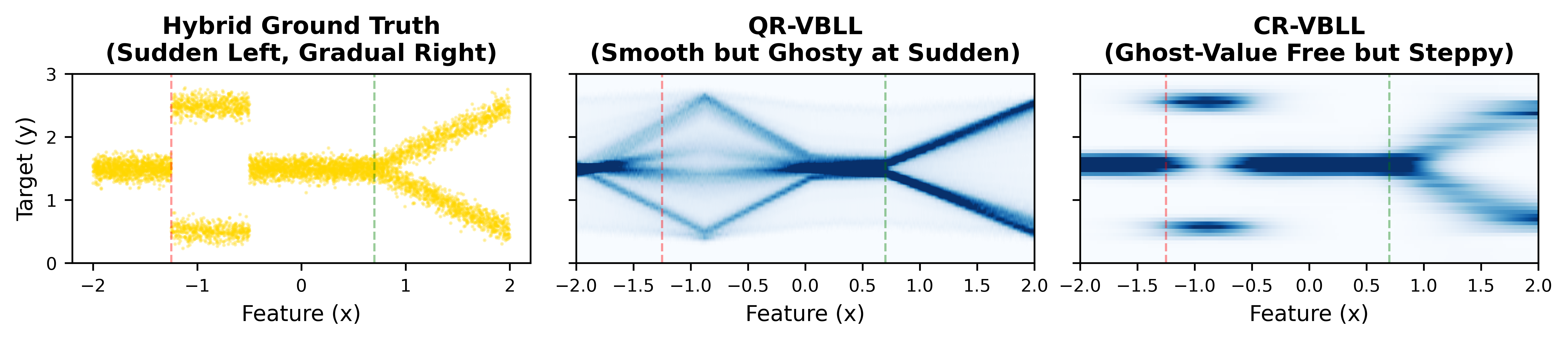}
        \caption{{Qualitative comparison of QR-VBLL and CR-VBLL on synthetic data with multimodal distribution.}}
    \end{subfigure}

    \caption{\textbf{The Modular Distribution-Agnostic Uncertainty Framework.} Both paths share a distance-aware backbone and variational inference engine, but diverge at the output head to address different topological needs:  (a) QR-VBLL is better at continuous, smooth shapes, since it excels at \emph{Global Fidelity}, smoothly capturing the gradual spread without binning artifacts and (b) CR-VBLL is better at sharp, disjoint densities, since it excels at \emph{Mode Precision}, i.e., resolving the sharp split.}
    \label{fig:unified_architecture}
        \vspace{-0.3cm}
\end{figure}
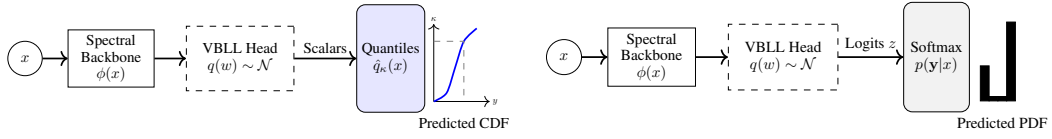
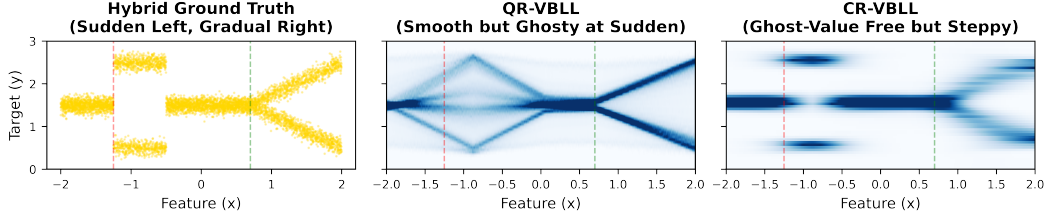

Our framework simultaneously achieves accurate density estimation with efficient UQ. At a technical level, it integrates three modular components: a distance-aware backbone, a shared Variational Bayesian inference engine, and task-specific output heads for either QR or CR.

\subsection{Component 1: Distance-Preserving Backbone}
Standard deep networks often suffer from ``feature collapse'', mapping OOD inputs to in-distribution representations. To mitigate this, we apply Spectral Normalization to the backbone $f(x)$, and adopt a \textbf{Spectral-Residual} architecture \citep{liu2020simple}, as well as an alternative that concatenates the raw input $x$ to form $\phi(x) = [f(x), x]$ (detailed in Appendix \ref{app:sn_and_raw}). 

% This structurally guarantees a bi-Lipschitz constraint 
% % $\exists c_2 > c_1 > 0$ s.t. $c_1 \|x_1 - x_2\| \le \|\phi(x_1) - \phi(x_2)\| \le c_2 \|x_1 - x_2\|$ 
% (Remark \ref{rem:bi_lipschitz}), ensuring that distances in input space are preserved in feature space, enabling reliable OOD detection. 

\begin{assumption}[Spectral Normalization \citep{liu2020simple}] \label{asm:spectral_norm}
The backbone network $\phi(x)$ employs spectral normalization, ensuring it is bi-Lipschitz: there exist constants $c_2 \geq c_1 > 0$ such that $c_1\|x_1 - x_2\|_2 \leq \|\phi(x_1) - \phi(x_2)\|_2 \leq c_2\|x_1 - x_2\|_2$ for all $x_1, x_2$.
\end{assumption}
This property is standard in distance-aware learning for OOD sensitivity \citep{liu2020simple}. 

% \begin{remark}[Distance Preservation Backbone]
% \label{rem:bi_lipschitz}
% The Bi-Lipschitz property is a standard mathematical constraint in distance-aware deep learning to guarantee out-of-distribution sensitivity \citep{liu2020simple}. While typically enforced via residual connections, an alternative  SN+Raw concatenation structure strictly satisfies this standard requirement (detailed in Appendix \ref{app:sn_and_raw}).
% \end{remark}
% We provide a detailed ablation study justifying this architectural choice in Appendix \ref{app:sn_and_raw}.

\subsection{Component 2: The Variational Bayesian Engine}
We employ a Variational Bayesian Last Layer (VBLL) \citep{harrison2024variational} to disentangle aleatoric and epistemic uncertainty efficiently. We treat the backbone $\phi(x)$ as deterministic but model the weights of the final layer $w$ as random variables, approximating the true posterior $p(w|\mathcal{D})$ with a variational distribution $q(w)$.
The Evidence Lower Bound (ELBO) is 
$
% \log p(Y | X) \geq \sum_t \mathbb{E}_{q(w)}[\log p(y_t \mid x_t, w)] - \text{KL}(q(w) \| p(w))   
\log p(\mathcal{Y}|\mathcal{X}) \ge \sum_{i=1}^N \mathbb{E}_{q(w)}[\log p(y_i|x_i, w)] - \text{KL}(q(w) \parallel p(w)) \, ,
$
where $p(w)$ is the prior over the last layer parameters and $i$ indexes the data. 

% To enable stable, single-pass training, we use \emph{Local Reparameterization} \citep{kingma2015variational}. 
Rather than sampling weights $w \sim q(w)$ to compute the expected log likelihood---which introduces high variance---we analytically compute the moments of the pre-activation logits and the resulting ELBO. 
For a linear last layer, the logits $z = w^\top \phi(x)$ follow a Gaussian distribution:
\begin{equation}
z \sim \mathcal{N}(m, s^2); ~ m := \overline{w}^\top \phi(x), ~ s^2 := \phi(x)^\top S \phi(x). \nonumber    
\end{equation}
% \label{eq:analytic_marginalization}
% $

Analytic marginalization allows us to compute the deterministic ELBO during training and exact predictive moments during inference in a \emph{single forward pass}. While the variance computation $\phi(x)^\top S \phi(x)$ scales quadratically with feature dimension $O(D^2)$ for dense covariance, this cost is negligible for standard bottleneck layers (e.g., $D=64$ in our experiments) compared to the backbone latency. Thus, our framework achieves effective $O(1)$ inference complexity relative to ensemble-based methods that require $O(M)$ full network evaluations.

% We discretize the target domain $[y_{\min}, y_{\max}]$ into $K$ fixed bins with centers $\mathcal{B} = \{b_1, \dots, b_K\}$. The model outputs a probability vector $p(x) = \text{Softmax}(z)$, where the logits $z$ are sampled from the VBLL posterior.
% \begin{equation} \label{eq:elbo_cr}
%     \mathcal{L}_{\text{CR}} = \mathbb{E}_{z \sim \mathcal{N}(m, s^2)} [\text{NLL}_{\text{CE}}(z, y)] + \beta \KL(q(w) || p(w))
% \end{equation}
% Crucially, this formulation allows the model to assign high probability to disjoint bins (e.g., $p_1 \approx 0.5$ for ``Skip'' and $p_K \approx 0.5$ for ``Complete'') while assigning zero probability to the mean, solving the ``Ghost Value'' problem.

\subsection{Path A: Variational Bayesian Quantile Regression (QR-VBLL)}
\label{sec:qr_vbll}
\snew{In multi-modal quantile regression, 
one estimates a number of quantiles of the conditional distribution $y | x$ using the pinball loss (see \Cref{eq:pinball-def}), and the final point estimate is derived by numerically integrating over the estimated quantiles using the trapezoidal rule.}
\snew{This path develops a VBI style method for quantile regression. }
% estimates the cumulative distribution function directly. In this section, we will develop a method for variational inference in the quantile regression setting. 
In particular, let
% \begin{equation}
%     z = w^\top \phi(x)
% \end{equation}
% which induces 
$
    \delta = y - z \sim \mathcal{N}(y - \bar{w}^\top \phi(x), \phi(x)^\top S \phi(x)).
$
VBLL-QR relies on computation of 
$
    E_{q(w)}[\log p(y_t \mid x_t, w)] = E_{\delta}[\log p(\delta; \sigma, \kappa)] \, ,
$
where
\begin{equation}
    p(\delta; \sigma, \kappa) = \frac{\kappa(1-\kappa)}{\sigma} \begin{cases} \exp\left(-\frac{\kappa}{\sigma}\delta\right) & \text{if } \delta \ge 0 \\ \exp\left(\frac{1-\kappa}{\sigma}\delta\right) & \text{if } \delta < 0 \end{cases}
\end{equation}
is the asymmetric Laplace distribution density.
A discussion of this likelihood function in the context of Bayesian quantile regression and its connection to pinball loss
is provided in Appendix \ref{app:vbll_qr_background}.
The following result gives a closed-form formula of this expected log likelihood function.
The proof is provided in the Appendix \ref{app:vbll_qr}.
\begin{theorem}
\label{thm:qr_elbo}
 Let $\sigma >0$, $\kappa \in (0,1)$, and $\delta \sim \mathcal{N}(\bar{\delta}, s^2)$. Then, 
 \begin{equation} 
 \label{eq:analytic_qr}
     \mathbb{E}_{\delta}[\log p(\delta; \sigma, \kappa)] = \log\frac{\kappa (1-\kappa)}{\sigma} - \frac{\bar{\delta}}{\sigma} (\kappa - 1 + \Phi(\frac{\bar{\delta}}{s})) 
     - \frac{s}{\sigma} \mathcal{N}(\frac{\bar{\delta}}{s}; 0, 1) \, ,
 \end{equation}
 where $\Phi(\cdot)$ denotes the CDF of the standard normal.
\end{theorem}
Thus, the expected log likelihood can be computed analytically, and can be represented in terms of functions that are typically available in automatic differentiation packages. The variational posterior and the neural network weights can be trained by backpropagation through this loss function. 

Substituting the analytic expectation in Eq.\eqref{eq:analytic_qr} into the variational objective yields the loss function \snew{to be minimized} for QR-VBLL:
\jwnew{
\begin{equation} 
\label{eq:loss_qr}
    \mathcal{L}_{\text{QR},\kappa} = \snew{-} \sum_{i=1}^N \underbrace{\mathbb{E}_{\delta}[\log p(y_i - z_i; \sigma_{\kappa}, \kappa)]}_{\text{Analytic Data Term (Eq.\eqref{eq:analytic_qr})}} \snew{+} \beta \underbrace{D_{\text{KL}}(q(w) || p(w))}_{\text{Regularization}}
\end{equation}
The overall loss is then summed over a predefined set of quantiles $\kappa \in \mathcal{K}$ with independent epistemic variance $\sigma_{\kappa}$: $\mathcal{L}_{QR} = \sum_{\kappa \in \mathcal{K}}\mathcal{L}_{QR, \kappa}$.}
% This objective enables stable, end-to-end training via standard backpropagation, as the gradients of the analytic log-likelihood (Eq. \ref{eq:analytic_qr}) are well-defined with respect to both the network parameters and the variational posterior moments.
Prediction in this model is done by marginalizing over the variational posterior: 
$
    p(y | x) = \mathbb{E}_{q(w)}[p(y | x, w)].
$
This following result shows that this is analytically tractable.
\begin{proposition}
    \label{thm:qr_predictive}
   Let $\sigma >0$, $\kappa \in (0,1)$, $\delta \sim \mathcal{N}(\bar{\delta}, s^2)$, $\lambda_+ = \frac{\kappa}{\sigma}$, and $\lambda_- = \frac{\kappa-1}{\sigma}$. Then, 
\begin{align}
    \label{eq:qr_vbll_density}
       \mathbb{E}_{\delta}[p(\delta; \sigma, \kappa)] =  -\sigma \lambda_- \lambda_+ \Bigg[ & \exp\left( - \lambda_- \bar{\delta} + \frac{1}{2} \lambda_-^2 s^2 \right) \Phi\left(- \frac{ \bar{\delta} - \lambda_- s^2}{s}\right) \nonumber \\
        & + \exp\left(-\lambda_+ \bar{\delta} + \frac{1}{2} \lambda_+^2 s^2\right) \Phi\left(\frac{ \bar{\delta} - \lambda_+ s^2}{s}\right) \Bigg].
\end{align}

\end{proposition}
This density is a mixture of exponentially-modified Gaussian distributions, smoothing the hard quantile estimates based on the local epistemic uncertainty $s^2$.

\subsection{Path B: Variational Bayesian Classification Restoration (CR-VBLL)}\label{subsec:cr_vbll}
% We discretize the target domain into $K$ fixed bins. The model outputs a probability vector $\mathbf{p}(y|x) = \text{Softmax}(z) \in \mathbb{R}^K$, 
% where logits $z$ are marginalized over the posterior
% \begin{equation} \label{eq:elbo_cr}
%     \mathcal{L}_{\text{CR}} = \mathbb{E}_{z \sim \mathcal{N}(m, s^2)} [\text{NLL}_{\text{CE}}(z, y)] + \beta \KL(q(w) || p(w))
% \end{equation}
% This formulation allows the model to assign high probability to disjoint bins (e.g., ``Skip'' and ``Complete'') while assigning zero probability to the Ghost Value region. The expectation of the cross entropy is lower-bounded by a deterministic objective (see \citet{harrison2024variational}), and prediction relies on local reparameterization for variance reduction \citep{kingma2015variational}.
We discretize the continuous target domain into $K$ fixed bins. The model outputs a probability vector $\mathbf{p}(y|x) = \text{Softmax}(z) \in \mathbb{R}^K$, where the pre-activation logits $z$ are marginalized over the variational posterior. 
Since our $K$ bins represent contiguous intervals of a single continuous target variable rather than heterogeneous categories, we impose a physically motivated structural assumption that the bins share a homogeneous epistemic uncertainty (formally detailed in Assumption \ref{asm:isotropic}), where the pre-activation logits follow $z \sim \mathcal{N}(m, s^2(x)\mathbf{I}_K)$, yielding the objective:
\begin{equation}
    \label{eq:elbo_cr}
    \mathcal{L}_{\text{CR}} = \mathbb{E}_{z \sim \mathcal{N}(m, s^2(x)\mathbf{I}_K)}[\text{NLL}_{\text{CE}}(z, y)] + \beta \KL(q(w) || p(w))
\end{equation}
% This isotropic formulation strictly preserves the predictive argmax of the base model and analytically reduces the intractable softmax integration to an $\mathcal{O}(1)$ closed-form temperature scaling. Furthermore, it allows the model to assign high probability to disjoint bins (e.g., ``Skip'' and ``Complete'') while assigning zero probability to the Ghost Value region. 
\snew{For training, we apply again Jensen's inequality to approximate the expected log-sum-exp term in \Cref{eq:elbo_cr}.}
For prediction, we leverage the multi-class probit approximation \citep{bishop2006pattern}
$
    \mathbb E_{ w } \left[ \mathbf{p}(y|w,x) \right] \approx \text{Softmax}\left( {m}/{(\sqrt{1 + \frac{\pi}{8} s^2(x)})} \right) ,
$
which yields a closed-form approximation that preserves the predictive argmax while flattening the confidence of uncertain inputs.

\begin{remark}[Path Selection, Figure \ref{fig:unified_architecture} (c)] 
\label{rem:path_selection}
    {QR-VBLL} excels at \emph{Global Fidelity}: it provides robust coverage for heavy-tailed or sparse distributions (minimizing Continuous Ranked Probability Score, CRPS \citep{gneiting2007strictly}) without binning artifacts. {CR-VBLL} excels at \emph{Mode Precision}: its discrete nature captures sharp, disjoint peaks with high fidelity, minimizing NLL. 
\end{remark}
% {CR-VBLL (Mode Precision):} By discretizing the output space, classification approaches can allocate arbitrary probability mass to specific bins. This allows CR-VBLL to model the \emph{sharpness} of disjoint peaks (e.g., the distinct 0.0 and 1.0 modes in user feedback) with high fidelity, minimizing NLL in multi-modal settings.

% {QR-VBLL (Global Fidelity):} In contrast, Quantile Regression approximates the Cumulative Distribution Function (CDF) directly. This smoothing effect makes QR-VBLL superior at capturing the \emph{global distributional fidelity} (measured by CRPS). It avoids the ``binning artifacts'' of CR and provides robust coverage for heavy-tailed or sparse distributions (e.g., Uber) where local modes are less distinct.

% \begin{figure*}[t]
%     \centering
%     \includegraphics[width=1.00\textwidth]{CRVBLL_vs_QRVBLL.png}
%     \caption{\textbf{Qualitative comparison of CR-VBLL and QR-VBLL on a hybrid synthetic dataset.} The ground truth (left) exhibits both a sudden, disjoint bifurcation and a gradual, continuous transition. \textbf{CR-VBLL (middle)} excels at \emph{Mode Precision}, i.e., resolving the sharp split without predicting intermediate Ghost Values, but its discrete nature introduces ``steppy'' binning artifacts in continuous regions. Conversely, \textbf{QR-VBLL (right)} excels at \emph{Global Fidelity}, smoothly capturing the gradual spread without binning artifacts, but smears the density (local smoothing) at sharp bifurcation.}
%     \label{fig:crvbll_vs_qrvbll}
%     \vspace{-0.3cm}
% \end{figure*}

\section{Framework Properties and Efficiency}
\label{sec:theory}
We provide the theoretical guarantees of our framework.
% (1) \emph{Efficient UQ:} We derive an analytic decomposition of uncertainty that enables single-pass inference ($\mathcal{O}(1)$), bypassing the prohibitive computational costs of sampling-based methods. 
% (2) \emph{OOD Sensitivity:} We establish that epistemic uncertainty grows quadratically with distance from the training manifold, structurally guaranteeing a conservative, calibrated ``Safety Net'' on out-of-distribution inputs. 
% (3) \emph{Asymptotic Uncertainty:} We show that our framework makes accurate distribution estimations as the sample number grows.

\textbf{Efficient UQ} 
We derive an analytic decomposition of uncertainty that enables single-pass inference ($\mathcal{O}(1)$), bypassing the prohibitive computational costs of sampling-based methods.
\begin{proposition}[Efficient UQ] \label{prop:decomposition}
% For both paths, the total predictive variance under the variational posterior $q(w)$ decomposes analytically:
% \begin{equation} \label{eq:variance_decomp}
%     \underbrace{\text{Var}_{q(y|x)}}_{\text{Total}} = \underbrace{\mathbb{E}_q[\text{Var}(y|x, w)]}_{\text{Aleatoric (Data Noise)}} + \underbrace{\text{Var}_q(\mathbb{E}[y|x, w])}_{\text{Epistemic (Model Uncertainty)}}
% \end{equation}
Consider the aleatoric variance and the epistemic variance defined as in \Cref{eq:variance_decomposition}
For both VBLL-QR and VBLL-CR, the two types of variances are fully analytic, and can be computed without Monte Carlo sampling.
% Crucially, both decompositions are fully analytic, requiring no Monte Carlo sampling. QR-VBLL derives closed-form moments via the Asymmetric Laplace Distribution. For CR-VBLL, the aleatoric term is the expected Gini impurity ($1 - \sum p_i^2$), while the epistemic variance is computed analytically via the Delta method and probit-softmax technique \citep{kristiadi2020being} (Appendix \ref{app:proof_decomposition}).
% Crucially, both decompositions are fully analytic, requiring no Monte Carlo sampling. 
\footnote{For QR-VBLL, $\text{Var}$ represents the standard scalar variance derived via the Asymmetric Laplace Distribution. For CR-VBLL, the use of $\text{Var}$ is a slight abuse of notation to maintain a unified equation format; the uncertainty quantification is formally resolved in the $K$-dimensional categorical distribution space via the Trace operator (see Appendix \ref{app:proof_decomposition}).}
\end{proposition}

% \subsection{Uncertainty Quantification Efficiency}
% \label{sec:efficiency}

\begin{table}[ht] % Changed from table* to table (better for NeurIPS single-column)
\vspace{-0.1in}
\caption{Comparison of UQ methods. $N$: training data size, $M$: ensemble size, $S$: MC samples.
% Our framework uniquely achieves $O(1)$ efficiency with analytic uncertainty decomposition, avoiding the linear scaling of sampling-based methods.
\emph{Distribution-Agnostic (Ours):} Makes no parametric assumption about the density shape $p(y|x)$, but relies on asymptotic consistency rather than finite-sample coverage bounds. \emph{Distribution-Free:} Guarantees finite-sample coverage (e.g., Conformal Prediction \citep{romano2019conformalized}) regardless of the underlying distribution $p(x,y)$. \emph{Parametric Mixture:} Assumes the density is a weighted sum of fixed base distributions (e.g., Gaussians in MDNs), which limits expressivity on disjoint supports.
}
\label{tab:theory_comparison}
\begin{center}
\resizebox{\textwidth}{!}{%  <--- ADD THIS
\begin{tabular}{lccccc}
% \small
\toprule
\textbf{Method} & \textbf{Multi-modal} & \textbf{Uncertainty} & \textbf{Inference} & \textbf{Training} & $p(y|x)$ \textbf{Distribution} \\
& \textbf{Consistency} & \textbf{Decomposition} & \textbf{Cost} & \textbf{Cost} & \textbf{Assumption} \\
\midrule
Deep Ensembles (MSE) & $\times$ & Via disagreement & $O(M)$ & $O(MN)$ & Parametric Mixture \\
MC-Dropout (MSE) & $\times$ & Via sampling & $O(S)$ & $O(N)$ & Approx. Mixture \\
Evidential (DER) & $\times$ & Analytic & $O(1)$ & $O(N)$ & Normal-Inv. Gaussian \\
SNGP (Standard, MSE) & $\times$ & Analytic & $O(1)$ & $O(N)$ & Gaussian \\
Conformal (CQR) & \checkmark & Coverage Only & $O(1)$ & $O(N)$ & Distribution-Free \\
\midrule
\textbf{QR-VBLL (Ours)} & \checkmark & \textbf{Analytic} & \textbf{$O(1)$} & \textbf{$O(N)$} & \textbf{Distribution-Agnostic} \\
\textbf{CR-VBLL (Ours)} & \checkmark & \textbf{Analytic} & \textbf{$O(1)$} & \textbf{$O(N)$} & \textbf{Distribution-Agnostic} \\
\bottomrule
\end{tabular}%
} % <--- AND THIS
\end{center}
\vspace{-0.1in}
\end{table}

% \begin{table*}[h]
% \caption{Comparison of uncertainty methods. $N$: training data size, $M$: ensemble size, $S$: MC samples. Our framework uniquely combines $O(1)$ efficiency with analytic uncertainty decomposition, avoiding the linear scaling of sampling-based methods.}
% \label{tab:theory_comparison}
% \begin{center}
% \small
% \begin{tabular}{lccccc}
% \toprule
% \textbf{Method} & \textbf{Multi-modal} & \textbf{Uncertainty} & \textbf{Inference} & \textbf{Training} & $p(y|x)$ \textbf{Distribution} \\
%  & \textbf{Consistency} & \textbf{Decomposition} & \textbf{Cost} & \textbf{Cost} & \textbf{Assumption} \\
% \midrule
% Deep Ensembles & $\times$ & Via disagreement & $O(M)$ & $O(MN)$ & Parametric Mixture \\
% MC-Dropout & $\times$ & Via sampling & $O(S)$ & $O(N)$ & Approximated Mixture \\
% Evidential (DER) & $\times$ & Analytic (Exact) & $O(1)$ & $O(N)$ & Normal-Inverse Gaussian \\
% SNGP & $\times$ & Analytic (Approx) & $O(1)$ & $O(N)$ & Gaussian \\
% Conformal (CQR) & \checkmark & Coverage Only & $O(1)$ & $O(N)$ & Distribution-Free \\
% \midrule
% \textbf{CR-VBLL (Ours)} & \checkmark & \textbf{Analytic (Approx)} & \textbf{$O(1)$} & \textbf{$O(N)$} & \textbf{Distribution-Agnostic} \\
% \textbf{QR-VBLL (Ours)} & \checkmark & \textbf{Analytic (Exact)} & \textbf{$O(1)$} & \textbf{$O(N)$} & \textbf{Distribution-Agnostic} \\
% \bottomrule
% \end{tabular}
% \end{center}
% \vskip -0.1in
% % \footnotetext[1]{Exact with respect to the variational posterior $q(w)$.}
% \end{table*}

We contrast the computational complexity of our framework against established baselines (Table~\ref{tab:theory_comparison}).

\emph{Sampling-Based vs. Deterministic.} Sampling methods incur prohibitive costs for real-time systems: Deep Ensembles \citep{lakshminarayanan2017simple} require $O(M)$ compute for both training and inference, while MC-Dropout \citep{gal2016dropout} requires $O(S)$ inference passes. Baselines with $O(1)$ efficiency either can not estimate the arbitrary conditional density (DER, CQR) or suffer from optimization instability (MDN), or rely on approximate uncertainty quantification (SNGP).

\emph{Ours (Analytic \& Efficient).} QR-VBLL and CR-VBLL achieve optimal $O(1)$ training and inference latency, matching standard deterministic regression while providing the rigorous uncertainty decomposition of Bayesian methods. Unlike SNGP, which approximates the GP posterior, our framework derives a tractable analytic decomposition of the variational posterior. Memory overhead is negligible ($<0.1\%$ parameters), restricted to the final layer's variational weights.

% \begin{remark}[Distributional Assumptions in Table \ref{tab:theory_comparison}]
% % To clarify the terminology in Table \ref{tab:theory_comparison}:
% \emph{Distribution-Agnostic (Ours):} Makes no parametric assumption about the density shape $p(y|x)$, but relies on asymptotic consistency rather than finite-sample coverage bounds. \emph{Distribution-Free:} Guarantees finite-sample coverage (e.g., Conformal Prediction \citep{romano2019conformalized}) regardless of the underlying distribution $p(x,y)$. \emph{Parametric Mixture:} Assumes the density is a weighted sum of fixed base distributions (e.g., Gaussians in MDNs), which limits expressivity on disjoint supports.
% \end{remark}
% \begin{remark}[Flexible vs. Distribution-Free]
% \label{rem:dist_assumption}
% We distinguish \textit{Flexible} mixture approximations (e.g., Ensembles) from true \textit{Distribution-Free} methods. While mixtures can theoretically model any shape, finite ensembles ($M \approx 5$) remain constrained by their parametric base learners, often ``oversmoothing'' disjoint modes. Our distribution-free heads remove this parametric bottleneck entirely, allowing the model to assign zero probability to "dead zones" without regularization penalties.
% \end{remark}

% \subsection{OOD Sensitivity and Efficiency}

\textbf{OOD Sensitivity}
We establish that epistemic uncertainty grows quadratically with distance from the training manifold, structurally guaranteeing a conservative, calibrated ``Safety Net'' on out-of-distribution inputs. 
% A crucial property is that the epistemic uncertainty grows with distance for inputs sufficiently far from the training samples. \snew{This is foundational to the sensitivity of our models' uncertainty estimation to OOD samples, which we refer to as the ``Safety Net'' effects.}
\begin{proposition}[OOD Sensitivity and Epistemic Growth] \label{prop:ood}
Under the Bi-Lipschitz Property (Assumption \ref{asm:spectral_norm}), let $x'$ be an OOD input at distance $d = \min_{x \in \mathcal{D}_{\text{train}}} ||x' - x||_2$ from the training set. For any $d \ge R/c_1$, where $R = \max_{x \in \mathcal{D}_{\text{train}}} ||\phi(x)||_2$, the epistemic uncertainty $s^2(x')$ is lower-bounded by a quadratic function of the distance:
$
    s^2(x') = \phi(x')^\top S \phi(x') \ge \lambda_{\min}(S) (c_1 d - R)^2.
$
\end{proposition}

% \begin{proof}[Proof Sketch]
% Spectral normalization ensures the backbone $\phi(x)$ satisfies the upper bound of the bi-Lipschitz property \citep{liu2020simple}, and when the model does not preserve the distance of $x'$. 
% Since $S$ is positive definite, the quadratic form $\phi(x')^\top S \phi(x')$ scales with $\|\phi(x')\|_2^2$. Thus, as inputs move away from the training manifold, the variance of the latent scores (logits for CR, quantile estimates for QR) increases quadratically, ensuring OOD detection.
% \end{proof}

\begin{remark} [Implicit Calibration via Marginalization] \label{rem:ood_calibration}
Note that because the raw conditional quantile estimators are deterministic posterior means, individual predicted quantiles do not natively inflate with epistemic variance. Instead, our OOD "Safety Net" is achieved via the marginalization of the weight posterior, which structurally flattens the full continuous predictive density proportional to $s^2(x)$ (see Appendix \ref{app:calibration_proof} for detailed description and derivations).
    % Beyond merely detecting OOD inputs (Proposition 4), the marginalization of the weight posterior structurally regularizes the model against overconfidence (formal proofs in Appendix \ref{app:calibration_proof}).  For QR-VBLL, the analytic smoothing of the Asymmetric Laplace likelihood (Theorem 1) into a mixture of exponentially-modified Gaussians explicitly forces the predictive variance to incorporate the epistemic term $s^2(x)$. This guarantees that predictive intervals dynamically inflate on OOD inputs, preventing the overconfident zero-width collapse common in standard quantile regression. For CR-VBLL, this marginalization induces an input-dependent temperature scaling, which strictly dampens the maximum softmax confidence in regions of high epistemic uncertainty \citep{seo2019learning, kristiadi2020being}.
\end{remark}

% \begin{proof}[Proof Sketch]
% The result follows from the concavity of the log-softmax (Path A) and the variance decomposition of mixture distributions (Path B). In both cases, high epistemic uncertainty ($\sigma^2_{\text{VBLL}}$) propagates to the output, strictly increasing the entropy (or width) of the predictive distribution relative to the MAP estimate.
% \end{proof}

% For the formal proposition and its derivation, including the probit approximation bounds and variance inflation proof, see Appendix~\ref{app:calibration_proof}.

% \textbf{Asymptotic Consistency.} Our framework preserves the statistical guarantees of its non-parametric base estimators. As the dataset size $N \to \infty$, the KL-divergence penalty in the ELBO vanishes relative to the data likelihood. This causes the variational posterior $q(w)$ to contract to the optimal point estimate, guaranteeing that both branches asymptotically recover their respective ground-truth targets: QR-VBLL converges to the true conditional quantiles, CR-VBLL converges in probability to the true target density. We provide the formal proposition and detailed proofs for this unified asymptotic consistency in Appendix \ref{app:proof-consistency}. 

\textbf{Asymptotic Consistency}
Last but not least,
we show that our framework makes accurate distribution estimations as the sample size grows. We provide the detailed mathematical derivations and formal proofs in Appendix \ref{app:proof-consistency}.
\begin{proposition}[Asymptotic Consistency]
As the dataset size $N \to \infty$, 
% the KL-divergence penalty in the ELBO vanishes relative to the data likelihood, causing the variational posterior $q(w)$ to contract to the optimal point estimate. Consequently, our framework preserves the statistical guarantees of its non-parametric base estimators: 
QR-VBLL converges to the true conditional quantiles, and CR-VBLL converges in probability to the true target density.
\end{proposition}

\section{Experiments}\label{sec:experiment}

\begin{figure*}[t]
    \centering
    \begin{subfigure}{0.34\textwidth}
        \includegraphics[width=\linewidth]{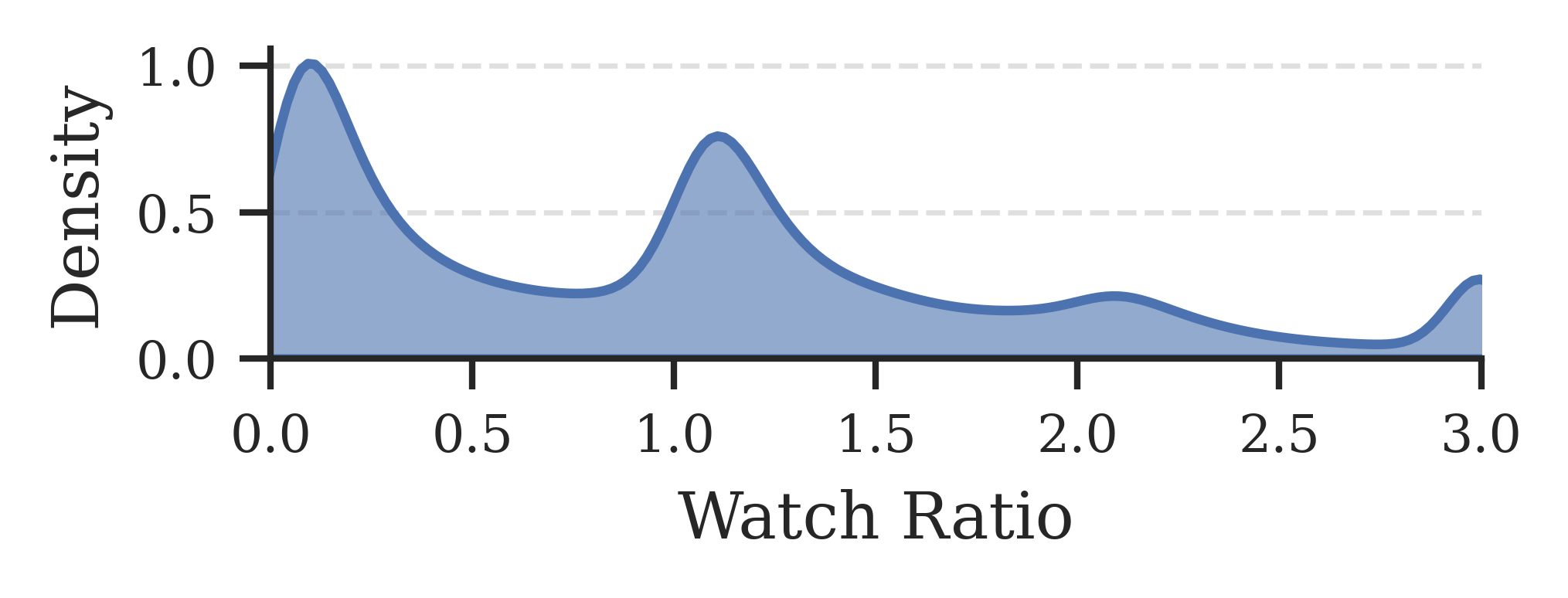}
        \caption{WeChat}
        \label{fig:wechat_dist}
    \end{subfigure}
    \hfill
    \begin{subfigure}{0.32\textwidth}
        \includegraphics[width=\linewidth]{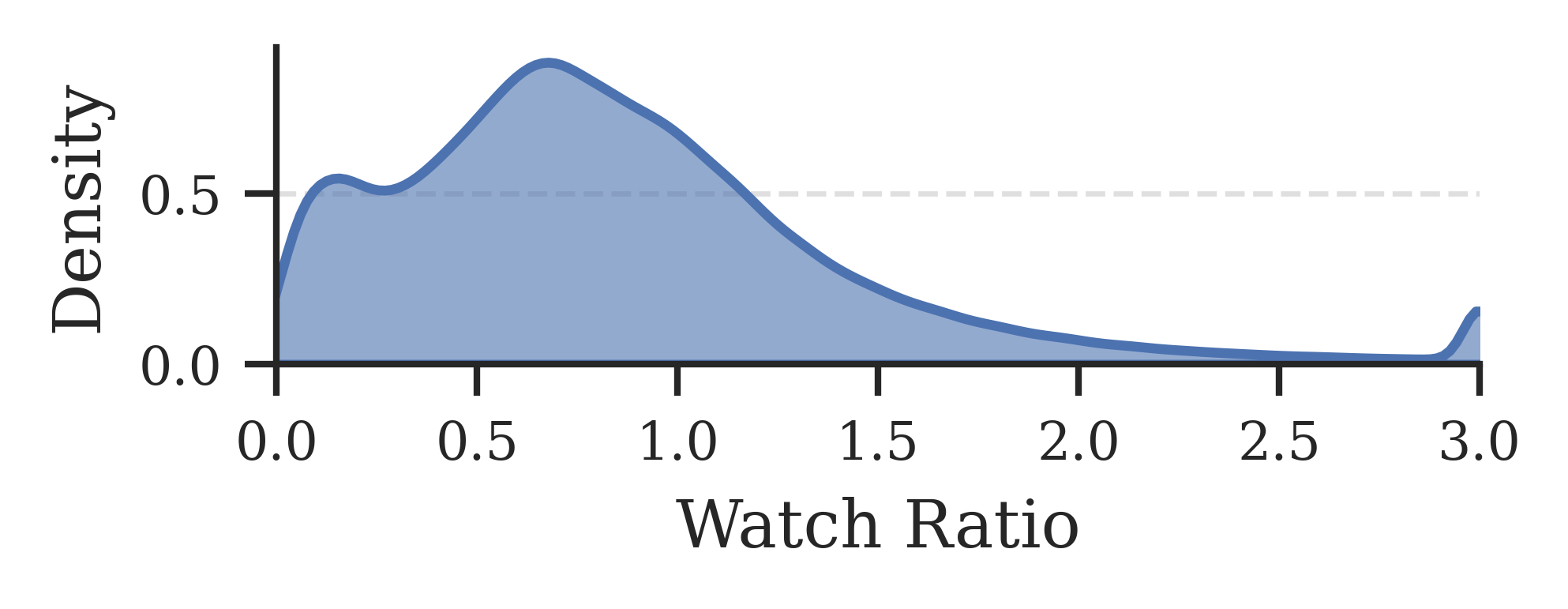}
        \caption{KuaiRec}
        \label{fig:kuairec_dist}
    \end{subfigure}
    \hfill
    \begin{subfigure}{0.32\textwidth}
        \includegraphics[width=\linewidth]{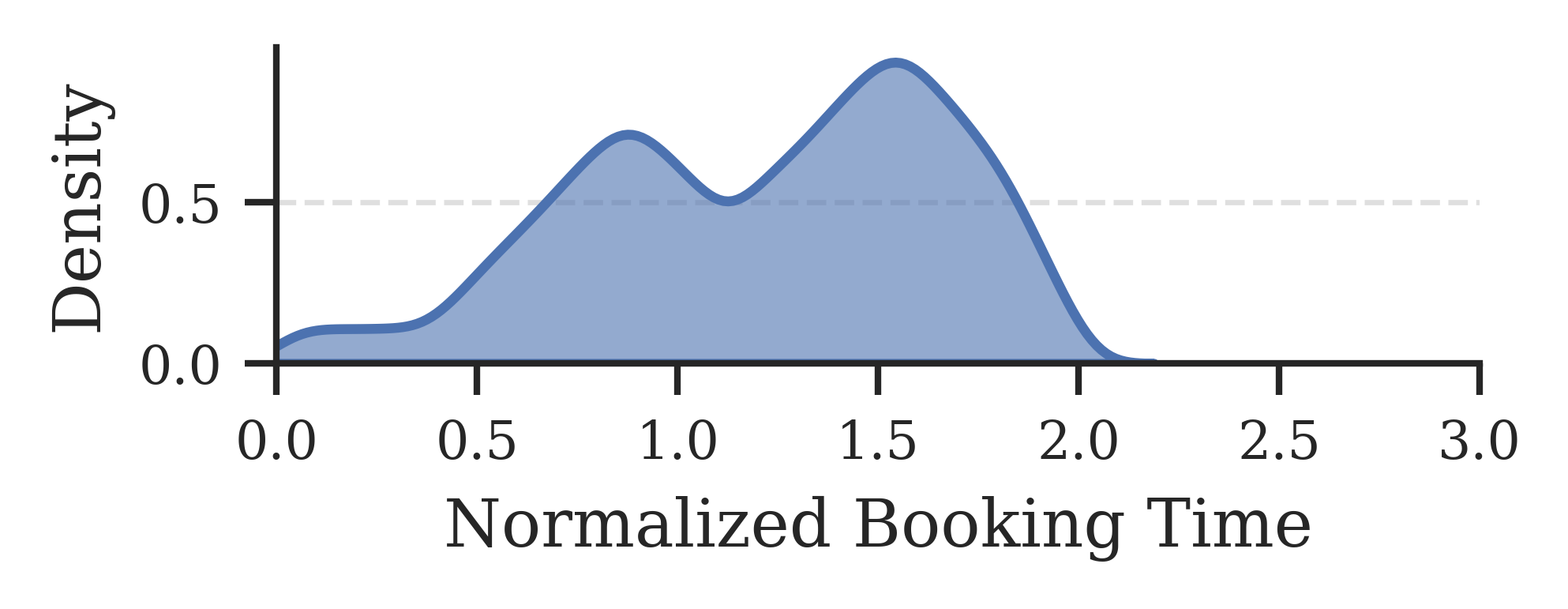}
        \caption{Uber }
        \label{fig:uber_dist}
    \end{subfigure}
    \caption{{Target distributions for Real-World Benchmarks.} The label density plots reveal the complex, multi-modal nature of real-world regression tasks. (a) {WeChat} exhibits sharp, disjoint modes typical of user engagement (skip vs. complete). (b) {KuaiRec} shows a heavy-tailed distribution with a spike at the max value. (c) {Uber} displays smooth bimodality representing different traffic conditions. 
    % Standard Gaussian regression fails to capture these topologies, often collapsing to the low-density mean (``Ghost Value'').
    }
    \label{fig:real_world_dist}
    \vspace{-0.3cm}
\end{figure*}

We evaluate our framework on three large-scale benchmarks: WeChat (Short-Form Video Recommendation), KuaiRec (Short-Form Video Recommendation), and Uber (Ride Sharing). As illustrated in Figure~\ref{fig:real_world_dist}, these datasets exhibit highly non-Gaussian, multi-modal target distributions where standard MSE-based approaches are prone to failure. We use {bold} font to indicate the best performance under each method category.

To ensure a rigorous and fair comparison on methods across disjoint (CR) and continuous (Gaussian/QR) paradigms, we standardize the evaluation metrics.
(1) Discretized NLL: We partition the target space into 10 uniform buckets. For continuous models (Gaussian, QR-VBLL), we integrate the predicted PDF (or difference in CDF) over each bucket to obtain probability masses. We then compute the NLL of these masses against the discretized ground truth. 
(2) Discretized ECE: Similarly, we compute ECE using these 10 fixed buckets to measure the alignment between the predicted probability mass and the empirical frequency of the target falling within each bucket.
(3) CRPS remains a continuous metric calculated on the full CDF.
(4) RMSE is computed based on the predicted conditional mean and the label.
Detailed definitions are provided in Appendix \ref{app:experiments}. 

Due to space constraints, simplified results for WeChat, KuaiRec and Uber datasets are presented in Table \ref{tab:all_datasets}, and standard deviations are included in the corresponding full tables, Table \ref{tab:wechat_full}, \ref{tab:kuairec_full}, and \ref{tab:uber_full} in Appendix \ref{app:experiments}.

\begin{table*}[htbp]
    \centering
    \caption{Performance comparison across WeChat, KuaiRec, and Uber datasets. We report the mean values for NLL, RMSE, CRPS, and ECE. Methods are categorized into (1) single-pass baselines (2) single-pass CR (3) single-pass QR and (4) ensemble. \textbf{Bold} (based on more precise metrics in Table \ref{tab:wechat_full}, \ref{tab:kuairec_full}, and \ref{tab:uber_full}) indicates the best performance among each method category. The \textit{Ensemble Baselines} (bottom section) require $5\times$ computational cost; notably, our single-pass solutions outperform existing methods with similar computation cost and even match or exceed the expensive baselines.}
    \label{tab:all_datasets}
    \resizebox{\textwidth}{!}{%
    \begin{tabular}{lcccc|cccc|cccc}
        \toprule
        & \multicolumn{4}{c}{\textbf{WeChat (Recommendation)}} & \multicolumn{4}{c}{\textbf{KuaiRec (Recommendation)}} & \multicolumn{4}{c}{\textbf{Uber (Transportation)}} \\
        \cmidrule(lr){2-5} \cmidrule(lr){6-9} \cmidrule(lr){10-13}
        \textbf{Method} & \textbf{NLL} & \textbf{RMSE} & \textbf{CRPS} & \textbf{ECE} & \textbf{NLL} & \textbf{RMSE} & \textbf{CRPS} & \textbf{ECE} & \textbf{NLL} & \textbf{RMSE} & \textbf{CRPS} & \textbf{ECE} \\
        \midrule
        % \multicolumn{13}{l}{\textit{Parametric Regression}} \\
        Gaussian & 2.1657 & 1.0466 & 0.4380 & 0.0264 & 1.5070 & 0.4457 & \textbf{0.2093} & 0.0291 & 2.6766 & 0.4686 & 0.3039 & 0.0848 \\
        SNGP & 2.2730 & \textbf{0.7740} & 0.4439 & 0.0302 & 1.7470 & \textbf{0.4440} & 0.2394 & 0.0497 & \textbf{1.8745} & \textbf{0.4516} & \textbf{0.2634} & \textbf{0.0341} \\
        VBLL & 2.1708 & 0.8327 & 0.4315 & 0.0251 & 1.5184 & 0.4484 & 0.2117 & 0.0327 & 3.3107 & 0.4533 & 0.3061 & 0.1043 \\
        DER & 2.1443 & 0.7811 & \textbf{0.4226} & 0.0263 & 1.4163 & 0.4580 & 0.2131 & 0.0231 & 2.8354 & 0.4541 & 0.3299 & 0.1196 \\
        MDN & \textbf{1.9515} & 1.2464 & 0.4418 & \textbf{0.0070} & \textbf{1.2943} & 0.4495 & 0.2117 & \textbf{0.0136} & 4.0078 & 0.4606 & 0.3264 & 0.1181 \\
        
        \midrule
        % \multicolumn{13}{l}{\textit{Quantile Regression (QR)}} \\
        Vanilla QR & 2.0317 & 0.7629 & 0.4156 & 0.0066 & 1.3938 & 0.4438 & 0.1968 & 0.0136 & 2.3728 & 0.4578 & 0.2686 & 0.0370 \\
        CQR & 2.0446 & 0.7723 & 0.4226 & 0.0065 & 1.4018 & 0.4447 & 0.1976 & 0.0148 & 1.9435 & 0.4590 & 0.2666 & 0.0379 \\
        QR-MCD & 2.0297 & 0.7646 & 0.4172 & \textbf{0.0062} & 1.4049 & 0.4439 & 0.1978 & 0.0145 & 2.1179 & 0.4606 & 0.2686 & 0.0357 \\
        \textbf{QR-VBLL} & \textbf{1.9616} & \textbf{0.7614} & \textbf{0.4134} & 0.0067 & \textbf{1.3716} & \textbf{0.4433} & \textbf{0.1964} & \textbf{0.0122} & \textbf{1.7830} & \textbf{0.4495} & \textbf{0.2560} & \textbf{0.0080} \\
        
        \midrule
        % \multicolumn{13}{l}{\textit{Classification Restoration (CR)}} \\
        Vanilla CR & 1.9199 & 0.7664 & \textbf{0.4063} & 0.0126 & 1.3978 & 0.4550 & 0.1991 & 0.0237 & 2.5328 & 0.4550 & 0.2911 & 0.0804 \\
        CR-SNGP & 1.9377 & 0.7756 & 0.4151 & 0.0133 & {1.3874} & 0.4584 & \textbf{0.1980} & {0.0237} & 1.7892 & 0.4512 & \textbf{0.2559} & 0.0087 \\
        CR-MCD & 1.9245 & 0.7687 & 0.4072 & \textbf{0.0119} & 1.3937 & 0.4551 & 0.1991 & 0.0237 & 2.2687 & 0.4584 & 0.2882 & 0.0719 \\
        \textbf{CR-VBLL} & \textbf{1.9189} & \textbf{0.7654} & 0.4068 & \textbf{0.0119} & \textbf{1.3867} & \textbf{0.4549} & 0.1989 & \textbf{0.0235} & \textbf{1.7839} & \textbf{0.4505} & \textbf{0.2559} & \textbf{0.0072} \\
        
        \midrule
        % \midrule
        % \multicolumn{13}{l}{\textit{Ensemble Baselines (Multi-Pass, 5x Cost)}} \\
        Ensemble & 2.1502 & 0.8035 & 0.4315 & 0.0300 & 1.4980 & \textbf{0.4405} & 0.2085 & 0.0285 & 2.7145 & 0.4668 & 0.3070 & 0.0903 \\
        QR-Ensemble & 1.9973 & \textbf{0.7506} & {0.4076} & 0.0055 & \textbf{1.3858} & 0.4433 & \textbf{0.1958} & \textbf{0.0136} & \textbf{2.1280} & 0.4578 & \textbf{0.2686} & \textbf{0.0393} \\
        CR-Ensemble & \textbf{1.8981} & 0.7547 & \textbf{0.3982} & 0.0137 & 1.3880 & 0.4540 & 0.1980 & 0.0237 & 2.3305 & \textbf{0.4506} & 0.2693 & 0.0506 \\
        \bottomrule
    \end{tabular}
    }
\end{table*}

The WeChat dataset (Figure \ref{fig:wechat_dist}) features a highly disjoint target distribution (0.0 vs 1.0 watch ratios) with sharp density peaks. CR-VBLL achieves``Mode Precision'' with the best NLL among single-pass methods, effectively matching or beating the computationally expensive methods except for CR-Ensemble. QR-VBLL excelled at global fidelity with the best CRPS  in single-pass methods.

% KuaiRec: High-Fidelity at Low Cost
The KuaiRec dataset (Figure \ref{fig:kuairec_dist}) exhibits a more smooth conditional label distribution with a heavy-tail. While the MDN achieves the lowest NLL among single-pass methods, this comes at a cost to global fidelity, yielding a CRPS ($0.2117$) more than that of QR-VBLL ($0.1964$). Meanwhile, QR-VBLL achieves the best CRPS and ECE among single-pass methods, also with more consistent performance than MDN on the other datasets, and even outperforms the computationally expensive QR-Ensemble on NLL and ECE.

% Uber: Robustness via Global Fidelity
The Uber dataset (Figure \ref{fig:uber_dist}) contains sparse data with a bimodal but continuous target distribution. Again QR-VBLL is giving a compelling performance against single-pass and ensemble baselines. Moreover, both CR-VBLL and QR-VBLL are having a much larger advantage margin from the ``safety net'' effect (like Figure \ref{fig:two_axes_contribution} B vs C) on sparse data.

\subsection{Uncertainty-Aware Active Learning}

\emph{Validating the epistemic signal.} We verify that the predicted epistemic variance correctly reflects true model ignorance. As illustrated in Figure~\ref{fig:safety_net}, the average epistemic uncertainty on the WeChat dataset is highest for cold-start users with minimal interaction history ($N < 20$) and decays monotonically as their history increases. This confirms that our architecture successfully identifies ``known unknowns'' in sparse data regimes, providing a reliable, inflated variance to prevent overconfident predictions for under-represented users.

\begin{figure}[t]
    \centering
    \begin{minipage}[t]{0.56\textwidth}
        \centering
        \includegraphics[width=\textwidth]{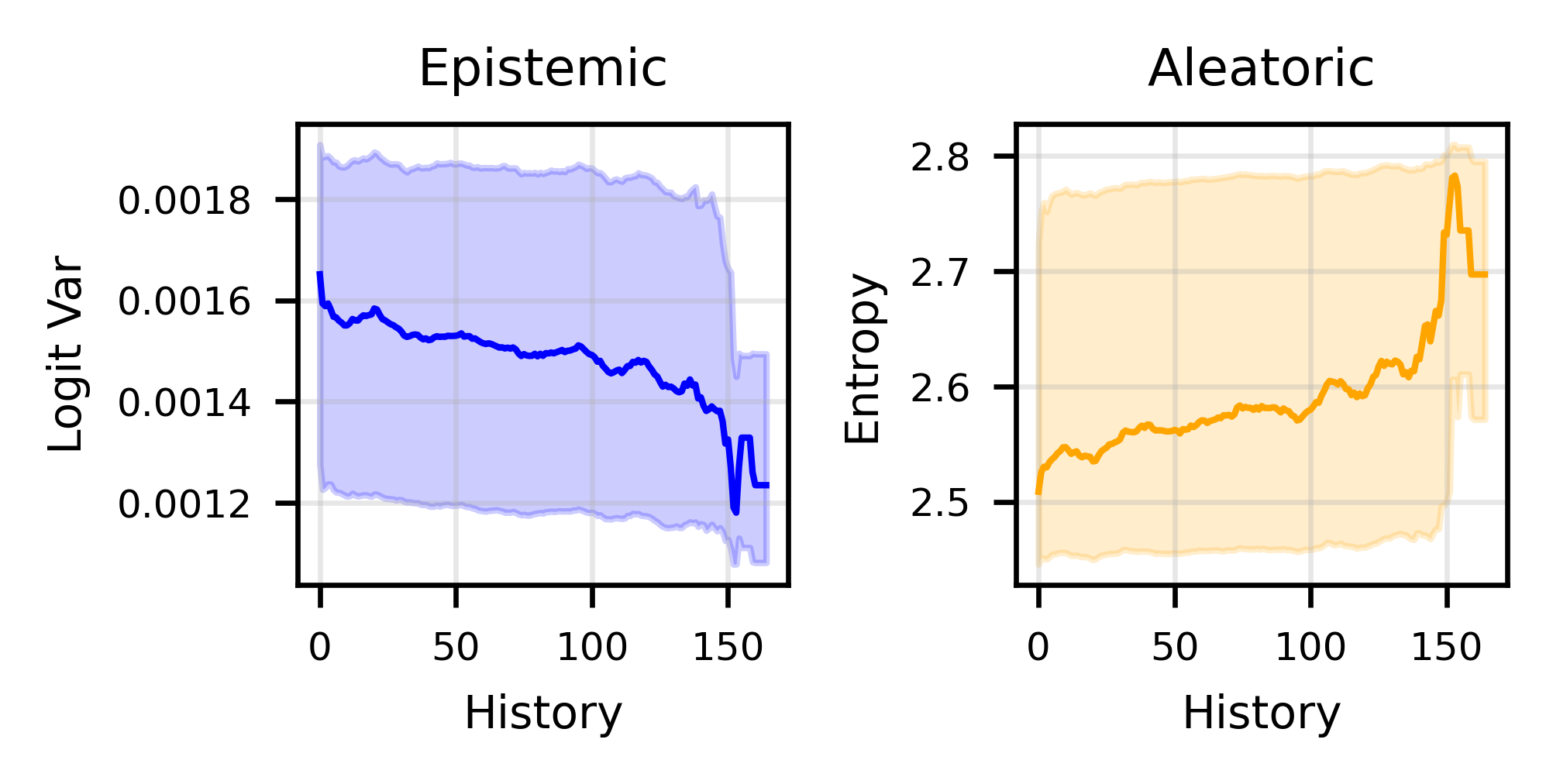}
        \caption{{Visualizing UQ.} Epistemic uncertainty is high for cold-start users and decays as the model observes more data.}
        \label{fig:safety_net}
    \end{minipage}
    \hfill
    \begin{minipage}[t]{0.41\textwidth}
        \centering
        \includegraphics[width=0.925\textwidth]{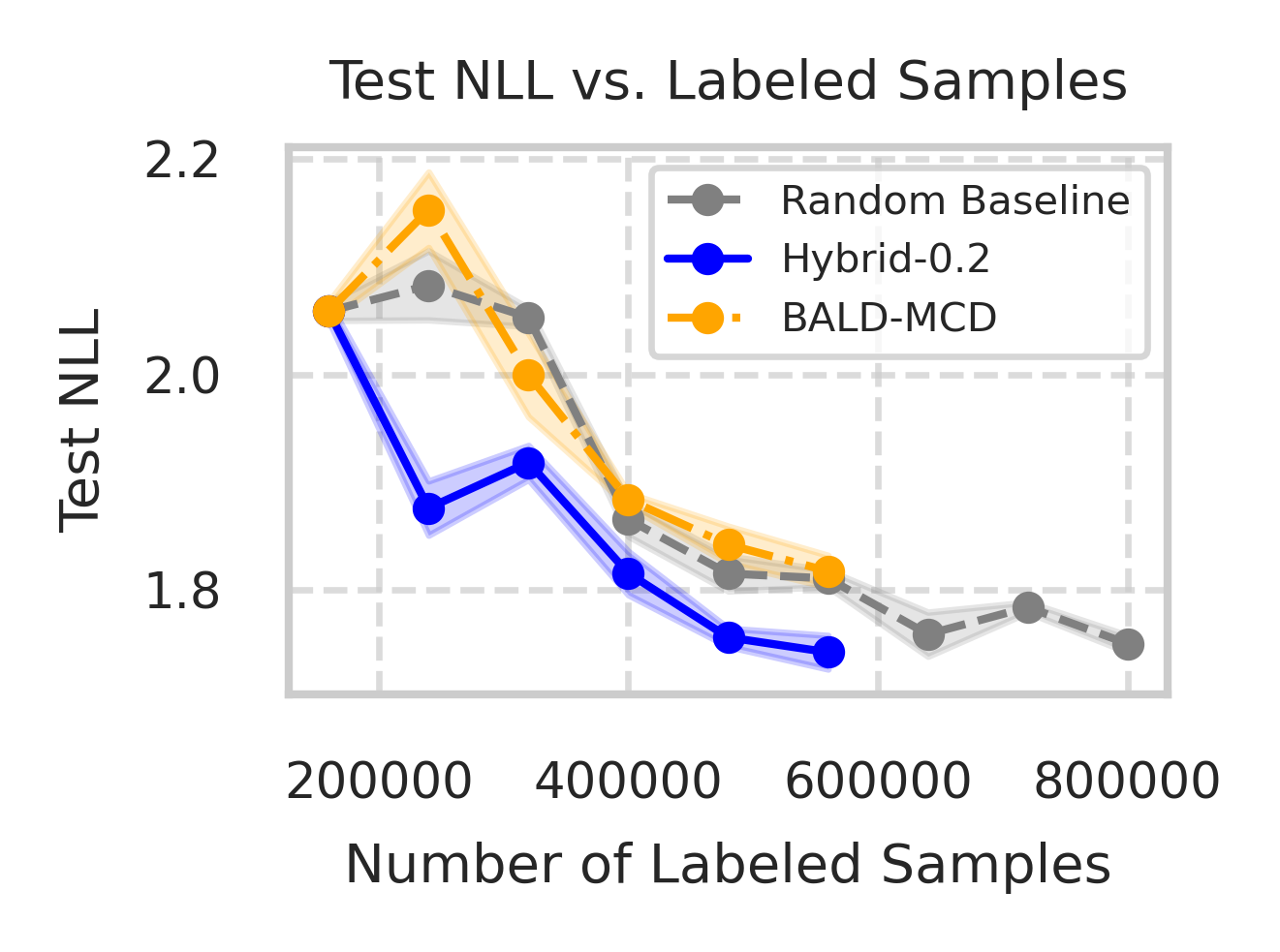}
        \caption{The test NLL curves for active learning methods on KuaiRec, demonstrating data efficiency improvement.}
        \label{fig:active_learning_curve}
    \end{minipage}
    \vspace{-0.3cm}
\end{figure}
% \begin{figure}[h]
% \centering\includegraphics[width=0.85\columnwidth]{uncertainty_trend.png}
%     \caption{\textbf{Visualizing Epistemic Uncertainty.} Epistemic uncertainty is high for cold-start users and decays as the model observes more data. This confirms that QR-VBLL correctly inflates variance for sparse inputs, providing a robust uncertainty estimate.}
%     \label{fig:safety_net}
% \end{figure}

\emph{Operationalizing uncertainty for data acquisition.} Having established that this epistemic signal is faithful, we next demonstrate how it can be leveraged to address a critical bottleneck in industrial applications: the prohibitive cost of acquiring labeled data. Because our framework reliably isolates epistemic ignorance from inherent aleatoric noise, it enables highly efficient active learning.

\emph{Experimental setup and baselines.} 
% We evaluate the data efficiency of our framework across all three distribution topologies (WeChat, KuaiRec, and Uber) by simulating a cold-start active learning loop. We initialize the model with $20\%$ of the training data and iteratively acquire samples. To do this, we propose a Hybrid acquisition strategy $\mathcal{A}(x)$ that leverages our analytic variance decomposition (Proposition \ref{prop:decomposition}) to balance epistemic exploration with aleatoric noise characterization:
We evaluate the data efficiency of our framework across all three distribution topologies (WeChat, KuaiRec, and Uber) by simulating a cold-start active learning loop. We initialize the model with $20\%$ of the training data and iteratively acquire samples. Standard uncertainty sampling relies on the total predictive variance (implicitly setting $\gamma=1$), which frequently wastes budget by getting trapped in irreducible data noise. To successfully bypass these extreme high-noise traps while preserving calibration, we propose a Hybrid acquisition strategy $\mathcal{A}(x)$. Leveraging our analytic variance decomposition (Proposition \ref{prop:decomposition}), we explicitly decouple and downweight the aleatoric component:
\begin{equation} \label{eq:acquisition_function}
    \mathcal{A}(x) = \text{Var}_q(\mathbb{E}[y|x, w]) + \gamma \cdot \mathbb{E}_q[\text{Var}(y|x, w)]
\end{equation}
Setting a discounted $\gamma \in (0, 1)$ (e.g., $\gamma = 0.1$ for WeChat and $\gamma = 0.2$ for KuaiRec and Uber; ablation in Appendix \ref{app:ablation}) ensures the model prioritizes exploring sparse regions. Crucially, retaining this downweighted aleatoric signal allows the model to deliberately sample just enough noisy areas to correctly calibrate the variance head, acting as a distributional regularizer without being overwhelmed by pure noise.
% Setting $\gamma = 0.1$ for WeChat and $\gamma = 0.2$ for KuaiRec and Uber (ablation study in Appendix \ref{app:ablation}),
% \fTBD{Add reference for this formulation i.e. including aleatoric uncertainty for active learning} 
% this strategy ensures the model explores sparse regions while simultaneously sampling high-noise areas to correctly calibrate the variance head, acting as a distributional regularizer. 
We compare our $\mathcal{O}(1)$ Hybrid strategy against four baselines: (1) Random sampling over the unlabeled pool; (2) Bayesian Active Learning by Disagreement via Vanilla QR MCD (BALD-MCD) ($\mathcal{O}(S)$), which measures Monte Carlo Dropout disagreement via the variance of predictive means over $S=25$ stochastic passes; (3) BALD of QR-VBLL (BALD-VBLL) ($\mathcal{O}(1)$), the analytic analogue of BALD for our setting, defined as the ratio of epistemic to total variance $\text{Var}_q(\mathbb{E}[y|x,w]) / \text{Var}_q(y|x)$ computed in closed form; and (4) Pure Epistemic ($\mathcal{O}(1)$) sampling based purely on the weight uncertainty variance ($\gamma=0$). Critically, all non-random strategies share the exact same model and posterior, ensuring an unconfounded comparison.

\begin{table}[ht]
\caption{Active Learning Performance across datasets. We report both NLL (mode precision) and CRPS (global fidelity) to ensure improvements are genuine. We transpose the evaluation to highlight method performance across topologies: WeChat (disjoint), KuaiRec (heavy-tailed), and Uber (bimodal). We also include the WeChat Cold-Start subset to demonstrate uncertainty acquisition value in sparse regimes. The full active learning metrics with standard deviations are covered in Appendix \ref{app:experiments}}
\label{tab:active_learning_comprehensive}
\begin{center}
\resizebox{1\linewidth}{!}{%
\small
\setlength{\tabcolsep}{4pt}
\begin{tabular}{l cc cc cc cc}
\toprule
\multirow{2}{*}{\textbf{Method}} & \multicolumn{2}{c}{\textbf{WeChat (Full)}} & \multicolumn{2}{c}{\textbf{WeChat (Cold)}} & \multicolumn{2}{c}{\textbf{KuaiRec}} & \multicolumn{2}{c}{\textbf{Uber}} \\
\cmidrule(lr){2-3} \cmidrule(lr){4-5} \cmidrule(lr){6-7} \cmidrule(lr){8-9}
& NLL & CRPS & NLL & CRPS & NLL & CRPS & NLL & CRPS \\
\midrule
\textbf{Hybrid (Ours)} & \textbf{2.0206} & \textbf{0.4434} & \textbf{2.0157} & \textbf{0.4602} & \textbf{1.7423} & \textbf{0.2085} & \textbf{1.7840} & \textbf{0.2562} \\
{BALD-MCD} ($\mathcal{O}(S)$) & 2.0258 & \textbf{0.4434} & 2.0195 & 0.4612 & 1.8175 & 0.2100 & 1.7851 & \textbf{0.2562} \\
{BALD-VBLL} ($\mathcal{O}(1)$) & 2.0277 & 0.4454 & 2.0162 & 0.4620 & 2.0450 & 0.2144 & 1.7853 & 0.2564 \\
{Pure Epistemic} & 2.0283 & 0.4452 & 2.0167 & 0.4618 & 1.7662 & 0.2100 & 1.7848 & 0.2568 \\
{Random (43\% more training \jwnew{samples})} & 2.0261 & 0.4438 & 2.0176 & \textbf{0.4602} & 1.7502 & \textbf{0.2085} & 1.7919 & \textbf{0.2562} \\
\bottomrule
\end{tabular}%
}
\end{center}
\vspace{-0.15in}
\end{table}

\emph{Results and insights.} As shown in Table~\ref{tab:active_learning_comprehensive}, the results support three major conclusions regarding our uncertainty decomposition: (1) \emph{Efficiency vs. quality:} Our Hybrid strategy matches or outperforms BALD-MCD across all three datasets despite requiring no additional forward passes. The $\mathcal{O}(1)$ analytic decomposition provides an acquisition signal of equal or superior quality to the computationally expensive $\mathcal{O}(S)$ Monte Carlo approximation. 
% (2) \emph{The danger of pure epistemic sampling in multi-modal domains:} There is a consistent performance gap between our Hybrid approach and both BALD-VBLL and Pure Epistemic sampling. This confirms that explicitly incorporating the aleatoric term into the acquisition function provides a critical regularizing signal. Pure epistemic sampling frequently over-indexes on inherently noisy regions of the target topology, whereas our Hybrid approach successfully avoids these regions to better calibrate the variance head. 
(2) \emph{The necessity of controlled aleatoric sampling:} There is a consistent performance gap between our Hybrid approach and both BALD-VBLL and Pure Epistemic sampling ($\gamma=0$). This confirms that explicitly incorporating a downweighted aleatoric term provides a critical regularizing signal. Pure epistemic sampling strictly avoids inherent data noise, which inadvertently starves the variance head of the samples needed to correctly calibrate the aleatoric uncertainty. Conversely, standard total-variance sampling (e.g., $\gamma=1$) would over-index on and be trapped by high-noise regions. By using a discounted $\gamma$, our Hybrid approach successfully avoids the trap of pure noise exploration, while simultaneously ensuring the variance head is properly calibrated.
(3) \emph{Superior data efficiency:} Our framework achieves better NLL and CRPS using $70\%$ of all the labeled samples on WeChat, Uber, and KuaiRec, the Hybrid approach yielded the best performance, suggesting our analytic decomposition is particularly valuable on complex, heavy-tailed distributions where aleatoric and epistemic signals are notoriously difficult to disentangle.

\section{Conclusion}
\label{sec:conclusion}

We introduced the first modular, distribution-agnostic framework to simultaneously achieve (1) accurate multimodal distribution estimation and (2) efficient uncertainty quantification. Our framework branches into Quantile Regression (QR) for continuous shapes and Classification Restoration (CR) for disjoint modes. Critically, we present sampling-free objectives for training both classes of models, including novel results for analytical variational QR. By pairing a distance-preserving backbone with a Variational Bayesian Last Layer (VBLL), our architecture injects calibrated epistemic uncertainty into predictions requiring only a single forward pass ($\mathcal{O}(1)$). We provide both theoretical and numerical evidence to show our framework provides a robust and efficient solution for trustworthy multi-modal regression.

\textbf{Limitations \& Future Directions.} Our approach currently relies on dataset-specific heuristics for routing between QR/CR paths; developing a fully adaptive, data-driven routing module is a key future direction. Furthermore, establishing tight finite-sample error bounds for the complete VBLL framework, and expanding to multivariate targets, remain important open challenges. Additionally, we currently use a relatively naive method for reconstructing predictive PDF/CDFs with QR, ensembling over the variational posterior. It may be possible and of interest to develop more efficient methods of reconstruction.

% \clearpage
\bibliography{references}

@article{li2010bayesian,
  title={Bayesian regularized quantile regression},
  author={Li, Qing and Lin, Nan and Xi, Ruibin},
  journal={Bayesian Analysis},
  year={2010}
}

@article{tsionas2003bayesian,
  title={Bayesian quantile inference},
  author={Tsionas, Efthymios G},
  journal={Journal of statistical computation and simulation},
  _volume={73},
  _number={9},
  _pages={659--674},
  year={2003},
  _publisher={Taylor \& Francis}
}

@book{koenker2005quantile,
  title={Quantile regression},
  author={Koenker, Roger},
  volume={38},
  year={2005},
  publisher={Cambridge university press}
}

@article{yang2016posterior,
  title={Posterior inference in Bayesian quantile regression with asymmetric Laplace likelihood},
  author={Yang, Yunwen and Wang, Huixia Judy and He, Xuming},
  journal={International Statistical Review},
  _volume={84},
  _number={3},
  _pages={327--344},
  year={2016},
  _publisher={Wiley Online Library}
}

@article{kozumi2011gibbs,
  title={Gibbs sampling methods for Bayesian quantile regression},
  author={Kozumi, Hideo and Kobayashi, Genya},
  journal={Journal of statistical computation and simulation},
  _volume={81},
  _number={11},
  _pages={1565--1578},
  year={2011},
  _publisher={Taylor \& Francis}
}

@article{koenker1978regression,
  title={Regression quantiles},
  author={Koenker, Roger and Bassett Jr, Gilbert},
  journal={Econometrica},
  _pages={33--50},
  year={1978},
  _publisher={JSTOR}
}

@inproceedings{abeywardana2015variational,
  title={Variational inference for nonparametric Bayesian quantile regression},
  author={Abeywardana, Sachinthaka and Ramos, Fabio},
  booktitle={Proceedings of the Twenty-Ninth AAAI Conference on Artificial Intelligence},
  year={2015}
}

@article{yu2001bayesian,
  title={Bayesian quantile regression},
  author={Yu, Keming and Moyeed, Rana A},
  journal={Statistics \& Probability Letters},
  _volume={54},
  _number={4},
  _pages={437--447},
  year={2001},
  _publisher={Elsevier}
}

@inproceedings{zhao2025egmn,
  title={Multi-Granularity Distribution Modeling for Video Watch Time Prediction via Exponential-Gaussian Mixture Network},
  author={Zhao, Xu and Ma, RuiBo and Chen, Jiaqi and Zhao, Weiqi and Yang, Ping and Hu, Yao},
  booktitle={Proceedings of the Nineteenth ACM Conference on Recommender Systems (RecSys '25)},
  pages={1--10},
  year={2025},
  organization={ACM}
}

@inproceedings{sun2024cread,
  title={{CREAD}: A Classification-Restoration Framework with Error Adaptive Discretization for Watch Time Prediction},
  author={Sun, Jie and Ding, Zhaoying and Wang, Ben and others},
  booktitle={Proceedings of the AAAI Conference on Artificial Intelligence (AAAI)},
  volume={38},
  pages={9002--9010},
  year={2024}
}

@inproceedings{bishop1994density,
  title={Mixture density networks},
  author={Bishop, Christopher M},
  year={1994},
  booktitle={Technical Report NCRG/4288, Aston University}
}

@inproceedings{kendall2017uncertainty,
  title={What uncertainties do we need in {Bayesian} deep learning for computer vision?},
  author={Kendall, Alex and Gal, Yarin},
  booktitle={Advances in Neural Information Processing Systems (NeurIPS)},
  volume={30},
  year={2017}
}

@inproceedings{lakshminarayanan2017simple,
  title={Simple and scalable predictive uncertainty estimation using deep ensembles},
  author={Lakshminarayanan, Balaji and Pritzel, Alexander and Blundell, Charles},
  booktitle={Advances in Neural Information Processing Systems (NeurIPS)},
  volume={30},
  year={2017}
}

@inproceedings{amini2020deep,
  title={Deep evidential regression},
  author={Amini, Alexander and Schwarting, Wilko and Rosman, Guy and Karaman, Sertac and Rus, Daniela},
  booktitle={Advances in Neural Information Processing Systems (NeurIPS)},
  volume={33},
  pages={14927--14937},
  year={2020}
}

@inproceedings{liu2020simple,
  title={Simple and principled uncertainty estimation with deterministic deep learning via distance awareness},
  author={Liu, Jeremiah and Lin, Zi and Padhy, Shreyas and Tran, Dustin and Bedrax-Weiss, Tania and Lakshminarayanan, Balaji},
  booktitle={Advances in Neural Information Processing Systems (NeurIPS)},
  volume={33},
  pages={7498--7512},
  year={2020}
}

@inproceedings{romano2019conformalized,
  title={Conformalized quantile regression},
  author={Romano, Yaniv and Patterson, Evan and Candes, Emmanuel},
  booktitle={Advances in Neural Information Processing Systems (NeurIPS)},
  volume={32},
  year={2019}
}

@inproceedings{harrison2024variational,
  title={Variational {Bayesian} Last Layers},
  author={Harrison, James and Willes, John and Snoek, Jasper},
  booktitle={International Conference on Learning Representations (ICLR)},
  year={2024}
}

@article{scott1979optimal,
  title={On optimal and data-based histograms},
  author={Scott, David W.},
  journal={Biometrika},
  volume={66},
  number={3},
  pages={605--610},
  year={1979},
  publisher={Oxford University Press},
  doi={10.1093/biomet/66.3.605}
}

@article{gal2016dropout,
  title={Dropout as a bayesian approximation: Representing model uncertainty in deep learning},
  author={Gal, Yarin and Ghahramani, Zoubin},
  journal={International Conference on Machine Learning (ICML)},
  pages={1050--1059},
  year={2016},
  organization={PMLR}
}

@article{bishop2006pattern,
  title={Pattern recognition and machine learning},
  author={Bishop, Christopher M},
  journal={Springer},
  year={2006}
}

@article{gneiting2007strictly,
  title={Strictly proper scoring rules, prediction, and estimation},
  author={Gneiting, Tilmann and Raftery, Adrian E},
  journal={Journal of the American Statistical Association},
  volume={102},
  number={477},
  pages={359--378},
  year={2007},
  publisher={Taylor \& Francis}
}

@article{takeuchi2006nonparametric,
  title={Nonparametric quantile estimation},
  author={Takeuchi, Ichiro and Le, Quoc V and Sears, Timothy D and Smola, Alexander J},
  journal={Journal of Machine Learning Research},
  volume={7},
  pages={1231--1264},
  year={2006}
}

@article{cencov1962estimation,
  title={Estimation of an unknown distribution density from observations},
  author={{\v{C}}encov, N. N.},
  journal={Soviet Mathematics},
  volume={3},
  pages={1559--1562},
  year={1962}
}

@article{steinwart2011estimating,
  title={Estimating conditional quantiles with the help of the pinball loss},
  author={Steinwart, Ingo and Christmann, Andreas},
  journal={Bernoulli},
  volume={17},
  number={1},
  pages={211--225},
  year={2011},
  publisher={Bernoulli Society for Mathematical Statistics and Probability},
  doi={10.3150/10-BEJ267}
}

@inproceedings{seo2019learning,
  title={Learning for Single-Shot Confidence Calibration in Deep Neural Networks through Stochastic Inferences},
  author={Seo, Seonguk and Seo, Paul Hongsuck and Han, Bohyung},
  booktitle={Proceedings of the IEEE/CVF Conference on Computer Vision and Pattern Recognition (CVPR)},
  pages={9030--9038},
  year={2019}
}

@inproceedings{kristiadi2020being,
  title={Being {B}ayesian, Even Just a Bit, Fixes Overconfidence in {ReLU} Networks},
  author={Kristiadi, Agustinus and Hein, Matthias and Hennig, Philipp},
  booktitle={Proceedings of the 37th International Conference on Machine Learning (ICML)},
  pages={5436--5446},
  year={2020},
  publisher={PMLR},
  volume={119}
}

@inproceedings{yang2017breaking,
  title={Breaking the Softmax Bottleneck: A High-Rank {RNN} Language Model},
  author={Yang, Zhilin and Dai, Zihang and Salakhutdinov, Ruslan and Cohen, William W.},
  booktitle={International Conference on Learning Representations},
  year={2018},
  url={https://openreview.net/forum?id=HkwZSG-CZ}
}

@book{van2000asymptotic,
  title={Asymptotic statistics},
  author={Van der Vaart, Aad W},
  volume={3},
  year={2000},
  publisher={Cambridge university press}
}

@article{houlsby2011bayesian,
  title={Bayesian active learning for classification and preference learning},
  author={Houlsby, Neil and Husz{\'a}r, Ferenc and Ghahramani, Zoubin and Lengyel, M{\'a}t{\'e}},
  journal={arXiv preprint arXiv:1112.5745},
  year={2011}
}

@inproceedings{gal2017deep,
  title={Deep bayesian active learning with image data},
  author={Gal, Yarin and Islam, Riashat and Ghahramani, Zoubin},
  booktitle={Proceedings of the 34th International Conference on Machine Learning (ICML)},
  pages={1183--1192},
  year={2017},
  organization={PMLR}
}

@inproceedings{kirsch2019batchbald,
  title={BatchBALD: Efficient and diverse batch active learning for deep neural networks},
  author={Kirsch, Andreas and Van Amersfoort, Joost and Gal, Yarin},
  booktitle={Advances in Neural Information Processing Systems (NeurIPS)},
  volume={32},
  year={2019}
}

@inproceedings{yoo2019learning,
  title={Learning loss for active learning},
  author={Yoo, Donggeun and Kweon, In So},
  booktitle={Proceedings of the IEEE/CVF Conference on Computer Vision and Pattern Recognition (CVPR)},
  pages={93--102},
  year={2019}
}

@inproceedings{sener2018active,
  title={Active Learning for Convolutional Neural Networks: A Core-Set Approach},
  author={Sener, Ozan and Savarese, Silvio},
  booktitle={International Conference on Learning Representations (ICLR)},
  year={2018}
}

% --- APPENDIX SETUP ---
\clearpage
\appendix
\onecolumn % Optional: Switch to one column for easier reading of math proofs
\section{Nomenclature and Notation}
\label{app:notation}
\begin{table}[h]
\centering
\small
\caption{Summary of mathematical notation.}
\label{tab:notation}
\begin{tabular}{@{} p{0.15\columnwidth} p{0.8\columnwidth} @{}}
\toprule
\textbf{Symbol} & \textbf{Description} \\ \midrule
\multicolumn{2}{l}{\textit{Data and Backbone}} \\
$\mathcal{D}$ & Dataset of $N$ i.i.d. observations, $\mathcal{D}=\{(x_i, y_i)\}_{i=1}^N$ \\
$\mathcal{X}$ & Collection of inputs of $\mathcal{D}$, $\mathcal{X}=\{x_i\}_{i=1}^N$ \\
$\mathcal{Y}$ & Collection of labels of $\mathcal{D}$, $\mathcal{Y}=\{y_i\}_{i=1}^N$ \\
$x \in \mathbb{R}^D$ & Input feature vector, $D$ is the feature dimension \\
$y \in \mathbb{R}$ & Continuous target variable \\
$\phi(x; \psi)$ & Deterministic neural backbone output (features) \\
$\psi$ & Deterministic parameters (weights) of the backbone network \\ % NEW ROW ADDED
$w$ & Stochastic weights of the final linear layer \\
$\text{vec}(w)$ & Flattened 1-D vector of $w$ \\
\midrule
\multicolumn{2}{l}{\textit{Variational Inference}} \\
$p(w)$ & Prior distribution over weights (typically $\mathcal{N}(0, I)$) \\
$q(w)$ & Variational approximate posterior distribution \\
$\overline{w}, S$ & Variational parameters (mean vector and covariance) \\
$z$ & Pre-activation logits, $z = w^\top \phi(x)$ \\
$m$ & Mean of logits under $q(w)$, $m = \overline{w}^\top \phi(x)$ \\
$s^2$ & Variance of logits (Epistemic), $s^2 = \phi(x)^\top S \phi(x)$ \\ % FIXED: Changed sigma^2 to s^2
\midrule
\multicolumn{2}{l}{\textit{Path A: Quantile (QR)}} \\
$\kappa$ & Quantile level, $\kappa \in (0, 1)$ \\
$\rho_\kappa(\cdot)$ & Pinball loss function \\
$\sigma$ & Scale parameter of the Asymmetric Laplace (set to 1) \\ % FIXED: Changed alpha to sigma
$\hat{q}_\kappa(x)$ & Predicted conditional quantile at level $\kappa$ \\
${q}_\kappa(x)$ & True conditional quantile at level $\kappa$ \\
$\hat{F}(y|x)$ & Reconstructed Cumulative Distribution Function (CDF) \\
$\mathcal{L}_{\text{QR}}$ & Expected Pinball Loss objective (deterministic ELBO) \\
\midrule
\multicolumn{2}{l}{\textit{Path B: Classification Restoration (CR)}} \\
$K$ & Number of discretization bins \\
$\mathcal{B}$ & Set of fixed bin centers, $\mathcal{B} = \{b_1, \dots, b_K\}$ \\
$B_k$ & The $k$-th bin \\
$\Delta$ & Bin width ($(y_{\max} - y_{\min}) / K$) \\
$\mathbf{p}(x)$ & Predicted probability vector over bins (Softmax output) \\
$\mathcal{L}_{\text{CR}}$ & Expected Negative Log-Likelihood objective for CR \\
\midrule
\multicolumn{2}{l}{\textit{Theoretical Analysis}} \\
$\delta_{\text{ood}}$ & Distance of OOD input from training manifold ($\min_{x \in \mathcal{D}_{train}} \|x' - x\|_2$) \\ % REC: Renamed to avoid clash
$R$ & Radius of the training feature manifold ($\max \|\phi(x)\|_2$) \\
$\lambda_{\min}(S)$ & Minimum eigenvalue of the variational covariance matrix \\
$L$ & Lipschitz constant of the true density \\
$L_f$ & Lipschitz constant of the backbone $f$ \\
$c_1, c_2$ & Constants for lower and upper bounds for Bi-Lipschitz condition (Assumption \ref{asm:spectral_norm})\\
\bottomrule
\end{tabular}
\end{table}

\begin{table}[h]
\centering
\small
\caption{Table of Abbreviations.}
\label{tab:abbreviations}
\begin{tabular}{@{} l p{0.75\columnwidth} @{}}
\toprule
\textbf{Abbreviation} & \textbf{Full Name} \\ \midrule
\multicolumn{2}{l}{\textit{Methods \& Models}} \\
VBLL & Variational Bayesian Last Layer \\
CR-VBLL & Classification Restoration VBLL (Ours) \\
QR-VBLL & Quantile Regression VBLL (Ours) \\
MDN & Mixture Density Networks \\
MCD & Monte Carlo Dropout \\
DER & Deep Evidential Regression \\
SNGP & Spectral-normalized Neural Gaussian Process \\
CQR & Conformalized Quantile Regression \\
BALD & Bayesian Active Learning by Disagreement \\
\midrule
\multicolumn{2}{l}{\textit{Metrics \& Objectives}} \\
NLL & Negative Log-Likelihood \\
MSE & Mean Squared Error \\
RMSE & Root Mean Squared Error \\
CRPS & Continuous Ranked Probability Score \\
ECE & Expected Calibration Error \\
ELBO & Evidence Lower Bound \\
KL & Kullback–Leibler Divergence \\
\midrule
\multicolumn{2}{l}{\textit{Concepts \& Distributions}} \\
UQ & Uncertainty Quantification \\
OOD & Out-of-Distribution \\
CDF & Cumulative Distribution Function \\
PDF & Probability Density Function \\
ALD & Asymmetric Laplace Distribution \\
PIT & Probability Integral Transform \\
SN & Spectral Normalization \\
\bottomrule
\end{tabular}
\end{table}

\clearpage
\section{Quantile Regression VBLL Algorithmic Details}
\label{app:vbll_qr}

% TODO: recover after abstarct submission
% \todo[inline]{TODO: connect this to the rest of the paper}

\subsection{Background and Related Work}
\label{app:vbll_qr_background}

Quantile regression \cite{koenker1978regression} uses the pinball loss function 
\begin{equation}
\label{eq:pinball-def}
\rho_{\kappa}(\delta) = \begin{cases}
\kappa \delta & \delta \geq 0\\
 (\kappa - 1) \delta & \delta < 0
\end{cases}
\end{equation}   
where $\delta$ is error $\delta = y - z$, for true label y and prediction $z$, and where $\kappa \in (0,1)$ corresponds to the quantile to be regressed. The asymmetric Laplace distribution\footnote{Note that we discuss zero-mean asymmetric Laplace distributions here. Because the asymmetric Laplace distribution is a member of the location-scale family, the density of the error under the zero mean distribution corresponds exactly to the density of the label with non-zero mean.},
\begin{equation*}
p(\delta; \sigma, \kappa) = \frac{\kappa(1-\kappa)}{\sigma} \exp(-\rho_{\kappa}(\frac{\delta}{\sigma}))
\end{equation*}   
is a natural probabilistic interpretation of the quantile loss function \cite{yu2001bayesian}; choosing scale parameter $\sigma = 1$ recovers exactly the quantile regression loss in maximum likelihood estimation (plus a constant function of $\kappa$). 

Bayesian inference for quantile estimation has primarily relied on Monte Carlo methods \citep{tsionas2003bayesian, li2010bayesian}. See \citet{yang2016posterior} for a survey of inferential methods. Most relevant to our work is the variational inference scheme of \cite{abeywardana2015variational}, which builds on the \citet{kozumi2011gibbs} formulation in developing a correspondence between the asymmetric Laplace distribution and a scalar mixture of Gaussians (with a particular distribution over mixture weights). They use Gaussian processes for their underlying function approximation. \citet{abeywardana2015variational} exploit the Gaussian mixture perspective, introducing a variational posterior for the mixture weights. 
Our variational approach exactly computes the ELBO for the asymmetric Laplace, avoiding the need to introduce a variational posterior that may differ from the true posterior. 

Regressing a set of quantiles (with e.g.~$(\kappa_1, \kappa_2, \ldots, \kappa_N)$ with $\kappa_i < \kappa_j$ for $i < j$) allows the construction of an estimate of the predictive density. Indeed, knowledge of the quantile locations and corresponding $\kappa$ values corresponds to a set of points on the CDF; construction of an empirical CDF or PDF relies on an interpolation method between these points. Linear interpolation between these points is a common choice, and results in a piecewise-constant density \cite{koenker2005quantile}. Bayesian inference of quantiles can therefore be thought of as generating a distribution over empirical CDFs, although the realized distribution still relies on interpolation method.

% However, this method is sensitive to quantile crossing, in which the quantiles do not follow the positional ordering required by definition. 
% An alternative approach is Kernel density estimation (KDE), in which each quantile center is the center of a mixture element. In the case of a Gaussian kernel, for example, this corresponds to interpolating the CDF with Gaussian CDF functions. While use of the Epanechnikov kernel is common \cite{gaglianone2012constructing}, we will focus on the Gausrsian kernel in this work. 
% While KDE avoids practical issues with quantile crossing, the setting of kernel hyperparameters can be sensitive. 

% TODO: recover after abstarct submission
% TODO: add discussion around related work
% \begin{itemize}
%     \item \cite{abeywardana2015variational} variational inference
%     \item 
% \end{itemize}

\subsection{ELBO and Predictive Likelihood}
\label{appx:proof-thm1}

We will first provide proofs of Theorem \ref{thm:qr_elbo} and Proposition \ref{thm:qr_predictive}, and then provide context for these results. 

\begin{figure}[htbp]
    \centering
    \includegraphics[width=0.8\linewidth]{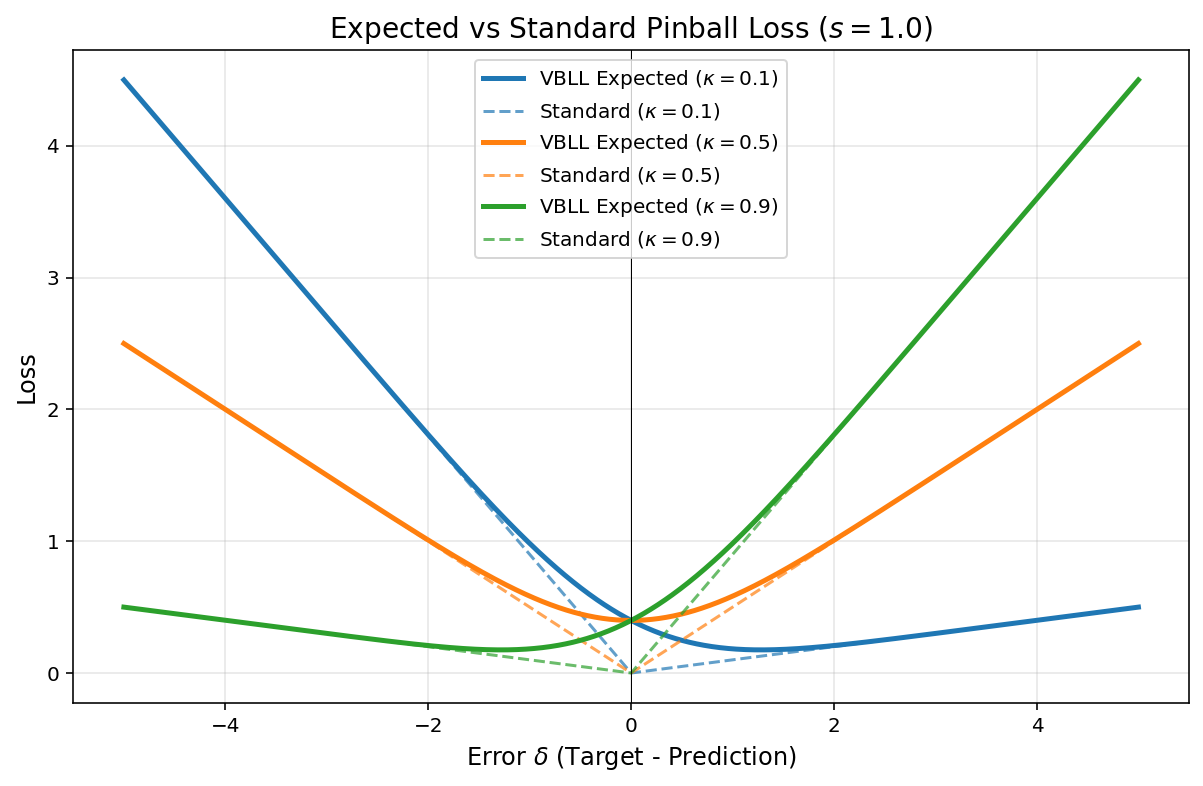}
    \caption{Analytic smoothing of the Asymmetric Laplace surrogate likelihood. The standard pinball loss (dashed) contains a non-differentiable kink at zero error. By taking the expectation over the Gaussian variational posterior in QR-VBLL, the effective data objective (solid) is analytically smoothed proportional to the epistemic uncertainty $s$. This renders the objective strictly twice-differentiable.}
    \label{fig:expected_pinball_loss}
\end{figure}

\begin{proof}[Proof of Theorem \ref{thm:qr_elbo}]
We have 
\begin{equation}
    \mathbb{E}[\log p(\delta; \sigma, \kappa)] = \log \frac{\kappa (1-\kappa)}{\sigma} - \mathbb{E}[\rho_{\kappa}(\frac{\delta}{\sigma})]
\end{equation}
for which
\begin{align}
    \mathbb{E}[\rho_{\kappa}(\frac{\delta}{\sigma})] &= \int_{-\infty}^{\infty} N(\delta; \bar{\delta}, s^2) \rho_{\kappa}(\frac{\delta}{\sigma}) d\delta \\
    &=\frac{\kappa}{\sigma} \int^{\infty}_0  \delta \mathcal{N}(\delta; \bar{\delta}, s^2)  d\delta + \frac{(\kappa - 1)}{\sigma} \int_{-\infty}^0 \delta \mathcal{N}(\delta; \bar{\delta}, s^2) d\delta
\end{align}
This is a known integral for which
\begin{equation}
    I \coloneq \int^{\infty}_0 \delta \mathcal{N}(\delta; \bar{\delta}, s^2)  d\delta = \bar{\delta} \Phi(\frac{\bar{\delta}}{s}) + s \mathcal{N}(\frac{\bar{\delta}}{s}; 0, 1)
\end{equation}
where $\Phi(\cdot)$ denotes the CDF of the standard normal.
\jwnew{
By symmetry, 
\begin{align}
    \int_{-\infty}^0 \delta \mathcal{N}(\delta; \bar{\delta}, s^2) d\delta
    &= - \int^{\infty}_0 (-\delta)\mathcal{N}(-\delta; \bar{\delta}, s^2) d\delta\\
    &= \bar{\delta} \Phi(-\frac{\bar{\delta}}{s}) - s \mathcal{N}(-\frac{\bar{\delta}}{s}; 0, 1)\\
    &= \bar{\delta} (1 - \Phi(\frac{\bar{\delta}}{s})) - s \mathcal{N}(\frac{\bar{\delta}}{s}; 0, 1).
\end{align}
}
Then, grouping terms, we have 
\begin{align}
    \mathbb{E}[\rho_{\kappa}(\frac{\delta}{\sigma})] &= \frac{\kappa}{\sigma} (\bar{\delta} \Phi(\frac{\bar{\delta}}{s}) + s \mathcal{N}(\frac{\bar{\delta}}{s}; 0, 1)) + \frac{(\kappa - 1)}{\sigma} (\bar{\delta} (1 - \Phi(\frac{\bar{\delta}}{s})) - s \mathcal{N}(\frac{\bar{\delta}}{s}; 0, 1))\\
    &= \frac{\bar{\delta}}{\sigma} (\kappa - 1 + \Phi(\frac{\bar{\delta}}{s})) + \frac{s}{\sigma} \mathcal{N}(\frac{\bar{\delta}}{s}; 0, 1)
\end{align}
\end{proof}

% Thus, the expected log likelihood can be computed analytically, and can be represented in terms of functions that are typically available in automatic differentiation packages. The variational posterior and the neural network weights can be trained by backpropagation through this loss function. 

% Prediction in this model is done by marginalizing over the variational posterior, computing 
% \begin{equation}
%     p(y \mid x) = \mathbb{E}_{q(w)}[p(y \mid x, w)].
% \end{equation}
% This marginal predictive distribution is analytically tractable, as shown by the following result. 

\begin{proof}[Proof of Proposition \ref{thm:qr_predictive}]
We will take the same approach as previously, and split the expectation into two integrals. We note
\begin{align}
    \mathbb{E}_{q(w)}[p(y \mid x, w)] &= \int_{-\infty}^\infty \frac{\kappa (1-\kappa)}{\sigma} \exp(-\rho_{\kappa}(\frac{\delta}{\sigma})) \mathcal{N}(\delta; \bar{\delta}, s^2) d\delta\\
    &=-\sigma \lambda_- \lambda_+ \left(  \int_{-\infty}^0 \exp(-\lambda_- \delta) \mathcal{N}(\delta; \bar{\delta}, s^2) d\delta + \int^{\infty}_0 \exp(- \lambda_+\delta) \mathcal{N}(\delta; \bar{\delta}, s^2) d\delta \right)
\end{align}
for which
\begin{align}
    \int_{-\infty}^0\exp(-\lambda_- \delta) \mathcal{N}(\delta; \bar{\delta}, s^2) d\delta &= \int_{-\infty}^0 \frac{1}{s \sqrt{2 \pi}} \exp(- \lambda_- \delta - \frac{1}{2s^2}(\delta - \bar{\delta})^2) d\delta\\
    &= \exp(- \lambda_- \bar{\delta} + \frac{1}{2} \lambda_-^2 s^2) \int_{-\infty}^0 \mathcal{N}(\delta; \bar{\delta} - \lambda_- s^2, s^2) d\delta\\
    &=\exp(- \lambda_- \bar{\delta} + \frac{1}{2} \lambda_-^2 s^2) \Phi(- \frac{ \bar{\delta} - \lambda_- s^2}{s})
\end{align}
where the second equality is arrived at by completing the square.
The RHS integral can be treated similarly, yielding the marginal density. 
\end{proof}

% This density is a mixture of exponentially-modified Gaussian distributions. 
We can also state the following result on the predictive mean and variance. 

\begin{corollary} \label{coro:qr_vbll_moments}
    Let $p(y|x)$ be the marginal distribution from Proposition \ref{thm:qr_predictive}. Then, the first two moments are 
    \begin{align}
        \mathbb{E}[y \mid x] &= \bar{w}^\top \phi(x) + \frac{\sigma (1-2\kappa)}{\kappa (1 - \kappa)} \label{eq:qr_mean}\\
        \textnormal{Var}(y \mid x) &= \phi(x)^\top S \phi(x) + \frac{\sigma^2 (1-2\kappa + 2\kappa^2)}{\kappa^2 (1 - \kappa)^2}. \label{eq:qr_var}
    \end{align}
\end{corollary}

\begin{proof}
    We can write $y = z + \delta$, where $z \sim \mathcal{N}(\bar{w}^\top \phi(x), \phi(x)^\top S \phi(x))$ and $\delta$ is asymmetric Laplace distributed with zero location parameter, and these are independent. Thus, the moments of the sum are the sum of the moments. The first term in each of Equations \ref{eq:qr_mean} and \ref{eq:qr_var} are the corresponding moments of the Gaussian, and the second term corresponds to the moments of the asymmetric Laplace with location zero. 
\end{proof}

The VBLL model structure pairs naturally with the asymmetric Laplace likelihood, and yields analytically computable expectations for both the ELBO and the predictive density. Further extensions to this model are possible: because the scale term $\sigma$ appears only through $\log \sigma$ and $1/\sigma$ in the ELBO, variational inference for this parameter is possible. 
% As discussed by \cite{harrison2026heteroscedastic}, 
We note that (inverse) gamma or log-Gaussian variational posteriors are possible, providing a more robust model and enabling modeling of heteroscedasticity. We compare the expected vs standard pinball loss in Figure \ref{fig:expected_pinball_loss}.

% \clearpage

% \subsection{Asymptotic Consistency}

% We first establish that our distribution-free estimators recover the true target distribution in the infinite-data limit. These results rely on the regularity of the target (Assumption \ref{asm:smooth_dist}) and the capacity of the model (Assumption \ref{asm:model_capacity}).

% \begin{proposition}[Asymptotic Consistency of VBLL Estimators]
% \label{prop:asymptotic_consistency}
% As the dataset size $N \to \infty$, the KL-divergence regularization term in the ELBO vanishes relative to the data likelihood. This causes the variational posterior $q(w)$ to contract to the optimal point estimate, guaranteeing that both branches asymptotically recover their respective ground-truth statistical targets (detailed proofs provided in Appendix \ref{app:proof-consistency}):

% \textbf{Path A (CR-VBLL): Density Estimation.} The predicted conditional distribution converges in probability to the true target density, eliminating the KL-divergence gap:
%     \begin{equation}
%         D_{\text{KL}} \left(p(y|x) \| \hat{p}_{\text{VBLL}}(y|x) \right) \xrightarrow{P} 0
%     \end{equation}

% \textbf{Path B (QR-VBLL): Quantile Estimation.} The predicted conditional quantiles converge in probability to the true target quantiles for any asymmetry level $\kappa \in (0,1)$:
%     \begin{equation}
%         | \hat{Q}_{\text{VBLL}}(\kappa|x) - Q^*(\kappa|x)| \xrightarrow{P} 0
%     \end{equation}

% \end{proposition}
% \clearpage

\clearpage

\section{Formal Proofs in Section \ref{sec:theory}}
\label{app:proofs}

Before detailing the derivations, we first formalize the structural assumption regarding the variational posterior of the last-layer weights, which is central to the $\mathcal{O}(1)$ efficiency of CR-VBLL. 

\begin{assumption}[Isotropic Logit Covariance via Homogeneous Discretization]
\label{asm:isotropic}

Following the exact factorization of the Variational Bayesian Last Layer \citep{harrison2024variational}, we model the variational posterior over the $K$ classes independently as $q(W) = \prod_{k=1}^K \mathcal{N}(\bar{w}_k, S_k)$. The independence of the weight vectors $w_k$ mathematically guarantees a strictly diagonal logit covariance matrix $\Sigma_Z = \text{diag}(s_1^2(x), \dots, s_K^2(x))$.

Crucially, because our $K$ classes do not represent heterogeneous, independent categories (as in standard classification), but rather contiguous, discretized intervals of a single continuous target variable (e.g., watch time), we impose a physically motivated structural assumption: the underlying epistemic uncertainty is homogeneous across these adjacent bins. We therefore constrain the class-specific posteriors to share the same covariance matrix $S$ (i.e., $S_k = S, \forall k$). This structural homogenization naturally yields an isotropic logit covariance $\Sigma_Z = s^2(x)\mathbf{I}_K$, where the shared scalar variance is $s^2(x) = \phi(x)^\top S \phi(x)$.
\end{assumption}

\begin{remark}[Efficiency Trade-off and Fallback to Diagonal Covariance]
It is important to note that relaxing Assumption \ref{asm:isotropic} to a standard diagonal covariance (i.e., allowing $S_k \neq S_j$) would still provide valid and rigorous uncertainty quantification. In fact, if we fallback to this weaker diagonal assumption, our formulation perfectly matches the scalable Laplace approximations established by prior works \citep{kristiadi2020being, harrison2024variational}, and we would derive the exact same UQ properties.

However, maintaining distinct variances $s_k^2(x)$ across classes introduces a critical computational and predictive trade-off. Mathematically, heterogeneous diagonal elements preclude a closed-form solution for the multi-class softmax expectation, necessitating computationally expensive Monte Carlo (MC) sampling. Furthermore, applying class-specific variance scaling inside the softmax can perturb the predictive argmax of the base model. By restricting our variational family to the isotropic space, we consciously trade the unnecessary degrees of freedom in bin-specific variances for an elegant $\mathcal{O}(1)$ closed-form analytic inference that strictly preserves the original mode precision.
\end{remark}

Under Assumption \ref{asm:isotropic}, the $K \times K$ epistemic covariance matrix of the predictive distribution is analytically approximated via the Delta method as $\tilde{J} \Sigma_Z \tilde{J}^\top$, where {$\tilde{J} = \text{diag}(\tilde{\mathbf{p}}) - \tilde{\mathbf{p}} \tilde{\mathbf{p}}^\top$ is the Jacobian of the softmax evaluated at the probit-scaled logits $\tilde{\mu} = \mu / \sqrt{1 + \frac{\pi}{8} s^2(x)}$, and $\tilde{\mathbf{p}} = \text{Softmax}(\tilde{\mu})$}.

The dimensionally aligned scalar epistemic variance $v_{epi}(x)$ is simply the trace of this matrix:
\begin{equation}
    v_{epi}(x) = \text{Tr}\left( \tilde{J} \Sigma_Z \tilde{J}^\top \right) = s^2(x)\text{Tr}\left( \tilde{J} \tilde{J}^\top \right) = s^2(x)\|\tilde{J}\|_F^2
\end{equation}

Expanding the trace of the squared Jacobian yields a fully analytic, closed-form scalar variance that solely depends on the \jwnew{scaled} predictive probabilities and our epistemic scalar $s^2(x)$:
\begin{equation}\label{eq:cr_epistemic}
    v_{epi}(x) = s^2(x) \left( \sum_{k=1}^K \tilde{p}_{k}^2 - 2 \sum_{k=1}^K \tilde{p}_{k}^3 + \left( \sum_{k=1}^K \tilde{p}_{k}^2 \right)^2 \right)
\end{equation}
% Under Assumption \ref{asm:isotropic}, the $K \times K$ epistemic covariance matrix of the predictive distribution is analytically approximated via the Delta method as $J_\mu \Sigma_Z J_\mu^\top$, where $J_\mu = \text{diag}(\mathbf{p}_\mu) - \mathbf{p}_\mu \mathbf{p}_\mu^\top$ is the Jacobian of the softmax evaluated at the mean logit $\mu = \overline{w}^\top \phi(x)$. 

% The dimensionally aligned scalar epistemic variance $v_{epi}(x)$ is simply the trace of this matrix:
% \begin{equation}
%     v_{\text{epi}}(x) = \text{Tr}\left(J_\mu \Sigma_Z J_\mu^\top\right) = s^2(x) \text{Tr}\left(J_\mu J_\mu^\top\right) = s^2(x) \|J_\mu\|_F^2
% \end{equation}
% Expanding the trace of the squared Jacobian yields a fully analytic, closed-form scalar variance that solely depends on the base model's predictive probabilities and our epistemic scalar $s^2(x)$:
% \begin{equation}
%     v_{\text{epi}}(x) = s^2(x) \left( \sum_{k=1}^K p_{\mu, k}^2 - 2 \sum_{k=1}^K p_{\mu, k}^3 + \left( \sum_{k=1}^K p_{\mu, k}^2 \right)^2 \right)
% \end{equation}

\begin{remark}
    Following \citet{harrison2024variational}, we apply Jensen's inequality to upper-bound the expected log-sum-exp term in $\mathbb{E}_{q}[\text{NLL}_{CE}(z, y)]$. This guarantees a fully analytic and deterministic surrogate objective for training:
    \begin{equation}
        \mathbb{E}_{q}[\text{NLL}_{CE}(z, y)] \le -m_{y} + \log \sum_{k=1}^K \exp\left(m_k + \frac{s^2(x)}{2}\right)
    \end{equation}
    where $m_y$ denotes the mean logit of the ground truth bin.
\end{remark}

\subsection{Proof of Proposition~\ref{prop:decomposition} (Analytic Decomposition)}
\label{app:proof_decomposition}

\begin{proof}
Let $p(y|x) = \int p(y|x, w) q(w) dw$ be the marginal predictive distribution. The Law of Total Variance states:
\[
\mathrm{Var}_q(y|x) = \underbrace{\mathbb{E}_{q(w)}[\mathrm{Var}(y|x, w)]}_{\text{Aleatoric}} + \underbrace{\mathrm{Var}_{q(w)}(\mathbb{E}[y|x, w])}_{\text{Epistemic}}
\]
where the moments are taken with respect to the variational posterior $q(w)$.

\textbf{Analytic Tractability in VBLL:}
The decomposition is computable in a single forward pass due to the closed-form moments of the logits $z = w^\top \phi(x)$. Under the Gaussian posterior, the logit variance is computed analytically as $s^2 = \phi(x)^\top S \phi(x)$. We apply this to both paths:

\begin{enumerate}
    \item 
    \textbf{Path A: QR-VBLL.} 
    
    For QR-VBLL, the surrogate likelihood is the Asymmetric Laplace Distribution (ALD). As derived in Corollary \ref{coro:qr_vbll_moments}, the predictive random variable for a specific quantile $\kappa$ is the sum of the latent Gaussian logits $z$ and the ALD noise $\epsilon$. By linearity of variance:
    $$ \text{Var}_\kappa(y|x) = \text{Var}(\epsilon) + s^2(x) $$
    Crucially, because the ALD is utilized as a surrogate likelihood to recover the pinball loss, its scale parameter $\sigma$ must be fixed (e.g., $\sigma=1$). Consequently, the theoretical ALD variance $\text{Var}(\epsilon)$ evaluates to a static constant depending only on $\kappa$. We explicitly note that this static term is an artifact of the surrogate objective and does not represent the true input-dependent aleatoric noise of the data. 
    
    % {Aleatoric Variance in QR-VBLL:} To quantify the actual, multi-modal aleatoric variance of the target variable $y$, the variance must be empirically computed from the global reconstructed CDF. Given the predicted conditional quantiles $\hat{q}_m(x)$ for an ordered set of $M$ quantile levels $\{\kappa_m\}_{m=1}^M$, we assume a piecewise uniform density between adjacent quantiles. The continuous predictive aleatoric variance is then explicitly computed by integrating over the inverse CDF:
    % \begin{equation}
    %     \label{eq:qr_aleatoric}
    %     v_{\text{aleatoric}}(x) = \text{Var}(y|x) \approx \sum_{m=1}^{M-1} (\kappa_{m+1} - \kappa_m) \left( \frac{\hat{q}_m^2 + \hat{q}_m \hat{q}_{m+1} + \hat{q}_{m+1}^2}{3} \right) - \hat{\mu}(x)^2
    % \end{equation}
    % where $\hat{\mu}(x) = \sum_{m=1}^{M-1} \frac{\hat{q}_m + \hat{q}_{m+1}}{2} (\kappa_{m+1} - \kappa_m)$ is the expected predictive mean. This analytic formulation strictly measures the inherent metric dispersion of the data distribution, satisfying the exact requirement for the aleatoric penalty in our hybrid active learning acquisition function (Eq.~\ref{eq:hybrid_acquisition}). Conversely, the epistemic term in this decomposition strictly and accurately captures the closed-form posterior logit variance $s^2(x)$ in $\mathcal{O}(1)$ time.
    
     We compute from the global reconstructed CDF To quantify the actual, multi-modal aleatoric variance of the target variable $y$, (derived from the full set of $K$ predicted quantiles). 
     \begin{equation}
        \label{eq:qr_aleatoric}
        v_{ale}(x) = \sum_{k=1}^{K-1} (\kappa_{k+1} - \kappa_k) \left( \frac{\hat{q}_k^2 + \hat{q}_k \hat{q}_{k+1} + \hat{q}_{k+1}^2}{3} \right) - \hat{\mu}(x)^2,
    \end{equation}
    \jwnew{where $\hat{\mu}(x) = \sum_{k=1}^{K-1} \frac{\hat{q}_k + \hat{q}_{k+1}}{2} (\kappa_{k+1} - \kappa_k)$ is the expected predictive mean.}
     The epistemic term in this decomposition strictly and accurately captures the closed-form posterior logit variance $s^2(x)$ in $\mathcal{O}(1)$ time.
    
    % \textbf{Path A: Quantile Regression (Exact).}
    % For QR-VBLL, the likelihood is the Asymmetric Laplace Distribution (ALD). As derived in \textbf{Corollary \ref{coro:qr_vbll_moments}}, the predictive random variable is the sum of two independent variables: the latent Gaussian logits $z$ and the ALD noise $\epsilon$. By the linearity of variance, the moments decompose \emph{exactly} without approximation:
    % \[
    % \mathrm{Var}(y|x) = \underbrace{\mathbb{E}_q[\text{Var}(\epsilon)]}_{\text{Aleatoric}} + \underbrace{s^2}_{\text{Epistemic}}
    % \]
    % where the aleatoric term is determined by the ALD scale $\sigma$ and asymmetry $\kappa$ \citep{yu2001bayesian}, and the epistemic term is exactly the posterior logit variance $s^2$.

    \item \textbf{Path B: CR-VBLL.}

    For CR-VBLL, $(y|x, w)$ follows a histogram distribution over $K$ bins. To strictly avoid the ``Ghost Value'' pathology—where computing continuous variance using bin midpoints artificially inflates uncertainty on multi-modal targets—we quantify uncertainty completely within the $K$-dimensional probability space.
    
    Let $\mathbf{p} = \text{Softmax}(w^\top \phi(x)) \in \mathbb{R}^K$ denote the conditional categorical probability vector given a specific weight $w$. Because $w \sim q(w)$ is a random variable under the variational posterior, $\mathbf{p}$ itself is a random vector. By the Law of Total Covariance, the total covariance matrix of the categorical prediction decomposes into the expected conditional covariance and the covariance of the expectation. To project this into a unified scalar metric and formally justify the variance notation in Proposition 2, we apply the Trace ($\text{Tr}$) operator, which preserves the linear additive decomposition.
    
    \textbf{Aleatoric Uncertainty:} The conditional covariance matrix of a categorical variable is $\text{diag}(\mathbf{p}) - \mathbf{p}\mathbf{p}^\top$. Applying the Trace operator yields the scalar aleatoric term, which resolves exactly to the expected Gini impurity of the distribution:
    \begin{equation}
        v_{ale}(x) = \mathbb{E}_q\left[\text{Tr}(\text{diag}(\mathbf{p}) - \mathbf{p}\mathbf{p}^\top)\right] = \mathbb{E}_q\left[1 - \mathbf{p}^\top \mathbf{p}\right]
    \end{equation}
    This metric cleanly captures the intrinsic data noise (e.g., dispersion across multiple modes) mathematically, without relying on the continuous semantic distance between bins.
    \knew{To compute the aleatoric variance $v_{ale}(x) = \mathbb{E}_q \left[1 - \mathbf{p}^\top \mathbf{p} \right]$ analytically in $\mathcal{O}(1)$ time, we evaluate the expected Gini impurity using the approximated predictive mean. Let the latent logits be distributed as $g(x) \sim \mathcal{N}(\mu, \Sigma_s)$, where $\mu = \bar{w}^\top \phi(x)$ and $\Sigma_s = \phi(x)^\top \Sigma_w \phi(x)$. We first approximate the expected categorical probability $\hat{p}_k = \mathbb{E}_q[p_k]$ by mapping the softmax function to a Bradley-Terry style pairwise probit formulation \citep{kristiadi2020being}:
    \begin{equation}
    \hat{p}_k \approx \frac{\exp\left(\frac{\mu_k}{\sqrt{1+ \eta \sigma_k^2}}\right)}{\sum_{j=1}^K \exp\left(\frac{\mu_j}{\sqrt{1+ \eta \sigma_j^2}}\right)},
    \end{equation}
    where $\sigma_k^2 = [\Sigma_s]_{kk}$ is the marginal logit variance for class $k$ and $\eta = \pi/8$ is the standard probit scaling factor.
    The aleatoric uncertainty is then efficiently approximated as $v_{ale}(x) \approx (1 - \hat{\mathbf{p}}^\top \hat{\mathbf{p}}) - v_{epi}(x) = 1 - \sum_{k=1}^K \hat{p}_k^2 - v_{epi}(x)$.
    % The aleatoric uncertainty is then efficiently approximated as $v_{ale}(x) \approx 1 - \hat{\mathbf{p}}^\top \hat{\mathbf{p}} = 1 - \sum_{k=1}^K \hat{p}_k^2$.
    }

    \snew{
    \textbf{Epistemic Uncertainty:} The epistemic uncertainty is quantified by the trace of the predictive covariance of the softmax outputs: $\text{Tr}(\mathbb{E}_q[\mathbf{p}\mathbf{p}^\top] - \bar{\mathbf{p}}\bar{\mathbf{p}}^\top)$, where $\bar{\mathbf{p}} = \mathbb{E}_q[\mathbf{p}]$. 

Assuming the last-layer weights follow a Gaussian variational posterior $q(w) \sim \mathcal{N}(\overline{w}, S)$, the pre-softmax logits $\mathbf{z} = w^\top \phi(x)$ are a deterministic linear projection of these weights. Consequently, their conditional distribution $q(\mathbf{z}|x)$ is strictly Gaussian, $\mathcal{N}(m, \Sigma_z)$, where $m = \bar{w}^\top \phi(x)$. Under an isotropic variance assumption across classes, we have $\Sigma_z = s^2 \mathbf{I}_K$, where $s^2 = \phi(x)^\top S \phi(x)$.

During inference, we approximate the expected softmax outputs using the probit-softmax approximation \citep{kristiadi2020being}: 
\snew{$\hat{\mathbf{p}} \approx \text{Softmax}(\tilde{\mathbf{z}})$, where $\tilde{\mathbf{z}} = \frac{m}{\sqrt{1 + \eta s^2}}$} denotes the temperature-scaled mean logits, and $\eta$ is a scaling parameter. 
To obtain a tight, closed-form expression for the covariance, we apply the first-order multivariate Delta method. 

Standard application of the Delta method requires a Taylor expansion around the mean of the input, $m$. However, doing so causes the softmax outputs to collapse toward a one-hot distribution for OOD samples, leading to a severe underestimation of epistemic uncertainty. To circumvent this issue, we instead evaluate at the temperature-scaled mean logits \snew{$\tilde{\mathbf{z}}$}.

Specifically, let $J = \text{diag}(\hat{\mathbf{p}}) - \hat{\mathbf{p}}\hat{\mathbf{p}}^\top$ denote the Jacobian of the softmax function evaluated at $\tilde{\mathbf{z}}$. The $K \times K$ epistemic covariance matrix of the predictive distribution is analytically approximated as $J \Sigma_z J^\top$. The scalar epistemic variance $v_{\text{epi}}(x)$ is simply the trace of this matrix \footnote{
Albeit having size $K\times K$, the term $\text{Tr}\left( J J^\top \right)$ can be evaluated in $O(K)$ time. Throughout the paper, we treat $K$ as a constant as it is usually much smaller than the backbone model size. }:
\begin{equation} \label{eq:delta_variance}
    v_{\text{epi}}(x) = \text{Tr}\left( J \Sigma_z J^\top \right) = s^2(x) \text{Tr}\left( J J^\top \right).
\end{equation}

By expanding around $\tilde{\mathbf{z}}$, the approximate probabilities $\hat{\mathbf{p}}$ remain bounded away from a one-hot vector even for extreme OOD samples. This ensures the trace term is strictly lower-bounded, allowing the final epistemic variance to scale properly with $s^2(x)$. Ultimately, this yields a fully closed-form and principled scalar variance decomposition, $\text{Var}_{\text{total}}(x) = v_{\text{ale}}(x) + v_{\text{epi}}(x)$, which avoids Monte Carlo sampling and can be computed efficiently
}

\end{enumerate}

All components are deterministic and closed-form, allowing for $O(1)$ inference.
\end{proof}

\begin{remark} [The Categorical Gap in Histogram Uncertainty]
    In standard histogram regression, variance is typically computed using the continuous metric distance between bin midpoints. However, applying continuous metric variance to highly multi-modal distributions severely inflates the uncertainty score whenever distant modes co-exist (the "Ghost Value" pathology), conflating the natural multi-modality of the data with true epistemic ignorance or excessive aleatoric noise. 

    To break this conflation, CR-VBLL deliberately treats the histogram bins as an unordered categorical space $\mathbf{p} \in \mathbb{R}^K$. By quantifying uncertainty strictly within the probability simplex (e.g., via Gini impurity and Trace of the covariance matrix) and disregarding the ordinal distance between bins, CR-VBLL measures the \textit{sharpness of the distributional modes} rather than their spatial spread. This "classification-style" uncertainty ensures that sharp, distant multi-modal predictions are correctly assigned low uncertainty, fulfilling the "Mode Precision" property described in Remark \ref{rem:path_selection}.
\end{remark}

% \begin{remark}[Distinction between Exact and Approximate Decomposition] \label{rem:analytic_distinction}
% In Table \ref{tab:theory_comparison}, we distinguish between the two forms of analytic decomposition:
% \begin{itemize}
%     \item \textbf{QR-VBLL (Exact):} Achieves mathematical exactness via the properties of the ALD and Gaussian convolution (Corollary \ref{coro:qr_vbll_moments}).
%     \item \textbf{CR-VBLL:} Relies on the Probit approximation for the Softmax integral to avoid Monte Carlo sampling. 
% \end{itemize}
% \end{remark}

\subsection{Proof of Proposition~\ref{prop:ood} (OOD Sensitivity)}
\label{app:proof-ood}

\begin{proof}
Let $x^* = \arg\min_{x \in \mathcal{D}_{\text{train}}} \|x' - x\|_2$ be the training point closest to the OOD input $x'$, such that $\|x' - x^*\|_2 = d$.

\textbf{1. Lower Bounding the Feature Norm:}
By the lower bound of the bi-Lipschitz property in Assumption \ref{asm:spectral_norm}, the distance in feature space is bounded by the distance in input space:
\[ \|\phi(x') - \phi(x^*)\|_2 \geq c_1 \|x' - x^*\|_2 = c_1 d \]
Using the reverse triangle inequality, we can bound the norm of the OOD feature $\|\phi(x')\|_2$:
\[ \|\phi(x')\|_2 \geq \|\phi(x') - \phi(x^*)\|_2 - \|\phi(x^*)\|_2 \]
By defining $R = \max_{x \in \mathcal{D}_{\text{train}}} \|\phi(x)\|_2$ as the maximum norm of any training feature, we have:
\[ \|\phi(x')\|_2 \geq c_1 d - R \]

\textbf{2. Lower Bounding the Epistemic Variance:}
To lower bound the quadratic form $v(x') = \phi(x')^\top S \phi(x')$, we use the minimum eigenvalue of $S$:
\[
v(x') \ge \lambda_{\min}(S) ||\phi(x')||_2^2
\]
We analyze the lower bound in two regimes:
\begin{itemize}
    \item \textbf{Outside the Horizon ($d \ge R/c_1$):} In this case, $c_1 d - R \ge 0$. Since both sides of the norm inequality are non-negative, we can square the inequality while preserving the direction:
    \[
    v(x') \ge \lambda_{\min}(S) (c_1 d - R)^2
    \]
    This confirms the quadratic growth of uncertainty with distance $d$.
    
    \item \textbf{Inside the Horizon ($d < R/c_1$):} Here, the lower bound is trivial ($v(x') \ge 0$) as uncertainty is governed by the local inductive bias of the kernel rather than asymptotic distance.
\end{itemize}

\textbf{Conclusion:} The quadratic "Safety Net" guarantee $v(x') \propto d^2$ strictly holds for all inputs beyond the OOD horizon $d_{\text{OOD}} = R/c_1$.
\end{proof}

\subsection{Proof of Calibration \& Coverage} \label{app:calibration_proof}

\begin{proposition}[Calibration and Coverage via Marginalization (Simplified)]\label{prop:calibration}
The marginalization of the weight posterior in the VBLL framework acts as an implicit regularizer against overconfidence. By optimizing the ELBO (Eq.~\ref{eq:elbo_cr} / Eq.~\ref{eq:analytic_qr}), the predictive distributions satisfy the following properties:

\textbf{Path A (QR):} Marginalization induces an analytic smoothing of the Asymmetric Laplace likelihood (Theorem \ref{thm:qr_elbo}) into a mixture of exponentially-modified Gaussians (Proposition \ref{thm:qr_predictive}). Structurally, this forces the predictive variance to explicitly incorporate the epistemic uncertainty $s^2(x)$, ensuring that predictive intervals inflate on OOD inputs to prevent the zero-width collapse common in standard quantile regression.

\textbf{Path B (CR):} Marginalization induces an input-dependent temperature scaling, strictly reducing maximum confidence in regions of high epistemic uncertainty \citep{seo2019learning, kristiadi2020being}.

% \begin{itemize}[leftmargin=*, label={}]
%     \item \textbf{Path A (CR):} Marginalization induces an input-dependent temperature scaling, strictly reducing maximum confidence in regions of high epistemic uncertainty \citep{seo2019learning, kristiadi2020being}.
%     \item \textbf{Path B (QR):} The predictive distribution becomes a mixture of Asymmetric Laplace Distributions. By the Law of Total Variance, this inflates the predictive intervals on OOD inputs, preventing the over-tight estimates common in standard pinball loss minimization \citep{takeuchi2006nonparametric}.
% \end{itemize}
\end{proposition}

We derive these results separately for the Classification and Regression paths, where the calibration guarantees are presented along with the derivations.

\textbf{Part 1: Path A (QR-VBLL)} \\
For a specific quantile level $\kappa$, the marginal predictive distribution $p_\kappa(y|x)$ is a mixture of two exponentially-modified Gaussian (EMG) terms. This distribution arises from the analytic convolution of the Asymmetric Laplace likelihood (specific to $\kappa$) with the Gaussian variational posterior $q(w) \sim \mathcal{N}(\overline{w}, S)$.

\textbf{Formal Local Density Function:} Using the results in Proposition \ref{thm:qr_predictive}, let the logit mean be $m = \overline{w}^{\top}\phi(x)$ and the epistemic variance be $s^2 = \phi(x)^{\top}S\phi(x)$. The marginal density for quantile $\kappa$ is defined analytically as:
\begin{align}
    \label{eq:qr_vbll_density_kappa}
    p_\kappa(y|x) = -\sigma\lambda_{-}\lambda_{+} & \bigg[ \exp\left(-\lambda_{-}\overline{\delta} + \frac{1}{2}\lambda_{-}^{2}s^{2}\right)\Phi\left(-\frac{\overline{\delta}-\lambda_{-}s^{2}}{s}\right) \nonumber \\ 
    &~ + \exp\left(-\lambda_{+}\overline{\delta} + \frac{1}{2}\lambda_{+}^{2}s^{2}\right)\Phi\left(\frac{\overline{\delta}-\lambda_{+}s^{2}}{s}\right) \bigg]
\end{align}
Where:
$\overline{\delta} = y - m$ is the centered error. $\lambda_{+} = \frac{\kappa}{\sigma}$ and $\lambda_{-} = \frac{\kappa-1}{\sigma}$ are parameters derived from the quantile level $\kappa$ and scale $\sigma$. $\Phi(\cdot)$ is the standard normal CDF.

\textbf{Part 2: Path B (CR-VBLL)} \\
Let the predictive distribution be the expectation of the softmax function over the Gaussian posterior $q(w) = \mathcal{N}(\overline{w}, S)$. The probability for class $k$ is given by the integral:
\begin{equation}
    \hat{\mathbf{p}}_k(x) = \int \text{Softmax}(w^\top \phi(x))_k \, \mathcal{N}(w|\overline{w}, S) \, dw
\end{equation}
The analytic solution is intractable. We employ the probit approximation for the Softmax-Gaussian integral \citep{kristiadi2020being}, which approximates the convolution of a sigmoid-like function with a Gaussian by scaling the input. This yields:
\begin{equation}
    \hat{\mathbf{p}}_k(x) \approx \text{Softmax}\left( \frac{\overline{w}^\top \phi(x)}{\sqrt{1 + \eta s^2(x)}} \right)_k
\end{equation}
% where $s^2(x) = \phi(x)^\top S \phi(x)$ is the epistemic variance of the logits (defined in Eq. \ref{eq:analytic_marginalization}). Let $T(x) = \sqrt{1 + \lambda s^2(x)}$ be the induced temperature scaling.
where $s^2(x) = \phi(x)^\top S\phi(x)$ is the epistemic variance of the logits, and $\eta = \pi/8$ is the standard probit approximation constant \citep{bishop2006pattern}. Let $T(x) = \sqrt{1 + \eta s^2(x)}$ be the induced temperature scaling.

\begin{itemize}
    \item \textbf{Inequality Proof:} Since the covariance matrix $S$ is positive semi-definite, $s^2(x) \ge 0$ implies $T(x) \ge 1$. For any logit vector $z$, the operation $\text{Softmax}(z/T)$ with $T > 1$ strictly increases the entropy and decreases the maximum probability compared to $T=1$ (unless $z$ is uniform). Thus $\forall x \text{ s.t. } s^2(x) > 0$:
  \jwnew{  \begin{equation}
        \max_k \hat{p}_k(x) < \max_k \text{Softmax}(\mu(x))_k \quad \text{subject to} \quad \mu(x)\neq c\mathbf{1},
    \end{equation}
    where $\mathbf{1} \in \mathbb{R}^K$ is a vector of ones and $c\in\mathbb{R}$.}
    \item \textbf{OOD Limit:} From Proposition \ref{prop:ood}, as the distance from the training manifold increases ($d \to \infty$), the epistemic variance grows quadratically ($s^2(x) \to \infty$). 
    % \jwnew{To ensure this property acts as an OOD safety net, the predictive distribution must be evaluated without a first-order local linear approximation, which otherwise suffers from analytical variance collapse due to softmax saturation. By preserving the unscaled logit variance or applying an explicit distance-dependent scaling $T(x) \to \infty$, we ensure that $\lim_{d\rightarrow \infty}T(x)=\infty$. Consequently, the predictive distribution approaches a uniform distribution: $\lim_{d \to \infty} \hat{p}(x) = \text{Uniform}(K)$.} This maximum entropy state minimizes ECE on OOD data by refusing to make confident incorrect predictions.
    % \item \textbf{OOD Limit:} From Proposition \ref{prop:ood}, as the distance from the training manifold increases ($d \to \infty$), the epistemic variance grows quadratically ($s^2(x) \sim \mathcal{O}(d^2)$). 
    \knew{Meanwhile, the mean logits $\mu(x)$ grow at most linearly ($\mathcal{O}(d)$). Under the probit approximation, the induced temperature $T(x) = \sqrt{1 + \eta s^2(x)}$ also grows linearly ($\mathcal{O}(d)$). Consequently, the scaled logits $\mu(x)/T(x)$ converge to a finite constant vector as $d \to \infty$, rather than diverging to infinity. This fundamental property ensures that the marginal predictive distribution $\hat{p}(x)$ is strictly bounded away from a one-hot (Dirac delta) distribution.
    \snew{This further ensures that the epistemic uncertainty given by $s^2(x) \text{Tr} \left( 
    \left( \text{diag}( \hat p(x) ) - \hat p(x) \hat p(x)^\top \right)^2
    \right)$ (defined as in \eqref{eq:delta_variance}) is at least $\Omega(s^2(x))$.}
    By naturally tempering the logits with epistemic uncertainty, the model effectively avoids making overly confident predictions on extreme OOD data, acting as a reliable OOD safety net.}
\end{itemize}

\textbf{Theoretical Implications for Calibration:}

\begin{itemize}
    \item \textbf{Analytic Smoothing (Local):} Unlike standard quantile regression which minimizes the pinball loss and can collapse to ``hard'' or crossing quantiles (Dirac deltas), the EMG formulation naturally ``smooths'' the local density estimate $p_\kappa(y|x)$ based on the epistemic uncertainty $s^2$.
    
    % \item \textbf{Total Variance Decomposition (Local):} By the Law of Total Variance, the predictive variance for a specific quantile estimate decomposes analytically:
    % \begin{equation}
    %     Var_\kappa(y|x) = \underbrace{\frac{\sigma^{2}(1-2\kappa+2\kappa^{2})}{\kappa^{2}(1-\kappa)^{2}}}_{\text{Aleatoric (ALD Noise)}} + \underbrace{s^2(x)}_{\text{Epistemic (Model)}}
    % \end{equation}
    % This additive form ensures that even if the aleatoric noise is small, the variance cannot collapse to zero if model uncertainty $s^2(x)$ is present.
    \textbf{Variance Decomposition (Local vs. Global):} By the Law of Total Variance, the predictive variance for a specific quantile surrogate density decomposes analytically:
    \begin{equation}
        \text{Var}_\kappa(y|x) = \underbrace{\frac{\sigma^2(1 - 2\kappa + 2\kappa^2)}{\kappa^2(1-\kappa)^2}}_{\text{Aleatoric}} + \underbrace{s^2(x)}_{\text{Epistemic}}
    \end{equation}
    Because $\sigma$ is fixed to recover the pinball loss, the first term evaluates to a static constant. While this theoretical ALD noise cannot physically represent the input-dependent aleatoric variance of $y$ (which must instead be derived globally from the reconstructed CDF), this local additive form remains structurally critical. It mathematically ensures that the localized variance for any specific quantile cannot collapse to zero if epistemic model uncertainty $s^2(x)$ is present.
    
    % \item \textbf{Global OOD Safety Net:} On out-of-distribution (OOD) data, the distance-aware backbone ensures that $s^2(x)$ grows quadratically. Since this term is additive in the local variance for \textit{every} quantile head $\kappa$, the entire set of predicted quantiles $\{ \hat{q}_\kappa(x) \}_{\kappa \in K}$ is forced to disperse. This effectively widens the reconstructed global confidence intervals (e.g., $\hat{q}_{0.9} - \hat{q}_{0.1}$), preventing overconfidence in unknown regions.
    \item \textbf{Global OOD Safety Net:} As established in Remark \ref{rem:ood_calibration}, the deterministic posterior means of the individual predicted quantiles ($\hat{q}_\kappa(x)$) do not natively inflate or disperse strictly as a function of epistemic variance. However, the OOD ``safety net'' mechanism is robustly preserved through the marginalization process. As the epistemic variance $s^2(x)$ grows on OOD data, the predictive distributions \textit{around} each estimated quantile become highly uncertain. When constructing the continuous predictive density by integrating over these uncertain quantile estimates, the resulting marginal predictive density significantly flattens. This analytic smoothing prevents the model from outputting overly confident, sharp densities in OOD regions, effectively recovering the desired safety net behavior without requiring the underlying quantile point estimates to physically disperse.
\end{itemize}

\subsection{Proof of Asymptotic Consistency}
\label{app:proof-consistency}

\begin{assumption}[Smoothness of True Distribution \citep{scott1979optimal}]\label{asm:smooth_dist}
The true conditional density $p(y|x)$ is bounded and Lipschitz continuous with respect to $y$ with constant $L$, i.e., $|p(y_1|x) - p(y_2|x)| \leq L|y_1 - y_2|$ for all $y_1, y_2 \in \mathcal{Y}$.
\end{assumption}

\begin{remark}[Smoothness vs. Discreteness]
    \label{rem:smoothness} 
    As visualized in the empirical target distribution in {Figure \ref{fig:real_world_dist}}, user feedback often exhibits sharp transitions.    
    We treat these not as mathematical discontinuities (Dirac deltas), which would violate Assumption \ref{asm:smooth_dist}, but as high-density regions separated by steep, continuous gradients (implying a large but finite Lipschitz constant $L$). The Lipschitz assumption serves as a sufficient condition for our convergence proofs without restricting the model's practical ability to capture the highly multi-modal nature of real-world data.
\end{remark}

% \begin{assumption}[Model Capacity \citep{mohri2018foundations}] \label{asm:model_capacity}
% The neural network backbone $\phi(x)$ has sufficient capacity to act as a universal approximator (i.e., the network representation capacity is allowed to grow with the sample size $N$) and is trained with a proper scoring rule: Categorical Cross-Entropy for CR; Pinball Loss for QR, which implies consistency of the estimator in the infinite-data limit.
% \end{assumption}

\begin{assumption}[Model Capacity and Dimensionality]
\label{asm:model_capacity}
The neural network backbone $\phi(x) \in \mathbb{R}^D$ has sufficient capacity to act as a universal approximator, meaning the network representation capacity is allowed to grow with the sample size $N$. Crucially for CR-VBLL, as we take the number of histogram bins $K \to \infty$ to analyze asymptotic consistency, we assume the backbone feature dimensionality $D$ grows alongside $K$ such that $D \ge K-1$. This prevents the pre-activation logits from being restricted to a low-dimensional affine subspace, ensuring the model retains the degrees of freedom required to independently control the $K$ bin probabilities and match arbitrary target densities.
\end{assumption}

\begin{remark}[Overcoming the Softmax Bottleneck]
\label{rem:softmax_bottleneck}
Assumption \ref{asm:model_capacity} mathematically circumvents the well-known ``Softmax Bottleneck'' \citep{yang2017breaking}. If the number of discrete bins $K$ strictly exceeds the feature dimension $D+1$, the pre-activation logit vector is confined to a low-rank affine subspace. Consequently, the output softmax distribution loses the capacity to express arbitrary categorical distributions, violating the universal approximation guarantee. By explicitly coupling $D$ and $K$ (i.e., $D \ge K-1$), we ensure the network retains full representational rank over the discretized bins as $K \to \infty$.
\end{remark}

\subsubsection{Path A: QR-VBLL}

\begin{assumption}[Bernstein-von Mises Conditions for Variational Posterior]
\label{asm:bvm}
We assume that the data-generating process and the chosen variational family satisfy the conditions for the Bernstein-von Mises (BvM) theorem. Specifically, as $N \to \infty$, the variational posterior $q(w)$ exhibits asymptotic normality and contracts its probability mass around the true parameter $w^*$ that minimizes the expected Pinball Loss. 
\end{assumption}

While BvM guarantees generally fail for fully stochastic deep neural networks due to weight unidentifiability and singular Fisher Information matrices, our Variational Bayesian Last Layer (VBLL) framework circumvents this limitation. By performing approximate inference exclusively on the final linear layer conditioned on a deterministic feature backbone $\phi(x)$, the inference task reduces to a generalized linear model. Under standard regularity conditions and assuming full-rank features, this restricted parameter space is strictly identifiable, rendering the BvM assumption rigorously justified \citep{liu2020simple, van2000asymptotic}.

\begin{lemma}[Consistency of QR]\label{lemma:qr_base}
Let $\hat{q}_\kappa(x)$ be the estimator minimizing the expected Pinball Loss $\mathcal{L}_\kappa$. Provided the hypothesis space is a universal approximator, the excess risk converges to zero as $N \to \infty$:
\begin{equation}
    R[\hat{q}_\kappa] - R[q^*_\kappa] \le O(N^{-1/2}) \to 0
\end{equation}
This risk bound is established in \citet[Theorem 6]{takeuchi2006nonparametric}, while universal consistency to the true quantile function is proven in \citet[Theorem 3.1]{steinwart2011estimating}.
\end{lemma}

\begin{proposition} [Asymptotic Consistency for QR-VBLL.]
    Let $\hat{Q}_{VBLL}(\kappa|x)$ be the estimator induced by the VBLL variational posterior $q(w)$. {Under Assumption \ref{asm:bvm}}, as $N \to \infty$, the KL-divergence regularization term vanishes relative to the data likelihood (Eq. 4), causing the posterior $q(w)$ to contract to the optimal quantile regression solution. Consequently, the VBLL estimator inherits the universal consistency of the non-parametric quantile estimator (Lemma \ref{lemma:qr_base}), ensuring that the estimated quantiles converge to the true conditional quantiles:
\begin{equation}
|\hat{Q}_{VBLL}(\kappa|x) - Q^*(\kappa|x)| \xrightarrow{P} 0
\end{equation}
\end{proposition}

\begin{proof}
Let the true conditional quantile function at level $\kappa$ be $f^*_\kappa(x)$. We model the quantile function using the VBLL backbone as $f(x; w) = w^\top \phi(x)$. The VBLL objective (Eq. 4) maximizes the Evidence Lower Bound (ELBO):
\begin{equation}
\mathcal{L}_{QR} = \sum_{i=1}^N \mathbb{E}_{q(w)} \left[ \log p(y_i | x_i, w) \right] - \beta D_{KL}(q(w) || p(w))
\end{equation}
where the log-likelihood corresponds to the Asymmetric Laplace Distribution (ALD): $\log p(y|x,w) \propto -\rho_\kappa(y - f(x;w))$.

\textbf{Step 1: Convergence of the Variational Posterior.}
As the sample size $N \to \infty$, the data term scales with $N$, while the KL-divergence prior regularization term remains constant (or scales sub-linearly if $\beta$ is annealed). The objective is dominated by the sum of expected log-likelihoods:
\[
\frac{1}{N}\mathcal{L}_{QR} \approx \frac{1}{N}\sum_{i=1}^N \mathbb{E}_{q(w)} [-\rho_\kappa(y_i - w^\top \phi(x_i))].
\]
By Assumption \ref{asm:bvm}, the variational posterior $q(w)$ concentrates its probability mass around the parameter $w^*$ that minimizes the expected Pinball Loss (Maximum Likelihood Estimation for ALD). Consequently, under this assumption, the variance of $q(w)$ shrinks as $\mathcal{O}(1/N)$, and the mean $\mu_w$ converges to $w^*$.

\textbf{Step 2: Consistency of the Estimator.}
Since the backbone $\phi(x)$ satisfies the universal approximation property (Assumption \ref{asm:model_capacity}), there exists a sequence of parameters $w_N$ such that $w_N^\top \phi(x)$ converges to the true quantile function $f^*_\kappa(x)$ in the sup-norm.
From Lemma 4.8 in \citep{takeuchi2006nonparametric}, the minimizer of the Pinball Loss is a consistent estimator of the true quantile. Since our variational objective asymptotically coincides with the Pinball Loss minimization, the VBLL estimator $\hat{q}_\kappa(x) = \mathbb{E}_{q(w)}[w^\top \phi(x)]$ converges to the true minimizer.

\textbf{Conclusion.}
Combining the concentration of the posterior $q(w)$ and the consistency of the risk minimizer, we have:
\begin{equation}
\sup_{x \in \mathcal{X}} | \hat{q}_\kappa(x) - f^*_\kappa(x) | \xrightarrow{P} 0 \quad \text{as } N \to \infty
\end{equation}
This confirms that QR-VBLL is an asymptotically consistent estimator for the true conditional quantiles.
\end{proof}

\subsubsection{Path B: CR-VBLL}

\begin{lemma}[Consistency of Histogram Density Estimation]\label{lemma:cr_base}
Let $\hat{p}_{\text{CR}}(y|x; K)$ be a histogram density estimator with $K$ bins of width $\Delta$. If the true density $p(y|x)$ is Lipschitz continuous and bounded, then as $N \to \infty$ and $K \to \infty$ such that $\Delta \to 0$ and $K/N \to 0$, the Integrated Mean Squared Error (IMSE) converges to zero:
\begin{equation}
    \text{IMSE} = \mathbb{E} \int (\hat{p}_{\text{CR}}(y|x) - p(y|x))^2 dy = O(N^{-2/3}) \to 0
\end{equation}
This consistency is explicitly stated in \citet[Section 2]{scott1979optimal}, relying on the convergence proofs of \citet{cencov1962estimation}.
\end{lemma}

\begin{proposition} [Asymptotic Consistency for CR-VBLL.]
    Let $\hat{p}_{\text{VBLL}}(y|x)$ be the estimator induced by the VBLL variational posterior $q(w)$. As $N \to \infty$, the KL-divergence regularization term in the ELBO vanishes relative to the data likelihood, causing the posterior $q(w)$ to contract to the Maximum Likelihood Estimate (MLE).
    Consequently, the CR-VBLL estimator inherits the consistency of the base histogram estimator (Lemma \ref{lemma:cr_base}):
    \begin{equation}
        D_{\text{KL}}\left(p(y|x) \| \hat{p}_{\text{VBLL}}(y|x)\right) \xrightarrow{P} 0
    \end{equation}
\end{proposition}

\begin{proof}
Let the true conditional density be $p(y|x)$ defined on compact support $\mathcal{Y} = [y_{\min}, y_{\max}]$. \modelname{} discretizes $\mathcal{Y}$ into $K$ bins $B_1, \dots, B_K$ with centers $z_k$ and width $\Delta = (y_{\max} - y_{\min})/K$. The estimated density is given by the histogram estimator:
\[
\hat{p}_{\text{CR}}(y|x) = \sum_{k=1}^K \pi_k(x) \cdot \frac{\mathbb{I}\{y \in B_k\}}{\Delta},
\]
where $\pi_k(x) = \text{Softmax}_k(\phi(x))$ represents the probability mass assigned to bin $k$.

The Kullback-Leibler (KL) divergence between the true and estimated density decomposes as:
\begin{align*}
D_{\text{KL}}(p\|\hat{p}_{\text{CR}}) &= \int_{\mathcal{Y}} p(y|x) \log\frac{p(y|x)}{\hat{p}_{\text{CR}}(y|x)} dy \\
&= \underbrace{\int p(y|x) \log p(y|x)\, dy}_{-\mathcal{H}(p) \text{ (Entropy)}} - \sum_{k=1}^K \int_{B_k} p(y|x) \log\left(\frac{\pi_k(x)}{\Delta}\right) dy.
\end{align*}

\textbf{Step 1: Approximation Error (Bias).} By Assumption \ref{asm:smooth_dist} (Distribution Smoothness), $p(y|x)$ is Lipschitz continuous with constant $L$. For histogram estimators, while the integrated squared bias component of the IMSE is bounded by $\mathcal{O}(\Delta^2)$ \citep[Theorem 2.1]{scott1979optimal}, the absolute pointwise approximation bias is bounded by $\mathcal{O}(\Delta)$. Specifically, for any $y \in B_k$, the deviation from the true bin probability density is bounded:
% \textbf{Step 1: Approximation Error (Bias).} 
% By {Assumption \ref{asm:smooth_dist}} (Distribution Smoothness), $p(y|x)$ is Lipschitz continuous with constant $L$. For histogram estimators, the approximation bias is bounded by $O(\Delta^2)$ \citep[Theorem 2.1]{scott1979optimal}. Specifically, for any $y \in B_k$, the deviation from the true bin probability density is bounded:
\[
\left|p(y|x) - \frac{\pi_k^*(x)}{\Delta}\right| \leq L\Delta,
\]
where $\pi_k^*(x) = \int_{B_k} p(y|x) dy$ is the true probability mass of bin $k$. Thus, as the bin width $\Delta \to 0$ (implying $K \to \infty$), the approximation error vanishes.

\textbf{Step 2: Estimation Error (Variance).} 
By {Assumption \ref{asm:model_capacity}} (Model Capacity), the neural network backbone $\phi(x)$ is trained with categorical cross-entropy, which is equivalent to minimizing the discrete KL divergence $D_{\text{KL}}(\pi^*(x)\|\pi(x))$. Given the universal approximation property of neural networks and the consistency of M-estimators, the estimated probabilities converge in probability to the true bin probabilities: $\hat{\pi}(x) \xrightarrow{P} \pi^*(x)$ as $N \to \infty$.

\textbf{Conclusion.} 
Combining both steps, as the sample size $N \to \infty$ and bin count $K \to \infty$ (such that $K/N \to 0$ to ensure sufficient samples per bin), the expected KL divergence converges to zero:
\[
\mathbb{E}_{x}\left[D_{\text{KL}}(p(y|x) \| \hat{p}_{\text{CR}}(y|x))\right] \to 0.
\]
This proves that CR-VBLL is a consistent density estimator for Lipschitz continuous targets.
\end{proof}

% \subsection{Finite-Bin Error Bound (Corollary A.1)}
% \label{app:proof_corollary_1_1}

In production recommendation systems, the number of bins $K$ is often fixed (e.g., $K=100$) for computational efficiency. We provide the detailed error bound for this practical setting, following the derivation in \citet{scott1979optimal}.

\clearpage

\section{Additional Experimental Details}
\label{app:experiments}

Discretized Negative Log-Likelihood (NLL). To benchmark continuous and discrete models on a unified scale, we discretize the target domain $[y_{\min}, y_{\max}]$ into $K=10$ bins $B_1, \dots, B_K$.For a continuous model predicting density $p(y|x)$ (or CDF $F(y|x)$), the probability mass for bin $B_k = [l_k, u_k]$ is computed as:$$\hat{P}(y \in B_k) = \int_{l_k}^{u_k} p(y|x) dy = F(u_k|x) - F(l_k|x)$$The NLL is then the negative log-probability of the bin containing the true $y_i$:$$NLL = -\frac{1}{N} \sum_{i=1}^N \sum_{k=1}^K \mathbb{I}(y_i \in B_k) \log \hat{P}(y \in B_k)$$

Discretized Expected Calibration Error (ECE). 
To rigorously evaluate conditional, instance-level confidence calibration, we compute the standard ECE by grouping the predictions into $M=10$ equally spaced confidence bins. For each test sample $i$, we identify the predicted target bin $\hat{k}_i = \text{argmax}_k \hat{P}(y_i \in B_k)$ and its associated predictive confidence $\hat{c}_i = \max_k \hat{P}(y_i \in B_k)$. We then partition the $N$ samples into $M$ subsets $\{I_m\}_{m=1}^M$ based on their confidence scores $\hat{c}_i$. The ECE is calculated as the weighted average of the absolute difference between the empirical accuracy and the average confidence within each bin:
\begin{equation}
    \text{ECE} = \sum_{m=1}^M \frac{|I_m|}{N} |\text{acc}(I_m) - \text{conf}(I_m)|
\end{equation}
where $\text{conf}(I_m) = \frac{1}{|I_m|} \sum_{i \in I_m} \hat{c}_i$ is the average confidence of samples in bin $m$, and $\text{acc}(I_m) = \frac{1}{|I_m|} \sum_{i \in I_m} \mathbf{1}(y_i \in B_{\hat{k}_i})$ is the empirical accuracy of the top predictions in that bin.
% Using the same $K=10$ bins, we compute ECE as the weighted deviation between confidence and accuracy:$$ECE = \sum_{k=1}^{K} \frac{|N_k|}{N} \left| \text{acc}(B_k) - \text{conf}(B_k) \right|$$where $\text{conf}(B_k)$ is the average predicted mass $\hat{P}(y \in B_k)$ for all samples, and $\text{acc}(B_k)$ is the fraction of samples where the true $y_i$ actually fell into $B_k$.

\subsection{Implementation and Hyperparameters}
All models were implemented in PyTorch and trained on a single NVIDIA A100 GPU. 
% We adopted a two-stage training protocol to strictly simulate production constraints: (1) The heavy backbone was pre-trained for a \textbf{single epoch} (online learning), and (2) the lightweight uncertainty heads were post-trained for 2--5 epochs on the fixed features.

\textbf{1. Data Preprocessing \& Feature Engineering.}
\begin{itemize}
    \item \textbf{ID Encoding:} WeChat has user ID embedding, pre-trained video content understanding embedding and video ID embedding; KuaiRec has user ID embedding, and video ID embedding; Uber has customer ID embedding and location name embedding.
    \item \textbf{Duration Transformation:} To capture non-linear biases (e.g., short vs. long videos), we discretized the continuous video duration into \textbf{100 quantile bins} ($K_{QR}=100$) following \citet{sun2024cread}. These bin indices were mapped to learnable embedding vectors of dimension $D_{dur}$, allowing the model to learn distinct latent representations for different duration cohorts.
    \item \textbf{Target Scaling:} The watch time ratio $y$ was clipped to the range $[0, 3.0]$ to mitigate the impact of outliers.
\end{itemize}

\textbf{2. Backbone Architecture.}
We employed a standard Multi-Tower architecture shared across all baselines:
\begin{itemize}
    \item \textbf{Embeddings:} User/Item embeddings of dimension $d=16$ (WeChat) and $d=32$ (KuaiRec).
    \item \textbf{Feature Interaction:} Concatenation of ID and Duration embeddings, followed by a 3-layer MLP ($256 \to 128 \to 64$) with ReLU activations and Dropout ($p=0.1$).
    \item \textbf{Spectral Normalization:} Applied to all linear layers in the backbone and also applied input concatenation at the last layer such that the bi-Lipschitz property in Assumption \ref{asm:spectral_norm} is guaranteed for SNGP, CR-VBLL and QR-VBLL.
\end{itemize}

\textbf{3. Training Configuration.}
\begin{table}[h]
    \centering
    \caption{Hyperparameter settings.}
    \label{tab:hyperparams}
    \begin{tabular}{l|ccc}
    \toprule
    Parameter & WeChat & KuaiRec & Uber \\
    \midrule
    Batch Size & 2048 & 2048 & 2048\\
    Learning Rate & $1e^{-3}$ & $1e^{-3}$ & $1e^{-3}$ \\
    Optimizer & Adam & Adam & Adam\\
    Training Bin Count  & 40 & 40 & 40\\
    Training Quantile Numbers & 100 & 100 & 100\\
    Evaluation Bin Count  & 10 & 10 & 10\\
    % VBLL KL Weight ($\beta$) & $1e^{-6}$ & $1e^{-6}$ \\
    % \midrule
    {Training Epochs} & 5 & 5 & 3 \\
    {MDN Mixture Number} & 5 & 5 & 5 \\
    {Train/Test Ratio} & 80\%/20\% & 80\%/20\% & 80\%/20\%\\
    % \textbf{Head Epochs} & \textbf{2} & \textbf{5} \\
    \bottomrule
    \end{tabular}
\end{table}
% We note that we choose the best epoch checkpoint for each method for comparison since different methods may have different convergence speed. 
Choice of bin counts and quantile numbers are covered in Appendix \ref{app:sensitivity}.

\subsection{Full Experiment Results Table}
In our experimental evaluation, we tailored our reporting strategy to the training stability of each dataset.
% For WeChat and KuaiRec, where all models showed stable and monotonic improvement as epochs increased, we simply reported the results from the 5th epoch in \cref{tab:kuairec_full,tab:wechat_full}. 
% While the Uber dataset presented more volatile training trends where most baseline methods show performance degradation as the number of epochs increases, we report the results from the checkpoint that achieved the best performance (lowest NLL) on an independent, held-out validation set within the first five epochs. 
% This standard early-stopping protocol ensures a fair comparison across all methods while strictly preventing any test set data leakage.
% While the Uber dataset presented a more volatile training trends that most methods show worse performance as number of epochs increase, we then report the best checkpoint achieved within the first five epochs rather than the final one in \cref{tab:uber_full} for fair comparison. 

\begin{table*}[h]
    \centering
    \caption{Full experimental results on the WeChat dataset. We report Mean $\pm$ Standard Deviation over 5 independent runs. QR-VBLL demonstrates state-of-the-art NLL performance among quantile regression methods.}
    \label{tab:wechat_full}
    \resizebox{0.95\textwidth}{!}{%
    \begin{tabular}{lcccc}
        \toprule
        \textbf{Method} & \textbf{NLL} $\downarrow$ & \textbf{RMSE} $\downarrow$ & \textbf{CRPS} $\downarrow$ & \textbf{ECE} $\downarrow$ \\
        \midrule
        \multicolumn{5}{l}{\textit{Parametric Regression Baselines}} \\
        Vanilla Gaussian & 2.1657 $\pm$ 0.0081 & 1.0466 $\pm$ 0.2960 & 0.4380 $\pm$ 0.0004 & 0.0264 $\pm$ 0.0017 \\
        SNGP (Reg) & 2.2730 $\pm$ 0.0249 & \textbf{0.7740 $\pm$ 0.0024} & 0.4439 $\pm$ 0.0038 & 0.0302 $\pm$ 0.0045 \\
        VBLL (Reg) & 2.1708 $\pm$ 0.0090 & 0.8327 $\pm$ 0.0299 & 0.4315 $\pm$ 0.0019 & 0.0251 $\pm$ 0.0006 \\
        DER & 2.1443 $\pm$ 0.0089 & 0.7811 $\pm$ 0.0016 & \textbf{0.4226 $\pm$ 0.0021} & 0.0263 $\pm$ 0.0008 \\
        MDN & \textbf{1.9515 $\pm$ 0.0133} & 1.2464 $\pm$ 0.1343 & 0.4418 $\pm$ 0.0052 & \textbf{0.0070 $\pm$ 0.0012} \\
        
        \midrule
        \multicolumn{5}{l}{\textit{Quantile Regression (QR) Methods}} \\
        Vanilla QR & 2.0317 $\pm$ 0.0107 & 0.7629 $\pm$ 0.0005 & 0.4156 $\pm$ 0.0004 & 0.0066 $\pm$ 0.0006 \\
        CQR & 2.0446 $\pm$ 0.0042 & 0.7723 $\pm$ 0.0016 & 0.4226 $\pm$ 0.0009 & 0.0065 $\pm$ 0.0005 \\
        QR-MC Dropout & 2.0297 $\pm$ 0.0038 & 0.7646 $\pm$ 0.0009 & 0.4172 $\pm$ 0.0006 & \textbf{0.0062 $\pm$ 0.0002} \\
        \textbf{QR-VBLL (Ours)} & \textbf{1.9616 $\pm$ 0.0004} & \textbf{0.7614 $\pm$ 0.0005} & \textbf{0.4134 $\pm$ 0.0004} & 0.0067 $\pm$ 0.0004 \\

        \midrule
        \multicolumn{5}{l}{\textit{Classification Restoration (CR) Methods}} \\
        Vanilla CR & 1.9199 $\pm$ 0.0074 & 0.7664 $\pm$ 0.0026 & \textbf{0.4063 $\pm$ 0.0003} & 0.0126 $\pm$ 0.0007 \\
        CR-SNGP & 1.9377 $\pm$ 0.0019 & 0.7756 $\pm$ 0.0007 & 0.4151 $\pm$ 0.0004 & 0.0133 $\pm$ 0.0006 \\
        CR-MC Dropout & 1.9245 $\pm$ 0.0042 & 0.7687 $\pm$ 0.0018 & 0.4072 $\pm$ 0.0007 & \textbf{0.0119 $\pm$ 0.0002} \\
        \textbf{CR-VBLL (Ours)} & \textbf{1.9189 $\pm$ 0.0028} & \textbf{0.7654 $\pm$ 0.0010} & 0.4068 $\pm$ 0.0006 & \textbf{0.0119 $\pm$ 0.0004} \\
        
        \midrule
        \midrule
        \multicolumn{5}{l}{\textit{Ensemble Benchmarks}} \\
        Ensemble (Reg) & 2.1502 $\pm$ 0.0021 & 0.8035 $\pm$ 0.0303 & 0.4315 $\pm$ 0.0005 & 0.0300 $\pm$ 0.0012 \\
        CR-Ensemble & \textbf{1.8981 $\pm$ 0.0022} & {0.7547 $\pm$ 0.0015} & {0.3982 $\pm$ 0.0007} & 0.0137 $\pm$ 0.0001 \\
        QR-Ensemble & 1.9973 $\pm$ 0.0009 & \textbf{0.7506 $\pm$ 0.0003} & \textbf{0.4076 $\pm$ 0.0002} & \textbf{0.0055 $\pm$ 0.0003} \\
        \bottomrule
    \end{tabular}
    }
\end{table*}

\begin{table*}[h]
    \centering
    \caption{Full experimental results on the KuaiRec dataset. We report Mean $\pm$ Standard Deviation over 5 independent runs. \textbf{QR-VBLL} (using the SN+Raw backbone) achieves NLL and CRPS virtually identical to the 5-model Ensemble, demonstrating high efficiency.}
    \label{tab:kuairec_full}
    \resizebox{0.95\textwidth}{!}{%
    \begin{tabular}{lcccc}
        \toprule
        \textbf{Method} & \textbf{NLL} $\downarrow$ & \textbf{RMSE} $\downarrow$ & \textbf{CRPS} $\downarrow$ & \textbf{ECE} $\downarrow$ \\
        \midrule
        \multicolumn{5}{l}{\textit{Parametric Regression Baselines}} \\
        Vanilla Gaussian & 1.5070 $\pm$ 0.0034 & 0.4457 $\pm$ 0.0006 & \textbf{0.2093 $\pm$ 0.0006} & 0.0291 $\pm$ 0.0011 \\
        SNGP (Reg) & 1.7470 $\pm$ 0.0143 & \textbf{0.4440 $\pm$ 0.0002} & 0.2394 $\pm$ 0.0025 & 0.0497 $\pm$ 0.0003 \\
        VBLL (Reg) & 1.5184 $\pm$ 0.0024 & 0.4484 $\pm$ 0.0001 & 0.2117 $\pm$ 0.0002 & 0.0327 $\pm$ 0.0005 \\
        DER & 1.4163 $\pm$ 0.0032 & 0.4580 $\pm$ 0.0005 & 0.2131 $\pm$ 0.0002 & 0.0231 $\pm$ 0.0003 \\
        MDN & \textbf{1.2943 $\pm$ 0.0039} & 0.4495 $\pm$ 0.0005 & 0.2117 $\pm$ 0.0009 & \textbf{0.0136 $\pm$ 0.0002} \\
        
        \midrule
        \multicolumn{5}{l}{\textit{Quantile Regression (QR) Methods}} \\
        Vanilla QR & 1.3938 $\pm$ 0.0024 & 0.4438 $\pm$ 0.0001 & 0.1968 $\pm$ 0.0002 & 0.0136 $\pm$ 0.0007 \\
        CQR & 1.4018 $\pm$ 0.0055 & 0.4447 $\pm$ 0.0003 & 0.1976 $\pm$ 0.0003 & 0.0148 $\pm$ 0.0015 \\
        QR-MC Dropout & 1.4049 $\pm$ 0.0041 & 0.4439 $\pm$ 0.0001 & 0.1978 $\pm$ 0.0000 & 0.0145 $\pm$ 0.0016 \\
        \textbf{QR-VBLL (Ours)} & \textbf{1.3716 $\pm$ 0.0048} & \textbf{0.4433 $\pm$ 0.0001} & \textbf{0.1964 $\pm$ 0.0002} & \textbf{0.0122 $\pm$ 0.0007} \\

        \midrule
        \multicolumn{5}{l}{\textit{Classification Restoration (CR) Methods}} \\
        Vanilla CR & 1.3978 $\pm$ 0.0049 & 0.4550 $\pm$ 0.0009 & 0.1991 $\pm$ 0.0001 & 0.0237 $\pm$ 0.0010 \\
        CR-SNGP & 1.3874 $\pm$ 0.0032 & 0.4584 $\pm$ 0.0015 & 0.1998 $\pm$ 0.0002 & \textbf{0.0223 $\pm$ 0.0006} \\
        CR-MC Dropout & 1.3937 $\pm$ 0.0010 & 0.4551 $\pm$ 0.0006 & \textbf{0.1987 $\pm$ 0.0001} & 0.0234 $\pm$ 0.0004 \\
        \textbf{CR-VBLL (Ours)} & \textbf{1.3867 $\pm$ 0.0015} & \textbf{0.4549 $\pm$ 0.0007} & 0.1989 $\pm$ 0.0001 & 0.0235 $\pm$ 0.0002 \\
        
        \midrule
        \midrule
        \multicolumn{5}{l}{\textit{Ensemble Benchmarks}} \\
        Ensemble (Reg) & 1.4980 $\pm$ 0.0022 & 0.4450 $\pm$ 0.0004 & 0.2085 $\pm$ 0.0003 & 0.0285 $\pm$ 0.0004 \\
        CR-Ensemble & 1.3880 $\pm$ 0.0006 & 0.4540 $\pm$ 0.0005 & 0.1980 $\pm$ 0.0000 & 0.0237 $\pm$ 0.0005 \\
        QR-Ensemble & \textbf{1.3858 $\pm$ 0.0024} & \textbf{0.4433 $\pm$ 0.0001} & \textbf{0.1958 $\pm$ 0.0002} & \textbf{0.0136 $\pm$ 0.0008} \\
        \bottomrule
    \end{tabular}
    }
\end{table*}

\begin{table*}[h]
    \centering
    \caption{Full experimental results on the Uber dataset. We report Mean $\pm$ Standard Deviation over 5 independent runs. \textbf{CR-VBLL} and \textbf{QR-VBLL} demonstrate superior performance across all metrics, with significant gains in NLL compared to baselines.}
    \label{tab:uber_full}
    \resizebox{0.95\textwidth}{!}{%
    \begin{tabular}{lcccc}
        \toprule
        \textbf{Method} & \textbf{NLL} $\downarrow$ & \textbf{RMSE} $\downarrow$ & \textbf{CRPS} $\downarrow$ & \textbf{ECE} $\downarrow$ \\
        \midrule
        \multicolumn{5}{l}{\textit{Parametric Regression Baselines}} \\
        Vanilla Gaussian & 2.6766 $\pm$ 0.3467 & 0.4686 $\pm$ 0.0008 & 0.3039 $\pm$ 0.0077 & 0.0848 $\pm$ 0.0106 \\
        SNGP (Reg) & \textbf{1.8745 $\pm$ 0.0046} & \textbf{0.4516 $\pm$ 0.0005} & \textbf{0.2634 $\pm$ 0.0007} & \textbf{0.0341 $\pm$ 0.0008} \\
        VBLL (Reg) & 3.3107 $\pm$ 0.2359 & 0.4533 $\pm$ 0.0016 & 0.3061 $\pm$ 0.0022 & 0.1043 $\pm$ 0.0039 \\
        DER & 2.8354 $\pm$ 0.0250 & 0.4541 $\pm$ 0.0002 & 0.3299 $\pm$ 0.0019 & 0.1196 $\pm$ 0.0015 \\
        MDN & 4.0078 $\pm$ 0.4319 & 0.4606 $\pm$ 0.0032 & 0.3264 $\pm$ 0.0063 & 0.1181 $\pm$ 0.0043 \\

        \midrule
        \multicolumn{5}{l}{\textit{Quantile Regression (QR) Methods}} \\
        Vanilla QR & 2.3728 $\pm$ 0.4327 & 0.4578 $\pm$ 0.0008 & 0.2670 $\pm$ 0.0019 & 0.0370 $\pm$ 0.0037 \\
        CQR & 1.9435 $\pm$ 0.0361 & 0.4590 $\pm$ 0.0020 & 0.2666 $\pm$ 0.0025 & 0.0379 $\pm$ 0.0053 \\
        QR-MC Dropout & 2.1179 $\pm$ 0.1991 & 0.4606 $\pm$ 0.0002 & 0.2686 $\pm$ 0.0006 & 0.0357 $\pm$ 0.0022 \\ 
        \textbf{QR-VBLL (Ours)} & \textbf{1.7830 $\pm$ 0.0028} & \textbf{0.4495 $\pm$ 0.0002} & \textbf{0.2560 $\pm$ 0.0002} & \textbf{0.0080 $\pm$ 0.0013} \\

        \midrule
        \multicolumn{5}{l}{\textit{Classification Restoration (CR) Methods}} \\
        Vanilla CR & 2.5328 $\pm$ 0.0509 & 0.4550 $\pm$ 0.0057 & 0.2911 $\pm$ 0.0125 & 0.0804 $\pm$ 0.0115 \\
        CR-SNGP & 1.7892 $\pm$ 0.0007 & 0.4512 $\pm$ 0.0001 & \textbf{0.2559 $\pm$ 0.0000} & 0.0087 $\pm$ 0.0010 \\
        CR-MC Dropout & 2.2687 $\pm$ 0.0777 & 0.4584 $\pm$ 0.0102 & 0.2882 $\pm$ 0.0152 & 0.0719 $\pm$ 0.0012 \\
        \textbf{CR-VBLL (Ours)} & \textbf{1.7839 $\pm$ 0.0005} & \textbf{0.4505 $\pm$ 0.0002} & \textbf{0.2559 $\pm$ 0.0000} & \textbf{0.0072 $\pm$ 0.0004} \\
        
        \midrule
        \midrule
        \multicolumn{5}{l}{\textit{Ensemble Benchmarks}} \\
        Ensemble (Reg) & 2.7145 $\pm$ 0.1050 & 0.4686 $\pm$ 0.0028 & 0.3068 $\pm$ 0.0008 & 0.0903 $\pm$ 0.0036 \\
        CR-Ensemble & 2.3305 $\pm$ 0.0502 & \textbf{0.4506 $\pm$ 0.0002} & 0.2693 $\pm$ 0.0029 & 0.0506 $\pm$ 0.0039 \\
        QR-Ensemble & \textbf{2.1280 $\pm$ 0.0587} & 0.4598 $\pm$ 0.0005 & \textbf{0.2686 $\pm$ 0.0006} & \textbf{0.0393 $\pm$ 0.0006} \\
        \bottomrule
    \end{tabular}
    }
\end{table*}

% In \cref{fig:uber_nll_learning_curves}, we present the NLL learning curves for the Uber dataset. 
% Most baseline methods (e.g., Vanilla QR and Gaussian) suffer from severe overfitting, with errors spiking after the second epoch.
% In contrast, both CR-VBLL and SNGP (Reg) demonstrate significant resistance to noise. 
% While SNGP achieves stability through spectral normalization which enforces smoothness by bounding the network's Lipschitz constant, CR-VBLL achieves robustness through the discretization of the target space (CR) combined with variational inference (VBLL). 
% The success of these two distinct approaches strongly supports the hypothesis that the Uber dataset requires explicit uncertainty modeling and regularization to prevent the memorization of outliers.
% \begin{figure}[h]
%     \centering
%     \includegraphics[width=0.8\textwidth]{uber_nll_learning_curves.png}
%     \caption{\textbf{NLL Learning Curves on the Uber Dataset (First 5 Epochs).} The plot illustrates the NLL trajectories for all candidate methods, highlighting the severe overfitting and volatility observed in baselines (e.g., Vanilla QR, Vanilla Gaussian). In contrast, CR-VBLL demonstrates superior robustness, maintaining stable performance without degradation. Shaded regions represent the standard error.}
%     \label{fig:uber_nll_learning_curves}
% \end{figure}

\subsection{Ablation Studies}\label{app:ablation}

\subsubsection{Spectral Normalization with Input Concatenation}\label{app:sn_and_raw}
\begin{table}[h]
    \centering
    \caption{Ablation study on backbone architecture: Spectral Normalization (SN) vs. SN + Raw Features (SN+Raw) on KuaiRec. The SN+Raw architecture consistently improves NLL and RMSE for QR-VBLL, demonstrating that re-introducing raw features recovers information lost by the strict spectral constraint. \textbf{Bold} indicates the best performance per method.}
    \label{tab:ablation_backbone}
    \resizebox{0.95\columnwidth}{!}{%
    \begin{tabular}{llccccc}
        \toprule
        \textbf{Method} & \textbf{Backbone} & \textbf{NLL} $\downarrow$ & \textbf{RMSE} $\downarrow$ & \textbf{CRPS} $\downarrow$ & \textbf{ECE} $\downarrow$ \\
        % \midrule
        % \multirow{2}{*}{Vanilla CR} 
        % & SN & \textbf{1.3947 $\pm$ 0.0051} & 0.4564 $\pm$ 0.0003 & \textbf{0.1994 $\pm$ 0.0001} & \textbf{0.0219 $\pm$ 0.0012} \\
        % & SN+Raw & 1.3978 $\pm$ 0.0049 & \textbf{0.4550 $\pm$ 0.0009} & 0.1991 $\pm$ 0.0001 & 0.0237 $\pm$ 0.0010\\
        % \midrule
        % \multirow{2}{*}{CR-VBLL} 
        % & SN & 1.3952 $\pm$ 0.0006 & 0.4556 $\pm$ 0.0001 & \textbf{0.1987 $\pm$ 0.0000} & 0.0240 $\pm$ 0.0001 \\                          
        % & {SN+Raw} & \textbf{1.3885 $\pm$ 0.0032} & \textbf{0.4549 $\pm$ 0.0005} & 0.1996 $\pm$ 0.0000 & 0.0240 $\pm$ 0.0008 \\
        % \midrule
        \midrule
        \multirow{2}{*}{Vanilla QR}  
        & SN & 1.4047 $\pm$ 0.0005 & 0.4451 $\pm$ 0.0004 & 0.1981 $\pm$ 0.0001 & 0.0142 $\pm$ 0.0005 \\
        & SN+Raw & \textbf{1.3938 $\pm$ 0.0024} & \textbf{0.4438 $\pm$ 0.0001} & \textbf{0.1968 $\pm$ 0.0002} & \textbf{0.0136 $\pm$ 0.0007} \\
        \midrule
        \multirow{2}{*}{QR-VBLL} 
        & SN & 1.4043 $\pm$ 0.0025 & 0.4451 $\pm$ 0.0004 & 0.1985 $\pm$ 0.0002 & 0.0169 $\pm$ 0.0015 \\
        & {SN+Raw} & \textbf{1.3716 $\pm$ 0.0048} & \textbf{0.4433 $\pm$ 0.0001} & \textbf{0.1964 $\pm$ 0.0002} & \textbf{0.0122 $\pm$ 0.0007} \\
        \bottomrule
    \end{tabular}
    }
\end{table}

To rigorously validate our architectural design, we conducted an ablation study on the KuaiRec dataset comparing the standard Spectral Normalization (SN-Only) backbone against our proposed SN+Raw configuration. The results are summarized in Table \ref{tab:ablation_backbone}.

The analysis demonstrates that the \textbf{SN+Raw} architecture is the optimal choice for the VBLL framework, providing consistent gains in NLL and RMSE for both proposed methods:

% \begin{itemize}
%     \item \textbf{Quantile Regression (QR-VBLL):} 
    The impact here is decisive. Adding raw features reduces NLL from $1.404$ to $1.372$ and significantly improves calibration (ECE drops from $0.017$ to $0.012$). This confirms that while Spectral Normalization is necessary for OOD detection (Section \ref{sec:theory}), strictly constraining the Lipschitz constant limits the model's ability to scale output magnitudes for precise regression in the tails. Re-introducing raw features resolves this bottleneck.
    
%     \item \textbf{Classification Restoration (CR-VBLL):} For our method, SN+Raw consistently improves both NLL ($1.395 \to 1.389$) and RMSE ($0.456 \to 0.455$). Interestingly, while the \textit{Vanilla CR} baseline sees a slight degradation in NLL with raw features ($1.395 \to 1.398$), \textbf{CR-VBLL} effectively leverages the unconstrained features to refine the probability density estimate, achieving the best overall performance in the classification setting.
% \end{itemize}

Given that SN+Raw consistently enhances the predictive likelihood of VBLL models with negligible parameter overhead ($\sim 2-3\%$), we adopt it as the default backbone.

\subsection{Sensitivity Analysis of Discretization Resolution}
\label{app:sensitivity}

We conducted a comprehensive sensitivity analysis to determine the optimal discretization resolution for both methods: the number of bins $K_{CR}$ for CR-VBLL and the number of quantiles $K_{QR}$ for QR-VBLL. We evaluated the performance on the WeChat, KuaiRec, and Uber datasets across the range $\{20, 40, 60, 80, 100\}$. The results are visualized in Figure \ref{fig:bin_ablation} and Figure \ref{fig:qt_ablation}.

\begin{figure}[t]
    \centering
    \includegraphics[width=0.9\textwidth]{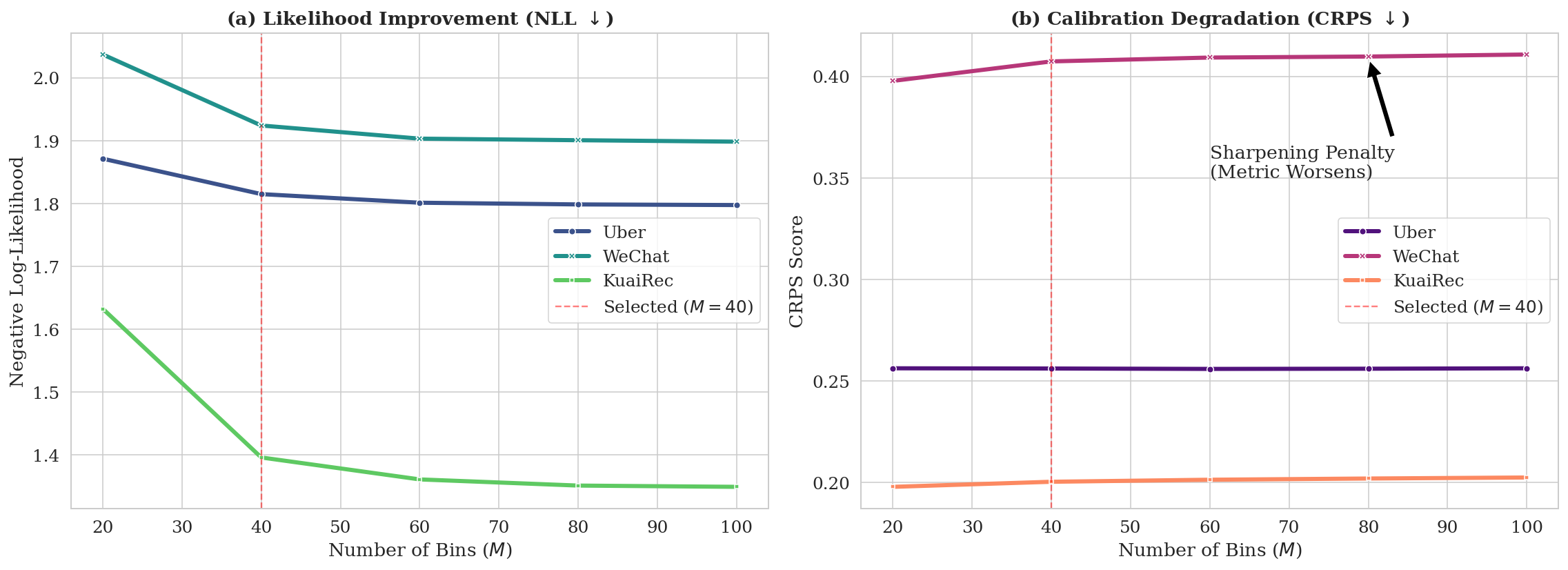}
    \caption{\textbf{CR-VBLL Sensitivity (Bins $K_{CR}$):} Efficiency vs. Calibration Trade-off. While increasing $K_{CR}$ from 20 to 40 yields significant NLL gains across all datasets (Left), further increasing $K_{CR}$ often leads to CRPS degradation (Right) due to the ``sharpening effect,'' where the model becomes overconfident in specific bins. We select $K_{CR}=40$ as the optimal balance between predictive likelihood and calibration stability.}
    \label{fig:bin_ablation}
\end{figure}

\begin{figure}[t]
    \centering
    \includegraphics[width=0.9\textwidth]{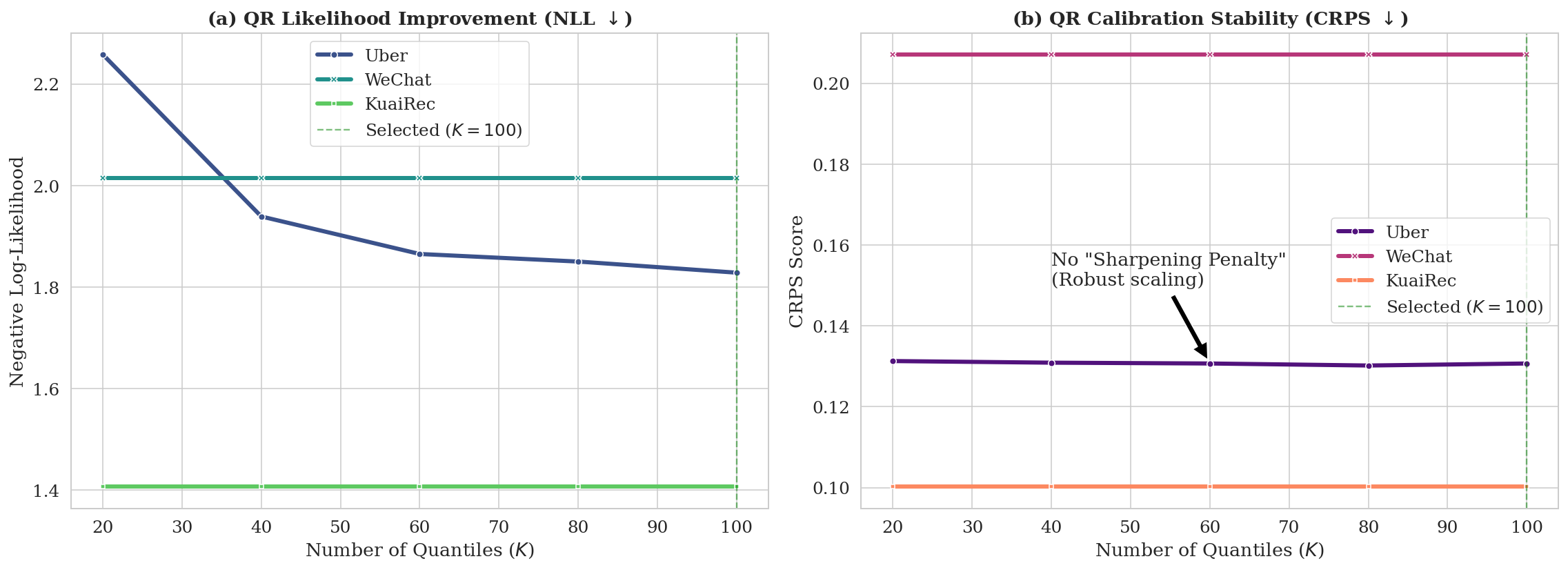}
    \caption{\textbf{QR-VBLL Sensitivity (Quantiles $K_{QR}$):} High-Resolution Robustness. Unlike CR, Quantile Regression benefits monotonically from higher resolution. On the sparse Uber dataset, NLL improves continuously as $K_{QR}$ increases (Left), while calibration (CRPS) remains stable without degradation (Right). We select $K_{QR}=100$ to maximize precision on difficult distributions without incurring penalties on simpler ones. Here the CRPS metrics in the curves are timed with a constant of 0.5.}
    \label{fig:qt_ablation}
\end{figure}

\textbf{Classification Restoration: The Efficiency Trade-off}
For CR-VBLL, Figure \ref{fig:bin_ablation} reveals a distinct trade-off between local likelihood and global calibration:
\begin{itemize}
    \item \textbf{NLL Saturation:} Increasing the bin count from 20 to 40 results in consistent NLL improvements across all datasets (e.g., significant drops on KuaiRec and Uber). However, beyond $M=40$, the marginal gain in likelihood diminishes significantly.
    \item \textbf{Calibration Degradation:} Crucially, for $M > 40$, we observe that CRPS often stagnates or degrades (particularly on WeChat and KuaiRec). This aligns with the theoretical "sharpening effect" \citep{gneiting2007strictly}: as the discrete distribution becomes highly granular, the model maximizes NLL by placing excessive mass on single bins, which penalizes the global cumulative ranked score if slightly misaligned.
\end{itemize}
\textbf{Decision:} We select \textbf{$M=40$} as the default. This setting captures the "knee" of the performance curve, ensuring high predictive accuracy and computational efficiency (keeping the final classification layer lightweight) while avoiding the calibration penalties associated with hyper-fine discretization.

\textbf{Quantile Regression: High-Resolution Robustness}
For QR-VBLL, Figure \ref{fig:qt_ablation} demonstrates a different behavioral pattern characterized by robustness:
\begin{itemize}
    \item \textbf{Monotonic Improvement:} On challenging, sparse distributions (Uber), increasing the quantile resolution $K$ provides strictly monotonic benefits for NLL ($2.26 \to 1.83$), indicating that high resolution is necessary to approximate irregular cumulative distribution functions.
    \item \textbf{No Overfitting:} Unlike the binning approach, increasing $K$ to 100 does not degrade CRPS on any dataset. The quantile estimates effectively approximate the continuous CDF without imposing artificial discrete boundaries, allowing the model to scale to high resolutions without "over-sharpening."
\end{itemize}
\textbf{Decision:} Given the absence of a performance penalty on dense data and the clear necessity for high resolution on sparse data, we standardize on \textbf{$K=100$} for all QR experiments.

\subsection{Ablation Study: Hybrid Acquisition Hyperparameter ($\gamma$)}
\label{app:active_learning_ablation}

To determine the optimal balance between exploration (epistemic uncertainty) and noise characterization (aleatoric uncertainty), we conduct an ablation study on the hyperparameter $\gamma$ in our hybrid acquisition function in Eqn \eqref{eq:acquisition_function}.

$\mathcal{A}(x) = s^2(x) + \gamma \sigma^2(x)$.

Table \ref{tab:ablation_lambda_full} and \ref{tab:ablation_lambda_cold} present the full results on WeChat. We compare various $\gamma$ values (trained on 70\% samples) against a \textbf{Random Sampling} baseline trained on the full dataset.

\textbf{Key Findings:}
\begin{itemize}
    \item \textbf{Benefit of Aleatoric Sampling ($\gamma > 0$):} Pure epistemic sampling ($\gamma=0$) yields a test NLL of 2.0283. Introducing a moderate amount of aleatoric uncertainty ($\gamma=0.1$) significantly improves this score, achieving the best Full NLL of \textbf{2.0206}. This confirms our hypothesis that sampling high-noise regions helps the model accurately calibrate its variance estimates, though penalizing it too aggressively ($\gamma \ge 0.3$) begins to degrade performance.
    \item \textbf{Efficiency vs. Random Baseline:} Notably, our optimal Hybrid strategy ($\gamma=0.1$) trained on only \textbf{560k} samples comfortably outperforms the pure Random Sampling baseline trained on the full \textbf{800k} dataset (NLL 2.0261).
    \item \textbf{Cold-Start Consistency:} These improvements are mirrored in the sparse Cold-Start subset. The $\gamma=0.1$ strategy achieves an NLL of 2.0157, significantly outperforming both the pure epistemic strategy (2.0167) and the 800k Random baseline (2.0176).
\end{itemize}

\begin{table*}[h]
    \centering
    \caption{\textbf{Ablation of Acquisition Hyperparameter $\gamma$ on WeChat (Full Test Set).} We compare different weights for aleatoric uncertainty against a Random Sampling baseline. The Hybrid strategy with $\gamma=0.1$ achieves the best performance across all metrics with only 560k samples, outperforming both the pure epistemic strategy ($\gamma=0$) and the pure Random baseline using 800k samples.}
    \label{tab:ablation_lambda_full}
    \resizebox{0.95\textwidth}{!}{%
    \begin{tabular}{l c c c c c}
        \toprule
        \textbf{Strategy} & \textbf{Samples} & \textbf{NLL} $\downarrow$ & \textbf{RMSE} $\downarrow$ & \textbf{CRPS} $\downarrow$ & \textbf{ECE} $\downarrow$ \\
        \midrule
        Random Baseline & 800k & 2.0261 $\pm$ 0.0017 & 0.8037 $\pm$ 0.0013 & 0.4438 $\pm$ 0.0008 & 0.0093 $\pm$ 0.0017 \\
        \midrule
        BALD-MCD & 560k & 2.0258 $\pm$ 0.0069 & \textbf{0.8033 $\pm$ 0.0021} & \textbf{0.4434 $\pm$ 0.0012} & 0.0100 $\pm$ 0.0013 \\
        BALD-VBLL & 560k & 2.0277 $\pm$ 0.0056 & 0.8063 $\pm$ 0.0038 & 0.4454 $\pm$ 0.0024 & 0.0138 $\pm$ 0.0020 \\
        Epistemic Only ($\gamma=0$) & 560k & 2.0283 $\pm$ 0.0063 & 0.8062 $\pm$ 0.0037 & 0.4452 $\pm$ 0.0024 & 0.0135 $\pm$ 0.0017 \\
        \textbf{Hybrid ($\gamma=0.1$)} & 560k & \textbf{2.0206 $\pm$ 0.0032} & {0.8034 $\pm$ 0.0010} & \textbf{0.4434 $\pm$ 0.0006} & \textbf{0.0091 $\pm$ 0.0007} \\
        Hybrid ($\gamma=0.2$) & 560k & 2.0225 $\pm$ 0.0119 & 0.8065 $\pm$ 0.0058 & 0.4456 $\pm$ 0.0038 & 0.0144 $\pm$ 0.0055 \\
        Hybrid ($\gamma=0.3$) & 560k & 2.0292 $\pm$ 0.0050 & 0.8073 $\pm$ 0.0027 & 0.4460 $\pm$ 0.0018 & 0.0130 $\pm$ 0.0022 \\
        Hybrid ($\gamma=0.5$) & 560k & 2.0284 $\pm$ 0.0121 & 0.8078 $\pm$ 0.0029 & 0.4464 $\pm$ 0.0020 & 0.0123 $\pm$ 0.0045 \\
        \bottomrule
    \end{tabular}
    }
\end{table*}

\begin{table*}[h]
    \centering
    \caption{\textbf{Ablation of Acquisition Hyperparameter $\gamma$ on WeChat (Cold-Start Users).} Evaluating the active learning strategies on the sparse, cold-start subset. Incorporating aleatoric uncertainty ($\gamma > 0$) consistently improves calibration and predictive likelihood over purely epistemic sampling.}
    \label{tab:ablation_lambda_cold}
    \resizebox{0.95\textwidth}{!}{%
    \begin{tabular}{l c c c c c}
        \toprule
        \textbf{Strategy} & \textbf{Samples} & \textbf{NLL} $\downarrow$ & \textbf{RMSE} $\downarrow$ & \textbf{CRPS} $\downarrow$ & \textbf{ECE} $\downarrow$ \\
        \midrule
        Random Baseline & 800k & 2.0176 $\pm$ 0.0010 & 0.8332 $\pm$ 0.0015 & \textbf{0.4602 $\pm$ 0.0008} & 0.0102 $\pm$ 0.0023 \\
        \midrule
        BALD-MCD & 560k & 2.0195 $\pm$ 0.0093 & 0.8349 $\pm$ 0.0004 & 0.4612 $\pm$ 0.0004 & 0.0120 $\pm$ 0.0024 \\
        BALD-VBLL & 560k & 2.0162 $\pm$ 0.0031 & 0.8358 $\pm$ 0.0037 & 0.4620 $\pm$ 0.0024 & 0.0142 $\pm$ 0.0014 \\
        Epistemic Only ($\gamma=0$) & 560k & 2.0167 $\pm$ 0.0038 & 0.8357 $\pm$ 0.0036 & 0.4618 $\pm$ 0.0022 & 0.0136 $\pm$ 0.0006 \\
        \textbf{Hybrid ($\gamma=0.1$)} & 560k & 2.0157 $\pm$ 0.0053 & \textbf{0.8330 $\pm$ 0.0010} & \textbf{0.4602 $\pm$ 0.0004} & \textbf{0.0092 $\pm$ 0.0008} \\
        Hybrid ($\gamma=0.2$) & 560k & 2.0204 $\pm$ 0.0109 & 0.8357 $\pm$ 0.0049 & 0.4620 $\pm$ 0.0032 & 0.0162 $\pm$ 0.0036 \\
        Hybrid ($\gamma=0.3$) & 560k & 2.0212 $\pm$ 0.0034 & 0.8371 $\pm$ 0.0023 & 0.4628 $\pm$ 0.0014 & 0.0149 $\pm$ 0.0018 \\
        Hybrid ($\gamma=0.5$) & 560k & \textbf{2.0155 $\pm$ 0.0112} & 0.8367 $\pm$ 0.0042 & 0.4626 $\pm$ 0.0028 & 0.0129 $\pm$ 0.0029 \\
        \bottomrule
    \end{tabular}
    }
\end{table*}

\begin{table*}[h]
    \centering
    \caption{\textbf{Active Learning Performance on KuaiRec (Full Test Set).} Comparison of active learning strategies using 560k samples against a Random baseline using 800k samples.}
    \label{tab:al_kuairec}
    \resizebox{0.95\textwidth}{!}{%
    \begin{tabular}{l c | c c c c}
        \toprule
        \textbf{Strategy} & \textbf{Samples} & \textbf{NLL} $\downarrow$ & \textbf{RMSE} $\downarrow$ & \textbf{CRPS} $\downarrow$ & \textbf{ECE} $\downarrow$ \\
        \midrule
        Random Baseline & 800k & 1.7502 $\pm$ 0.0079 & \textbf{0.4531 $\pm$ 0.0001} & \textbf{0.2085 $\pm$ 0.0002} & 0.0093 $\pm$ 0.0003 \\
        \midrule
        BALD-MCD & 560k & 1.8175 $\pm$ 0.0146 & 0.4554 $\pm$ 0.0002 & 0.2100 $\pm$ 0.0002 & \textbf{0.0079 $\pm$ 0.0008} \\
        BALD-VBLL & 560k & 2.0450 $\pm$ 0.0078 & 0.4620 $\pm$ 0.0003 & 0.2144 $\pm$ 0.0004 & 0.0133 $\pm$ 0.0033 \\
        Epistemic Only ($\gamma=0$) & 560k & 1.7662 $\pm$ 0.0080 & 0.4543 $\pm$ 0.0004 & 0.2100 $\pm$ 0.0008 & 0.0119 $\pm$ 0.0025 \\
        Hybrid ($\gamma=0.1$) & 560k & 1.7564 $\pm$ 0.0113 & 0.4535 $\pm$ 0.0004 & 0.2090 $\pm$ 0.0004 & 0.0129 $\pm$ 0.0034 \\
        Hybrid ($\gamma=0.2$) & 560k & \textbf{1.7423 $\pm$ 0.0156} & \textbf{0.4531 $\pm$ 0.0002} & \textbf{0.2085 $\pm$ 0.0004} & 0.0120 $\pm$ 0.0041 \\
        Hybrid ($\gamma=0.3$) & 560k & 1.7519 $\pm$ 0.0274 & 0.4533 $\pm$ 0.0003 & 0.2094 $\pm$ 0.0012 & 0.0149 $\pm$ 0.0053 \\
        Hybrid ($\gamma=0.5$) & 560k & 1.7553 $\pm$ 0.0212 & 0.4532 $\pm$ 0.0003 & 0.2090 $\pm$ 0.0006 & 0.0126 $\pm$ 0.0037 \\
        \bottomrule
    \end{tabular}
    }
\end{table*}

\begin{table*}[h]
    \centering
    \caption{\textbf{Active Learning Performance on Uber (Full Test Set).} Comparison of active learning strategies using $\sim$84k samples against a Random baseline using $\sim$120k samples.}
    \label{tab:al_uber}
    \resizebox{0.95\textwidth}{!}{%
    \begin{tabular}{l c | c c c c}
        \toprule
        \textbf{Strategy} & \textbf{Samples} & \textbf{NLL} $\downarrow$ & \textbf{RMSE} $\downarrow$ & \textbf{CRPS} $\downarrow$ & \textbf{ECE} $\downarrow$ \\
        \midrule
        Random Baseline & $\sim$120k & 1.7919 $\pm$ 0.0062 & 0.4494 $\pm$ 0.0001 & 0.2564 $\pm$ 0.0000 & \textbf{0.0110 $\pm$ 0.0006} \\
        \midrule
        BALD-MCD & $\sim$84k & 1.7851 $\pm$ 0.0020 & 0.4494 $\pm$ 0.0001 & {0.2564 $\pm$ 0.0001} & 0.0146 $\pm$ 0.0007 \\
        BALD-VBLL & $\sim$84k & 1.7853 $\pm$ 0.0030 & \textbf{0.4493 $\pm$ 0.0000} & \textbf{0.2562 $\pm$ 0.0000} & 0.0139 $\pm$ 0.0008 \\
        Epistemic Only ($\gamma=0$) & $\sim$84k & 1.7848 $\pm$ 0.0030 & 0.4494 $\pm$ 0.0001 & {0.2568 $\pm$ 0.0002} & 0.0142 $\pm$ 0.0019 \\
        Hybrid ($\gamma=0.1$) & $\sim$84k & 1.7841 $\pm$ 0.0018 & 0.4494 $\pm$ 0.0001 & \textbf{0.2562 $\pm$ 0.0000} & 0.0150 $\pm$ 0.0008 \\
        Hybrid ($\gamma=0.2$) & $\sim$84k & \textbf{1.7840 $\pm$ 0.0021} & \textbf{0.4493 $\pm$ 0.0000} & \textbf{0.2562 $\pm$ 0.0000} & 0.0138 $\pm$ 0.0006 \\
        Hybrid ($\gamma=0.3$) & $\sim$84k & 1.7850 $\pm$ 0.0025 & 0.4494 $\pm$ 0.0001 & \textbf{0.2562 $\pm$ 0.0000} & 0.0152 $\pm$ 0.0010 \\
        Hybrid ($\gamma=0.5$) & $\sim$84k & 1.7854 $\pm$ 0.0024 & 0.4494 $\pm$ 0.0001 & \textbf{0.2562 $\pm$ 0.0000} & 0.0147 $\pm$ 0.0008 \\
        \bottomrule
    \end{tabular}
    }
\end{table*}

\subsection{Visualizing Calibration Improvements}
While the main text reports the scalar Expected Calibration Error (ECE), we visualize the full probabilistic quality of our method using Reliability Diagrams (Calibration Curves) in Figure~\ref{fig:calibration_curve}.

For regression tasks, calibration is assessed via the Probability Integral Transform (PIT) values, $u_i = \hat{F}(y_i|x_i)$, where $\hat{F}$ is the predicted cumulative distribution function. For a perfectly calibrated model, the PIT values should be uniformly distributed in $[0,1]$, resulting in a calibration curve that lies on the diagonal $y=x$.

As observed in Figure~\ref{fig:calibration_curve}, the \textbf{Vanilla Gaussian} baseline (\jwnew{Blue}) deviates significantly from the diagonal. The characteristic ``S-shape'' indicates that the model is miscalibrated at the tails—a direct symptom of the ``Ghost Value'' pathology, where the model artificially inflates variance to cover disjoint modes, effectively assigning probability mass to regions that should be empty. In contrast, \textbf{CR-VBLL} (\jwnew{Green}) closely hugs the diagonal, confirming that our distribution-free approach produces reliable probability estimates across the entire support of the target variable.

\begin{figure}[h]
    \centering
    \includegraphics[width=0.6\textwidth]{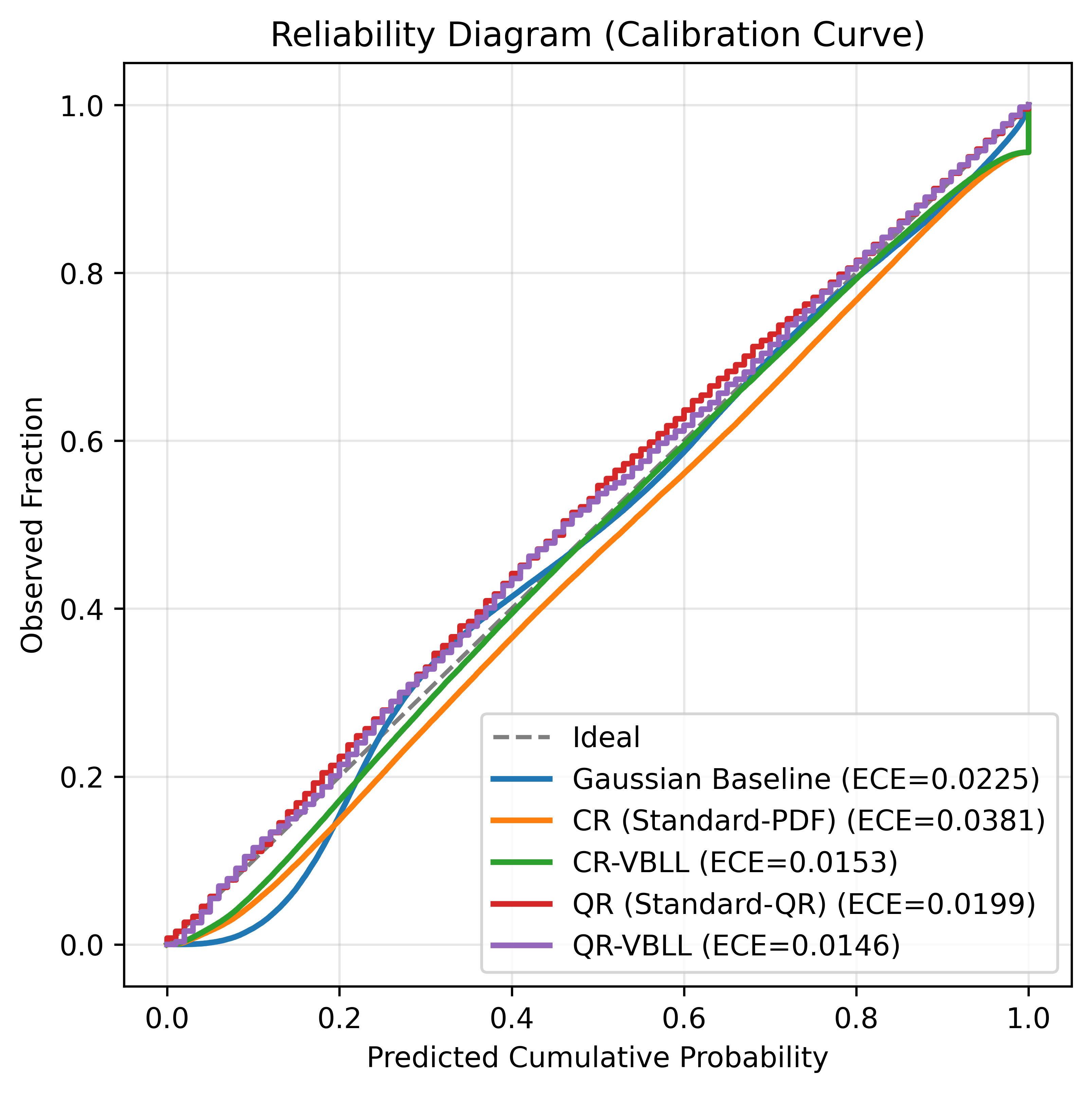}
    \caption{\textbf{Reliability Diagram on WeChat.} Comparison of calibration curves between Vanilla Gaussian regression and our CR-VBLL method. The Gaussian model shows significant deviation due to the unimodal assumption failure, while CR-VBLL maintains robust calibration (hugging the diagonal), effectively acting as a ``Safety Net'' by correctly modeling the multi-modal density.}
    \label{fig:calibration_curve}
\end{figure}

\clearpage

\end{document}